\documentclass{article}

\usepackage[final]{pdfpages}
\usepackage{bibentry}
\usepackage{floatrow}
\usepackage{setspace}
\usepackage{booktabs, makecell}
\usepackage{keyval}
\usepackage{lipsum}
\usepackage{graphicx}
\usepackage{epstopdf}
\usepackage{hhline}
\usepackage{multirow}
\usepackage {tikz}
\usetikzlibrary{calc, positioning, decorations.pathreplacing}
\usepackage{bm}
\usepackage{enumitem}
\usepackage{etoolbox}
\usepackage{todonotes}

\usepackage{mathtools}
\usepackage{appendix}
\usepackage{amsmath,amssymb,amsfonts,dsfont,mathrsfs,environ}
\usepackage{amsthm}
\usepackage{amsopn}
\usepackage{bbm}
\usepackage{verbatim}
\usepackage{algorithm}
\usepackage{algpseudocode}
\usepackage{color}
\usepackage{authblk}
\usepackage{titlesec}
\usepackage{geometry}
\usepackage{fancyhdr,lipsum}
\usepackage[square,numbers]{natbib}
\usepackage{mathptmx}
\usepackage{environ}
\usepackage[acronym]{glossaries}
\usepackage{anyfontsize}
\usepackage{thmtools}
\usepackage{thm-restate}
\usepackage{cleveref}

\algrenewcommand\algorithmicrequire{\textbf{Input:}}
\algrenewcommand\algorithmicensure{\textbf{Output:}}

\ifpdf
\DeclareGraphicsExtensions{.eps,.pdf,.png,.jpg}
\else
\DeclareGraphicsExtensions{.eps}
\fi
\graphicspath{{./figures/}}
\definecolor {processblue}{cmyk}{0.96,0,0,0}
\newfloatcommand{capbtabbox}{table}[][\FBwidth]
\title{Spectral neighbor joining for reconstruction of latent tree models}
\date{}
\author[1]{Ariel Jaffe}
\author[1]{Noah Amsel}
\author[1]{Yariv Aizenbud}
\author[2]{Boaz Nadler}
\author[3]{Joseph T. Chang}
\author[1,4,5]{Yuval Kluger}
\affil[1]{Program in Applied Mathematics, Yale University, New Haven, CT 06511}
\affil[2]{Department of Computer Science, Weizmann Institute of Science, Rehovot, 76100, Israel}
\affil[3]{Department of Statistics, Yale University, New Haven, CT 06520, USA}
\affil[4]{Interdepartmental Program in Computational Biology and Bioinformatics, New Haven, CT 06511}
\affil[5]{Department of Pathology, New Haven, CT 06511}

\newcommand\YUGE{\fontsize{48}{60}\selectfont}
\definecolor{gray75}{gray}{0.75}
\newcommand{\hsp}{\hspace{20pt}}
\titleformat{\chapter}[hang]{\YUGE\bfseries}{\thechapter\hsp\textcolor{gray75}{$|$}\hsp}{0pt}{\LARGE\bfseries}

\AtBeginEnvironment{subappendices}{%
	\chapter*{Appendix}
	\addcontentsline{toc}{section}{Appendices}
}
\newcommand{\nocontentsline}[3]{}
\let\oldaddcontentsline\addcontentsline
\newcommand{\tocless}[2]{%
	\let\addcontentsline\nocontentsline
	#1{#2}
	\let\addcontentsline\oldaddcontentsline}

\DeclareMathOperator*{\argmin}{argmin}

\newcommand{\T}{\mathcal T}

\newcommand{\R}{\mathbb R}

\newcommand{\beq}{\begin{eqnarray*}}
\newcommand{\eeq}{\end{eqnarray*}}
\newcommand{\beqn}{\begin{eqnarray}}
\newcommand{\eeqn}{\end{eqnarray}}
\newcommand{\nrm}[1]{{\|#1\|}}

\newcommand{\RNum}[1]{\uppercase\expandafter{\romannumeral #1\relax}}

\NewEnviron{salign}{%
	\begin{align}\begin{split}
			\BODY
	\end{split}\end{align}}

\newtheorem{theorem}{Theorem}
\newtheorem{lemma}{Lemma}

\newtheorem{conjecture*}{Conjecture}

\newtheorem{corollary}{Corollary}
\numberwithin{lemma}{section}
\numberwithin{theorem}{section}
\numberwithin{corollary}{section}
\newcommand{\frob}[1]{\nrm{#1}_F}

\titleformat*{\section}{\Large \bfseries}

\DeclareMathAlphabet{\mathcal}{OMS}{cmsy}{m}{n}
\newcommand{\update}[1]{#1}
\usepackage{subcaption}

\usepackage{thm-restate}
\usepackage {tikz}
\usetikzlibrary{patterns}
\usetikzlibrary{calc, positioning, decorations.pathreplacing}
\usepackage{color}

\begin{document}

\maketitle


\begin{abstract}	
A  common  assumption in  multiple scientific applications is that the distribution of observed data can be modeled by a latent tree graphical model. An important example is phylogenetics, where the tree models the evolutionary lineages of a set of observed organisms.
    Given a set of independent realizations of the random variables at the leaves of the tree, a key challenge is to infer the underlying tree topology. 
    In this work we develop Spectral Neighbor Joining (SNJ), a novel method to recover the structure of latent tree graphical models. 
    Given a matrix that contains a measure of similarity between all pairs of observed variables, SNJ computes a spectral 
    measure of cohesion between groups of observed variables.
    We prove that SNJ is consistent, and derive a sufficient condition for correct tree recovery from an estimated similarity matrix. Combining this condition with a concentration of measure result on the similarity matrix, we bound the number of samples required to recover the tree with high probability. 	
	We illustrate via extensive simulations that in comparison to several other reconstruction methods, SNJ requires fewer samples to accurately recover trees with a large number of leaves or long edges.  
\end{abstract}

\begin{keywords}
	latent variable models, Markov random fields, evolutionary trees, singular values, spectral methods, neighbor joining, phylogenetics, tree graphical model
\end{keywords}


\section{Introduction}

Learning the structure of an unobserved tree graphical model is a fundamental problem in many scientific domains. For example, phylogenetic tree reconstruction methods are used to infer the evolutionary 
history of different organisms, see \cite{durbin1998biological,steel2016phylogeny} and references therein. In machine learning, applications of latent tree models include human interaction recognition, medical diagnosis and classification of documents \cite{mourad2013survey, harmeling2010greedy,huang2017scalable}. 
  
As described in Section \ref{sec:setup}, in tree based graphical models, each node of the tree has an associated random variable. In many applications one can only observe the values at the terminal nodes of the tree, while the structure of the tree, as well as the values at the internal nodes, are unknown. 
Given a set of independent realizations of the observed variables, a common task is to infer the tree structure.
In phylogeny, the terminal nodes correspond to present-day species, also known as taxa, and the hidden nodes correspond to their common ancestors. Each species is described by an observed string of characters such as a DNA or protein sequence. The task is to infer a tree that models the evolutionary lineages of the observed organisms \cite{felsenstein2004inferring,delsuc2005phylogenomics,steel2016phylogeny}. 

Many algorithms have been developed to recover the latent tree structure from observed data. These include 
distance based methods such as 
the classic neighbor joining (NJ)  \cite{saitou1987neighbor} and UPGMA \cite{sokal1958statistical},  
maximum parsimony \cite{camin1965method,fitch1971toward}, maximum likelihood \cite{felsenstein1981evolutionary,guindon2003simple,stamatakis2006raxml,roch2006short} , quartets and meta trees \cite{anandkumar2011spectral,strimmer1996quartet,pearl1986structuring,snir2008short,rusinko2012invariant,jiang2001polynomial}, and Bayesian methods \cite{rannala1996probability}. 
Other approaches for tree recovery are based on a measure of statistical dependency between pairs of terminal nodes, see  \cite{pearl1986structuring,harmeling2010greedy}. 
As reviewed in \cite{yang2012molecular, john2003performance}, each of these strategies has different strengths and weaknesses.

\update{
We elaborate here on two approaches that are of particular relevance to our work.
}
The first is the neighbor joining algorithm, one of the most important methods used in phylogeny. Due to its simplicity and scalability, neighbor joining is widely used in practice, and often serves as a baseline when testing new methods for reconstruction of evolutionary trees \cite{lanciotti2016phylogeny,hajdinjak2018reconstructing,tamura2004prospects,gascuel2006neighbor}. For completeness, this approach is 
briefly outlined in Section \ref{sec:standard_nj}.
Several works investigated the theoretical properties of the neighbor joining algorithm \cite{atteson1999performance,gascuel2016stochastic,bryant2005uniqueness,gascuel2006neighbor,pauplin2000direct,mihaescu2009neighbor}.
Atteson \cite{atteson1999performance} studied its consistency and derived a sufficient condition for correct tree recovery. 
A different guarantee for exact recovery was derived in \cite{mihaescu2009neighbor}, by exploiting a link between NJ and quartet-based methods. 
%
As discussed in  \cite{lacey2006signal,susko2004inconsistency,strimmer1996accuracy}, to recover certain tree topologies or trees with a large number of terminal nodes, NJ may require a very large number of samples. 

\update{ 
A second relevant line of work includes methods based on invariant features \cite{cavender1987invariants,allman2007molecular}. One example is the \textit{Tree SVD} algorithm, derived by Eriksson in \cite{eriksson2005tree}.
In this algorithm, the tree is constructed using the spectral properties of a 
matrix called the \textit{flattening matrix}. Every element of this matrix contains the probability of observing one possible assignment of characters in the terminal nodes.
Since the number of possible assignments increases exponentially with the number of terminal nodes, applying this method to large trees is intractable.
The Tree SVD algorithm was modified in \cite{fernandez2016invariant} by averaging 
the row-normalized and the column normalized flattening matrices. For trees with four terminal nodes, it was shown to have similar performance to the maximum likelihood approach. In \cite{allman2017split} the flattening matrices were used to detect changes in the tree topology within a DNA sequence.
In Section \ref{sec:spectral_nj} we elaborate on the Tree SVD algorithm and its relation to our approach.
}

	 
\paragraph{Our contribution} 
In this work we derive spectral neighbor joining (SNJ), a novel method to reconstruct tree graphical models.  
Our approach, described in Section \ref{sec:spectral_nj} is based on the spectral structure of a similarity matrix between all pairs of observed nodes. 
The key property we use is the conditional independence of a node from the rest of the tree given the values of its immediate neighboring nodes. As we prove in Lemma \ref{lem:affinity_spectral}, this implies that certain matrices have a rank one structure, as in our previous works on latent variable models \cite{jaffe2015estimating,jaffe2016unsupervised,jaffe2018learning,parisi2014ranking}.
On the theoretical front, in Section \ref{sec:population} we prove the consistency of SNJ given an exact similarity matrix. Furthermore, in Theorem \ref{thm:attenson_equivalent} we derive a sufficient condition on the difference between the exact and estimated similarity matrices that guarantees perfect recovery of the tree.
Next, Lemma \ref{lem:tail_bound} provides a concentration of measure result on the estimated similarity matrix in the case of the Jukes-Cantor model, a popular model of sequence evolution \cite{jukes1969evolution}.
Subsequently, in Theorem \ref{thm:finite_sample} we combine these results and 
derive an explicit expression of the number of samples that suffice for SNJ to correctly recover the underlying tree under this model, with high probability. 
In Section \ref{sec:quartet_link} we show that our spectral criterion for joining subsets of nodes is closely linked to quartet based approaches for reconstructing trees. Loosely speaking, at each step SNJ merges the two 
subsets for which the sum of all quartet tests is most consistent with the tree topology. We compare the finite sample guarantee in Theorem \ref{thm:finite_sample} to guarantees obtained in quartet based methods \cite{erdHos1999few,anandkumar2011spectral}, and discuss the tradeoff between statistical efficiency and computational complexity when recovering trees. 

In Section \ref{sec:theoretical_comparison} we discuss the analogy between Theorem  \ref{thm:attenson_equivalent} and a classic result obtained by Atteson \cite[Theorem 4]{atteson1999performance} for correct tree reconstruction by NJ.
We compare the two sufficient conditions under the assumption of equal distances between all adjacent nodes. We show that for trees with a large diameter, our sufficient condition is considerably less strict than the
analogous one for classical NJ. Consequently, we anticipate that SNJ will recover the correct tree structure with fewer samples. 
In Section \ref{sec:simulations} we illustrate, via extensive simulations, the improved tree reconstruction accuracy of SNJ over NJ \cite{saitou1987neighbor}, Recursive Grouping \cite{choi2011learning} Tree SVD \cite{eriksson2005tree} and Binary forest \cite{harmeling2010greedy}, under a variety
of simulated settings.  

In summary, the proposed SNJ method 
shares several desirable properties with NJ, including consistency, scalability to large trees, and simplicity of implementation.
Furthermore, as we show both theoretically and via simulations, SNJ outperforms NJ and other methods under various scenarios of relevance to biological applications.
	

\section{Problem setup}\label{sec:setup}
	
	Let $\T$ be an unrooted bifurcating tree with $m$ terminal nodes. In such a tree, the leaves or terminal nodes each have a single neighbor, while internal nodes have three neighbors. 	
	We assume that each node of the tree has an associated discrete random variable
	attaining values  in the set $\{1,\ldots,d\}$. We denote by
	$\bm x = (x_1,\ldots,x_m)$ the vector of random variables at the $m$ observed terminal nodes of the tree, and by $h_A,h_B, \ldots$ the random variables at the internal nodes.
	We assume that all of these random variables form a Markov random field on $\T$. This means that the random variable at each node is statistically independent of the rest of the tree given the value of its neighbors.
	An edge $e(h_A,h_B)$ connecting a pair of adjacent nodes $(h_A,h_B)$ is equipped with two transition matrices of size $d \times d$,
	\begin{equation}\label{eq:transition_matrix}
	P_{h_A|h_B}(a,b) = \Pr[h_A=a | h_B=b] \qquad P_{h_B|h_A}(b,a) = \Pr[h_B=b | h_A=a].
	\end{equation}	

	The observed data is a matrix $X=[\bm x^{(1)},\ldots,\bm x^{(n)}] \in \{1,\ldots ,d\}^{m \times n}$, where $\bm x^{(j)}$ are i.i.d. realizations of the random variables at the $m$ terminal nodes of the tree. Each row in the matrix is a sequence of length $n$ that corresponds to one terminal node, see Figure \ref{fig:model_tree}. 
	For example, in phylogenetics, each row  corresponds to a different species, while each column corresponds to a different site in a DNA or protein sequence.
	The latent nodes in the tree correspond to the common ancestors of different subsets of the observed organisms, see \cite{durbin1998biological} and references therein.  
	

	
	Given the matrix $X$,
	the task at hand is to recover the structure of the tree $\T$. For the tree to be identifiable, we assume that for every pair of adjacent nodes $h_A,h_B$, the corresponding $d \times d$ stochastic matrices $P_{h_A|h_B}$ and $P_{h_B|h_A}$ defined in \eqref{eq:transition_matrix} are full rank, with determinants that satisfy
	\begin{equation} 
		\label{eq:assumption_1}
		0 < \delta < |P_{h_A|h_B}|< \xi < 1  \qquad 0 < \delta < |P_{h_B|h_A}|< \xi < 1 .
	\end{equation}
	Eq. \eqref{eq:assumption_1} implies that all edge transition matrices are invertible and are not permutation matrices. These are critical conditions for identifiability of the tree topology, see Proposition 3.1 in \cite{chang1996full} and \cite{mossel2005learning}. We remark that though our approach can be applied to recover the topology of  \textit{rooted trees} as well as unrooted ones, determining the location of the root requires additional assumptions, see \cite{smith1994rooting}.
		
	\section{The spectral neighbor joining algorithm}\label{sec:spectral_nj}
	
To introduce our novel spectral approach, 	in Section \ref{sec:definitions} we first review  known measures for similarity and distance between nodes in 
a latent tree model. For completeness,  
	Section  \ref{sec:standard_nj} briefly describes the standard neighbor joining algorithm. In Section \ref{sec:spectral_nj_der} we derive a new spectral criterion for neighbor joining and present our algorithm in detail. 
	
	\subsection{The symmetric affinity and distance matrices}\label{sec:definitions}
We denote by
	$P_{x_i|x_j}$ the stochastic matrix containing the distribution of $x_i$ given $x_j$. Under the tree model, $P_{x_i|x_j}$ is the product of the stochastic matrices of the edges along the directed path from $x_i$ to $x_j$. For example, in the tree shown in Figure \ref{fig:model_tree}, the hidden nodes on the path from $x_1$ to $x_3$ are $h_C$ and $h_A$. Thus,
	\[
	P_{x_1|x_3} = P_{x_1|h_C} P_{h_C|h_A}P_{h_A|x_3}.
	\]
	Several methods to reconstruct trees are based on a measure of similarity or distance between the observed nodes. Accordingly, we denote by $r(x_i,x_j)$ the \textit{symmetric affinity} between a pair of terminal or hidden nodes,
	\begin{equation}
	\label{eq:symmetric_affinity}
	r(x_i,x_j) = \sqrt{|P_{x_i|x_j}| \cdot |P_{x_j|x_i}|},  \qquad r(h_A,h_B) = \sqrt{|P_{h_A|h_B}| \cdot |P_{h_B|h_A}|}.
	\end{equation}
	Here $|P_{x_i|x_j}|$ denotes the determinant of the matrix $P_{x_i|x_j}$.
	Let $R \in \R^{m \times m}$ denote the symmetric affinity matrix between all pairs of terminal nodes,  
	\begin{equation}\label{eq:node_symmetric_affinity}
	R(i,j) = r(x_i,x_j) = \sqrt{|P_{x_i|x_j}| \cdot |P_{x_j|x_i}|}.
	\end{equation}
	Note that the symmetric affinity always falls within the range $[0,1]$. An important property of $R(i,j)$ is that it is {\em multiplicative} along the path between $x_j$ and $x_i$. For example, in Figure \ref{fig:model_tree}, the affinity between $x_1$ and $x_3$ is equal to
	\begin{align*}
	R(1,3) &= \sqrt{|P_{x_1|h_C}| \cdot |P_{h_C|h_A}| \cdot |P_{h_A|x_3}|} \sqrt{|P_{x_3|h_A}| \cdot |P_{h_A|h_C}| \cdot |P_{h_C|x_1}|}\\
	&=   \sqrt{|P_{x_1|h_C}| \cdot |P_{h_C|x_1}|}\sqrt{|P_{h_A|h_C}| \cdot |P_{h_C|h_A}|}
	\sqrt{|P_{x_3 |h_A}| \cdot |P_{h_A|x_3}|}\\
	&= r(x_1,h_C)r(h_C,h_A)r(h_A,x_3).
	\end{align*}
	This fact follows directly from the multiplicative property of determinants. The following transformation from the similarity measure \eqref{eq:node_symmetric_affinity} to a distance function between terminal nodes was proposed in \cite{chang1991reconstruction} and \cite{lake1994reconstructing}, 
	\begin{equation}\label{eq:symmetric_distance}
	D(i,j) = -\log r(x_i,x_j).    
	\end{equation}
	Eq. \eqref{eq:symmetric_distance}, known as the paralinear distance, was used in several distance based methods for reconstructing trees, see \cite{nei2000molecular,semple2003phylogenetics} and references therein. 
	Note that the $\log$ transformation in \eqref{eq:symmetric_distance} yields a distance measure between two observed nodes $x_i,x_j$ that is additive along the path connecting them. The additive property is a necessary condition for the consistency of any distance based method  \cite{chang1996full,chang1991reconstruction}.
	

	\subsection{Background: the neighbor joining algorithm}\label{sec:standard_nj}
	To motivate our approach, we first briefly describe the classical neighbor joining algorithm  \cite{saitou1987neighbor}. 
	The input to NJ is a matrix $\hat D \in \R^{m \times m}$ of estimated distances between observed nodes.
	NJ iteratively reconstructs the tree via the following procedure:
	\begin{enumerate}
		\item Compute the $Q$ criterion between all pairs,  		
		\begin{equation}\label{eq:q_criteria}
		Q(i,j) = (m-2)\hat D(i,j)-\sum_{k \neq \{i,j\}} \hat D(k,i) -\sum_{k \neq \{i,j\}} \hat D(k,j).
		\end{equation}
		
		\item Reconstruct the tree by repeating the following two steps, until there are three nodes left:
		\begin{itemize}
		\item[\RNum{1}] identify the pair $(\hat i,\hat j)$ that minimizes the $Q$ criterion,
		\[
		(\hat i,\hat j) = \argmin_{i,j} Q(i,j).
		\]
		\item[\RNum{2}] merge the pair $(\hat i,\hat j)$ into a single node $l$, and update the $Q$ criterion by
		\begin{equation}\label{eq:reduction}
		Q(k,l) = \frac{1}{2}(Q(k,\hat i)+Q(k,\hat j)) \qquad \forall k.
		\end{equation}
		\end{itemize}
	\end{enumerate}	
		
	The neighbor joining method is \textit{consistent}. If the estimated matrix $\hat D$ is sufficiently close to the true distance matrix $D$, the method is guaranteed to reconstruct the correct tree.  As proved by \cite{atteson1999performance}, 	
	a sufficient condition for recovering the tree is
	\begin{equation}\label{eq:atteson}
	\max_{i,j} |D(i,j) - \hat D(i,j)| \leq \frac{d_{min}}{2},
	\end{equation}
	where $d_{min}$ is the distance between the closest pair of adjacent nodes in the tree. The distance between adjacent (not necessarily terminal)  nodes is defined identically to the distance between terminal nodes given in Eq. \eqref{eq:symmetric_distance}. 
	

	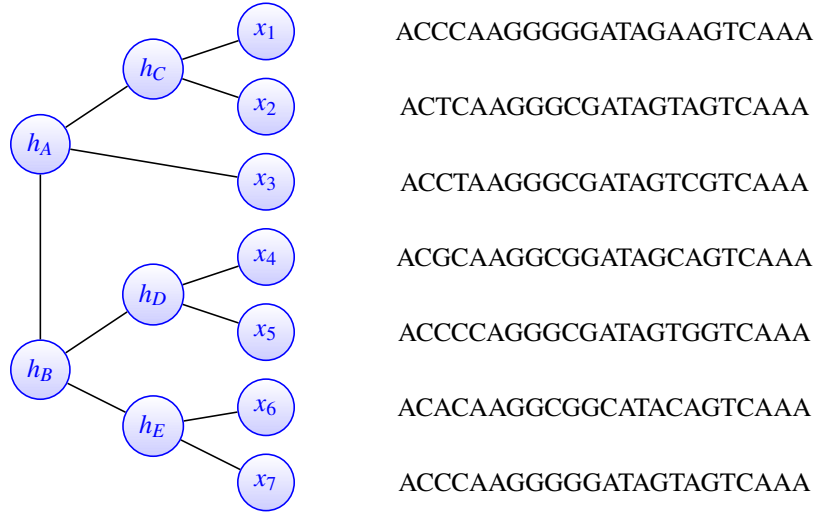
\begin{figure}[t]
	
	\begin {tikzpicture}[-latex ,auto ,node distance =4 cm and 5cm ,on grid ,
	semithick ,
	state/.style ={ circle ,top color =white , bottom color = blue!20 ,
		draw,blue , text=blue , minimum width = 0.75 cm}]	
	\node[state] (h2) at (1.5,4.5) {$h_A$};
	\node[state] (h3) at (1.5,1.5) {$h_B$};
	\node[state] (h4) at (3,5.5) {$h_C$};
	\node[state] (h5) at (3,2.5) {$h_D$};
	\node[state] (h6) at (3,0.75) {$h_E$};
	\node[state] (x1) at (4.5,6) {$x_1$};
	\node[state] (x2) at (4.5,5) {$x_2$};
	\node[state] (x3) at (4.5,4) {$x_3$};
	\node[state] (x4) at (4.5,3) {$x_4$};
	\node[state] (x5) at (4.5,2) {$x_5$};
	\node[state] (x6) at (4.5,1) {$x_6$};
	\node[state] (x7) at (4.5,0) {$x_7$};
	\path[-] (h2) edge node [above =0.15 cm,left = 0.15cm] {}(h3);
	\path[-] (h2) edge node [above =0.15 cm,left = 0.15cm] {}(h4);
	\path[-] (h3) edge node [above =0.15 cm,left = 0.15cm] {}(h5);
	\path[-] (h3) edge node [above =0.15 cm,left = 0.15cm] {}(h6);
	\path[-] (h4) edge node [above =0.15 cm,left = 0.15cm] {}(x1);
	\path[-] (h4) edge node [above =0.15 cm,left = 0.15cm] {}(x2);		
	\path[-] (h2) edge node [above =0.15 cm,left = 0.15cm] {}(x3);		
	\path[-] (h5) edge node [above =0.15 cm,left = 0.15cm] {}(x4);		
	\path[-] (h5) edge node [above =0.15 cm,left = 0.15cm] {}(x5);			
	\path[-] (h6) edge node [above =0.15 cm,left = 0.15cm] {}(x6);
	\path[-] (h6) edge node [above =0.15 cm,left = 0.15cm] {}(x7);				
	\node at (9,0) {ACCCAAGGGGGATAGTAGTCAAA};
	\node at (9,1) {ACACAAGGCGGCATACAGTCAAA};
	\node at (9,2) {ACCCCAGGGCGATAGTGGTCAAA};
	\node at (9,3) {ACGCAAGGCGGATAGCAGTCAAA};
	\node at (9,4) {ACCTAAGGGCGATAGTCGTCAAA};
	\node at (9,5) {ACTCAAGGGCGATAGTAGTCAAA};
	\node at (9,6) {ACCCAAGGGGGATAGAAGTCAAA};
	
\end{tikzpicture}

\caption{A tree with $m=7$ observed nodes. In a typical phylogenetic application, the data consists of a sequence of characters for every terminal node.} 
\label{fig:model_tree}
\end{figure}

    \subsection{A spectral criterion for neighbor joining}\label{sec:spectral_nj_der}


	To describe our approach, we use the terminology of unrooted trees provided by \cite{wilkinson2007clades}. We define a \textit{clan} of nodes 
	in $\T$ as a subset of nodes that can be separated from the rest of the tree by removing a single edge. For example, in Figure \ref{fig:model_tree} the subset $\{x_1,x_2,h_C\}$ forms a clan. 
	In the paper, we will sometimes  refer to the set of \textit{terminal nodes} of a clan such as $\{x_1,x_2\}$, as a clan.
	
	Let $A$ be a subset of $[m] = \{1,\ldots,m\}$. We denote by $x_A$ the set of corresponding terminal nodes $\{x_i\}_{i \in A}$.
	Let $A$ and $B$ be two disjoint subsets of $[m]$ such that $x_A$ and  $x_B$ each form  two different clans. We say that $x_A$ and $x_B$ are \textit{adjacent clans} if their union forms another, larger clan. Otherwise, we say that the clans are non-adjacent. 
	
    Equipped with these definitions, we describe the spectral neighbor joining approach. 
   	In contrast to previous methods that use the symmetric distance \eqref{eq:symmetric_distance} or other distance measures, our approach uses the symmetric affinity matrix between terminal nodes $R$ introduced in Section \ref{sec:setup}. 
    Let $A$ be a subset of $\{1,\ldots,m\}$ with size $|A|\geq 2$. We denote by $R^A$ the submatrix of $R$ of size $|A| \times (m-|A|)$ that contains all the affinities $R(i,j)$ with ${i \in A}$ and $j \in A^c$, where $A^c$ is the complement of $A$. Lemma \ref{lem:affinity_spectral} provides the theoretical foundation for our approach.
    \begin{lemma}\label{lem:affinity_spectral}
    The matrix $R^A$ is rank-one if and only if the subset $x_A$ is equal to the terminal nodes of a clan in $\T$.
	\end{lemma}
	By Lemma \ref{lem:affinity_spectral}, two nodes $x_i$ and $x_j$ are adjacent if and only if their affinities $r(i,k)$ to all other observed nodes are identical up to a multiplicative factor.
	This will be a crucial property in developing our spectral neighbor joining algorithm.  
		The proof of Lemma \ref{lem:affinity_spectral} relies on the following auxiliary lemma which is proven in the appendix. 
	\begin{restatable}[]
    {lemma}{lemequivalent}
	\label{lem:equivalent_statements}
	The following two statements are equivalent:
	\begin{enumerate}
	    \item The subset $x_A$ is equal to the terminal nodes of a clan in $\T$.
	    \item All quartets of terminal nodes $\{x_i,x_k,x_j,x_l\}$ where $i,k \in A$ and $j,l \in A^c$ have a topology as in Figure \ref{fig:model_subtree}, in which $(x_i,x_k)$ and $(x_j,x_l)$ are adjacent.
	\end{enumerate}
	\end{restatable}
	
	\begin{proof}[Proof of Lemma \ref{lem:affinity_spectral}]
	Suppose that $x_A$ consists of the terminal nodes of a clan in $\T$, and $x_B$ be the complementary subset. 
	Let $e(h_A,h_B)$ be the edge that separates the clan from the rest of the tree, so that all paths between nodes $x_A$ and $x_{B}$ pass through $e(h_A,h_B)$. 
	By the multiplicative property of $R(i,j)$,  for all $i \in A ,j \in B$
	\[
	R(i,j) = r(x_i,h_A)r(h_A,h_B)r(h_B,x_j).
	\]
	Let $\bm u_A$ denote a vector of size $|A|$, whose elements are the affinities between $h_A$ and $x_i$ for $i \in A$. Similarly, let $\bm u_B$ be a vector of size $m-|A|$ whose elements are the affinities between $h_B$ and $x_j$ for $j \notin A$. Then $R^A$ is equal to
	\begin{equation}\label{eq:affinity_struct}
	R^A = r(h_A,h_B)\bm u_A \bm u_B^T.
	\end{equation}
	Eq. \eqref{eq:affinity_struct} implies that $R^A$ is rank $1$. 
	
	Now suppose that $x_A$ does not equal the terminal nodes of a clan. By part 2 of Lemma  \ref{lem:equivalent_statements}, this implies that
	there is at least one quartet of nodes $x_i,x_k,x_j,x_l$ with $(i,j) \in A$ and $(k,l) \in B$ with a structure as in Figure \ref{fig:model_subtree}, where $x_i$ is closer to $x_k$ than to $x_j$. 
	Let $R_{ij}^{kl}$ be the $2 \times 2$ submatrix of $R^A$ that contains the pairwise affinities between $x_i,x_j$ and $x_k,x_l$. Then its determinant is 
	\begin{eqnarray}
		\label{eq:determinant_quartet}
	|R_{ij}^{kl}| &=& R(i,k)R(j,l)-R(i,l)R(j,k)  \nonumber \\
	&=& 
	r(x_i,h_A)r(h_A,x_k)r(x_j,h_B)r(h_B,x_l) 
	\nonumber \\	& & 
	- r(x_i,h_A)r(h_A,h_B)r(h_B,x_l)r(x_j,h_B)r(h_A,h_B)r(h_A,x_k)\notag \\
	& =& r(x_i,h_A)r(h_A,x_k)r(x_j,h_B)r(h_B,x_l)(1-r(h_A,h_B))^2. 
	\end{eqnarray}
	Combining Eq. \eqref{eq:assumption_1} with $r(x_i,x_j)$ in Eq. \eqref{eq:symmetric_affinity}  implies that all terms in Eq.  \eqref{eq:determinant_quartet} are bounded away from zero and from one. Hence,   $|R_{ij}^{kl}| \neq 0$ and so $R_{ij}^{kl}$ is full rank. Since $R_{ij}^{kl}$ is a submatrix of $R^A$, it follows that $R^A$ is at least rank two. 
	\end{proof}
	
	\begin{figure*}[t]
		\centering
		\begin {tikzpicture}[-latex ,auto ,node distance =4 cm and 5cm ,on grid ,
		semithick ,
		state/.style ={ circle ,top color =white , bottom color = blue!20 ,
			draw,blue , text=blue , minimum width = 0.75 cm}]

		\node[state] (h3) at (8.5,1.5) {$h_A$};
		\node[state] (h4) at (10,1.5) {$h_B$};
		\node[state] (x5) at (7,0.5) {$x_i$};
		\node[state] (x6) at (7,2.5) {$x_k$};
		\node[state] (x7) at (11.5,0.5) {$x_j$};
		\node[state] (x8) at (11.5,2.5) {$x_l$};
		\path[-] (h3) edge node [above =0.15 cm,left = 0.15cm] {}(h4);
		\path[-] (h4) edge node [above =0.15 cm,left = 0.15cm] {}(x7);
		\path[-] (h4) edge node [above =0.15 cm,left = 0.15cm] {}(x8);
		\path[-] (h3) edge node [above =0.15 cm,left = 0.15cm] {}(x5);
		\path[-] (h3) edge node [above =0.15 cm,left = 0.15cm] {}(x6);
	\end{tikzpicture}
	\caption{A subtree with $m=4$ observed nodes.}
	\label{fig:model_subtree}
\end{figure*}
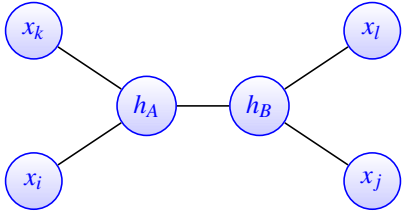
	
	Lemma \ref{lem:affinity_spectral} implies that given perfect knowledge of $R$, we can determine whether a given set of terminal nodes $x_A$ is equal to the terminal nodes of a clan by computing the rank of $R^A$. In practice, we typically only have a noisy estimate of the entries of $R$. Then, all submatrices of $R$ are full rank, though for true clans, the corresponding submatrices are approximately rank $1$. Accordingly, instead of the rank, our criterion for
	whether to join two subsets $x_{A_i}$ and $x_{A_j}$ is based on the second largest singular value of $R^{A_i \cup A_j}$, a matrix of dimension $(|A_i|+|A_j|) \times (m - |A_i|-|A_j|)$ that contains the affinities between terminal nodes in $x_{A_i \cup A_j}$ and the remaining terminal nodes. 
	We denote its second largest singular value by $\sigma_2(R^{A_i \cup A_j})$. Specifically, SNJ recovers the tree by performing the following operations:
    \begin{itemize}
        \item  Set $A_i = \{i\}$ for all $i$.  Compute a matrix $\Lambda \in \R^{m \times m}$ where
        \[
        \Lambda(i,j) = \sigma_2(R^{A_i \cup A_j}).
        \]
        \item Repeat the following two steps until only three subsets remain.
        \begin{itemize}
            \item[\RNum{1}] Identify the pair $(\hat i,\hat j)$ that minimizes $\Lambda(i,j)$, 

        \begin{equation}\label{eq:sigma_criterion}
        (\hat i,\hat j) = \argmin_{ij} \Lambda(i,j).
        \end{equation}

            \item[\RNum{2}] Merge $A_{\hat i},A_{\hat j}$ into a subset $A_l=A_{\hat i} \cup A_{\hat j }$. Update the $\Lambda$ criterion via

            \begin{equation}\label{eq:sec_eigv_criterion}
                \Lambda(k,l) = \sigma_2(R^{A_k \cup A_l}) \qquad \forall k.
            \end{equation}
                    \end{itemize}
        \end{itemize}
	As one can see, SNJ has a similar algorithmic structure to NJ, with the key difference being the use of the second singular value instead of the Q-criteria. Hence, it is interesting to compare the power of these two test statistics to distinguish between adjacent and non-adjacent terminal nodes. To this end, we generated a random Jukes-Cantor tree model with $m=512$ terminal nodes, associated with random variables with support of $d=4$ characters. The topology of the tree was generated by the following process: Given $m$ nodes, we merged a pair of random terminal nodes and replaced them with a single non-terminal node. Next, we merged another pair of random nodes, either terminal or non terminal, and again replaced them with a single non-terminal node. We continued this process until three nodes remained, and we connected them all to a non terminal node. We set the mutation rates between adjacent nodes to be $10\%$ and the number of realizations to be $n=500$.
	The left panel of Figure \ref{fig:second_eigenvalue_hist} shows the empirical distribution of $\log\sigma_2(R^{A_i \cup A_j})$ at the first SNJ iteration, 
	where  $A_i = \{i\}$ for all $i$. 
The right panel shows the empirical distribution of the NJ $Q$ criterion in Eq. \eqref{eq:q_criteria}. The red and blue lines correspond to pairs of adjacent and non-adjacent terminal nodes, respectively.  Comparing the two panels, we clearly see that adjacent pairs can be perfectly separated from non-adjacent pairs by their $\sigma_2$ values, whereas $Q$ values of adjacent and non-adjacent pairs have a significant overlap. As we will illustrate in Section \ref{sec:simulations}, 
this better separation of $\sigma_2$ vs. the $Q$-criterion allows SNJ to accurately reconstruct trees from fewer number of samples, where NJ fails. 

\update{
We note that $\sigma_2(R^{A_i \cup A_j})$ is not the only possible measure of how close $R^{A_i \cup A_j}$ is to a rank-1 matrix. An alternative measure is the Euclidean distance to the closest rank-1 matrix, computed by the sum of squares of all but the first singular value. The second singular value criterion is justified by Lemma \ref{lemma:rank2} in the following section, where we prove that if $x_{A_i}$ and $x_{A_j}$ are clans, then $R^{A_i \cup A_j}$ is \textit{at most} rank 2.  Thus, all non-zero singular values besides $\sigma_1$ and $\sigma_2$ are the result of noise in the similarity estimates, and should not be taken into account.
}

\update{
	\subsection{Heterogeneity of mutation rates}
	The problem setup presented in Section \ref{sec:setup} assumed a fixed rate of mutation across all sites in the sequence. In many applications, this assumption does not hold, and may lead to bias in estimating the distance between terminal nodes \cite{ArisBrosou1996impact}. 
	In biological applications, rate heterogeneity is commonly modeled using a gamma distribution or the related gamma-invariable model \cite{gufuli, intraspecifichetero}. 
	Many substitution models have variants that account for heterogeneity in the mutation rate along a sequence \cite{nei2000molecular,waddell1997general}.
	For example, the classic Jukes-Cantor model has a variant that takes into account heterogeneity in mutation rate, termed Gamma Jukes-Cantor. Similar to the homogeneous rate case, the distances computed by these models are additive along the tree, a key property for reconstruction of trees with distance based methods \cite{chang1991reconstruction}.
	Regardless of the assumed model, distance estimates can be transformed into   similarity estimates by inverting Equation \eqref{eq:symmetric_distance},
	\begin{equation} \label{eq:distance2affinity}
	    R(i,j) = e^{-D(i,j)},
	\end{equation}
	If the distance measure $D$ is additive, then the similarity scores obtained from Equation \eqref{eq:distance2affinity} maintain the key property of being multiplicative along the tree, as described in Section \ref{sec:definitions}. Thus, SNJ can be combined with any procedure for estimating distances. 
	In particular, SNJ can consistently recover trees with heterogeneity in mutation rates. In Section \ref{sec:simulations} we show empirically that SNJ outperforms NJ for data generated according to the Gamma model of heterogeneity in mutation rates.
}

\update{
\subsection{Related spectral methods}
A different spectral-based approach to reconstruct trees is the Tree SVD algorithm \cite{eriksson2005tree}, which similarly to SNJ merges subsets of terminal nodes based on a spectral criterion. 
The Tree SVD algorithm first estimates the probability of observing all $d^m$ possible patterns in the terminal nodes. For every partition $[m] = A \cup A^c$, these estimates are rearranged into a
\textit{flattening matrix} of size $d^{|A|} \times d^{m-|A|}$.  
Each row of the matrix contains the probabilities of all possible patterns of terminal nodes in $A^c$, with a fixed pattern for the terminal nodes in $A$. 
The key property of the flatenning matrix is that with the \textit{exact} (rather than estimated) probabilities, its rank is equal to $d$ if and only if $A$ corresponds to a clan in the tree. 
}

\update{
Though Tree SVD is consistent, 
it is impractical for large trees due to the size of the flattening matrix. In contrast, computing the SNJ similarity matrix can be done efficiently, as its dimension is equal to the number of terminal nodes.}

\update{
A second drawback of Tree SVD, outlined in \cite{allman2007molecular,allman2017split} is that it compares flattening matrices of different sizes. 
In \cite{allman2007molecular}, this fact was shown to cause a bias towards balanced trees. Potentially, this drawback is also relevant to SNJ, as 
the $\sigma_2$ criterion is compared for matrices of different sizes. One way to measure if an algorithm suffers from such a bias is to count the number of cherries - clans with two terminal nodes - in the reconstructed tree and in the original one. Figure \ref{fig:cherries} (right) shows the \textit{bias}, the number of cherries in the  trees estimated by SNJ and NJ minus the number of cherries in the ground truth, as a function of $m$. The trees were generated according to the birth death model. Figure \ref{fig:cherries} (left) shows the RF distance between the estimated trees and the ground truth. Though the recovered tree is not perfect, the results do not indicate any bias towards trees that are more balanced.    
}


	\begin{figure}		
	\begin{center}
	\includegraphics[width=0.47\textwidth,height=5.7cm]{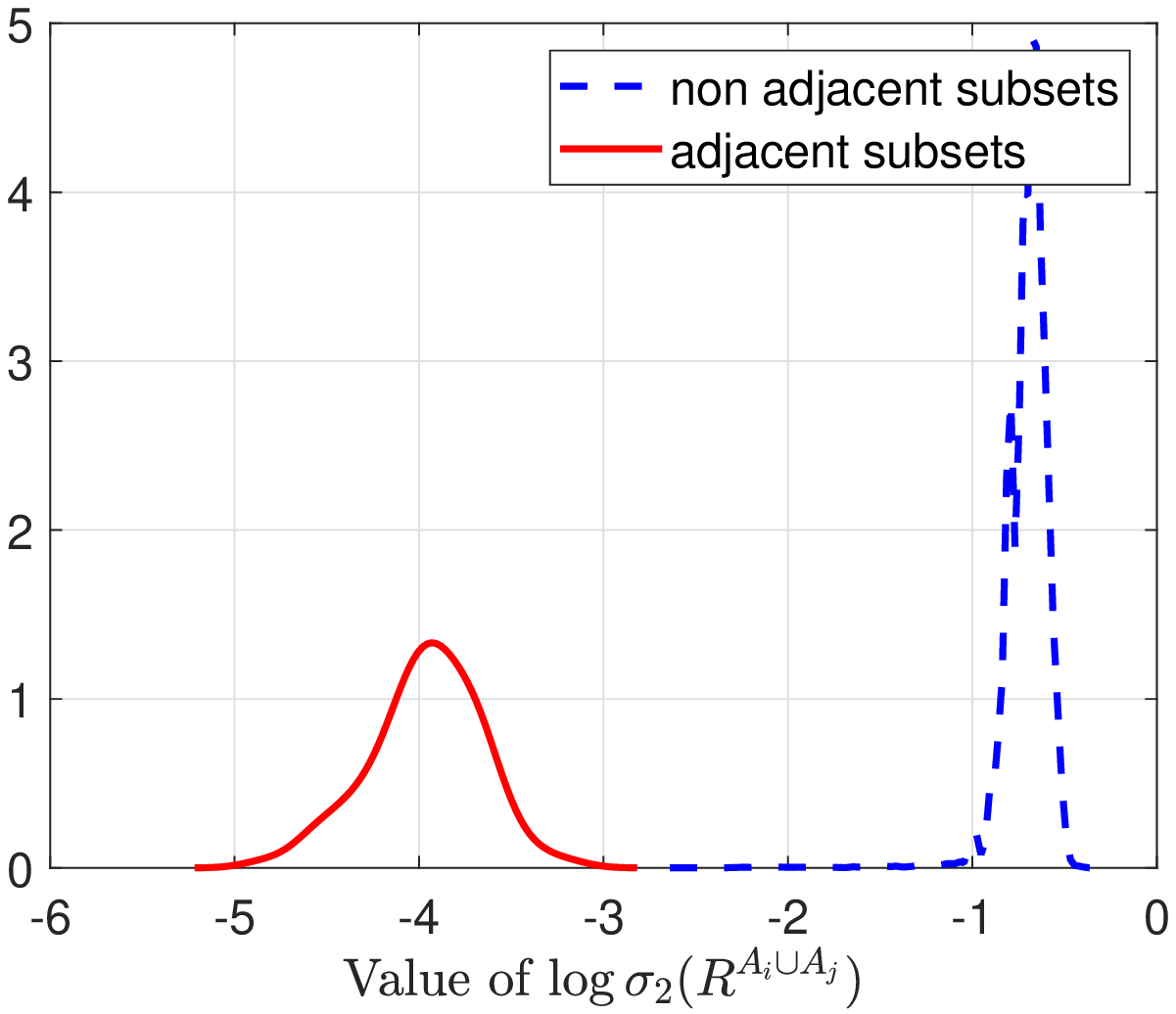}		
	\includegraphics[width=0.47\textwidth]{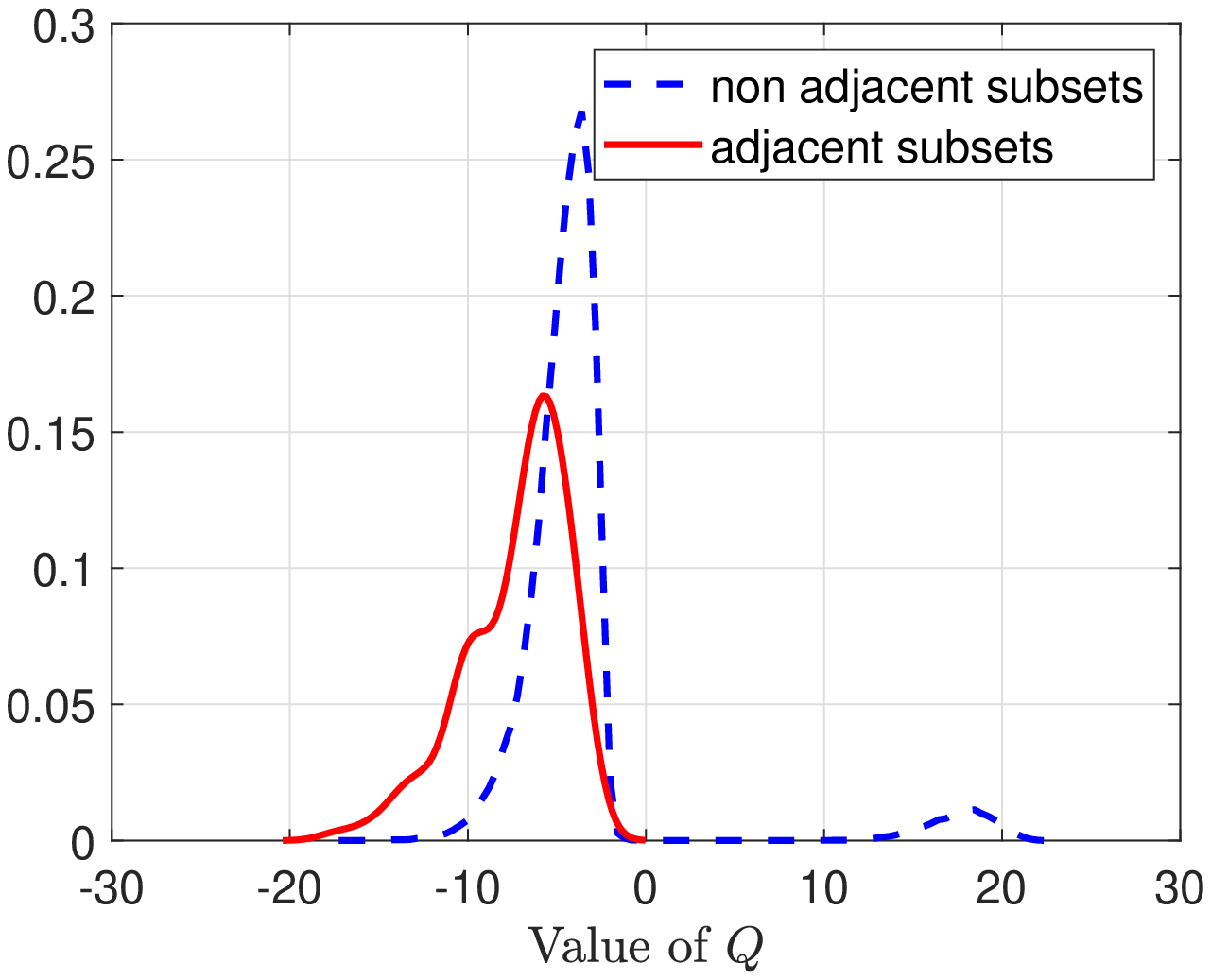}
	\end{center}	
		\caption{The red and blue lines are the empirical distributions of the $\sigma_2(R^{A_i \cup A_j})$ criterion from SNJ (left), and the Q criterion from NJ (right), for cases where $A_i,A_j$ are adjacent and non-adjacent pairs of singleton sets.}		\label{fig:second_eigenvalue_hist}
	\end{figure}
		
\begin{figure}[htbp]
	    \begin{subfigure}[b]{0.49\textwidth}
	    \includegraphics[width=1\textwidth]{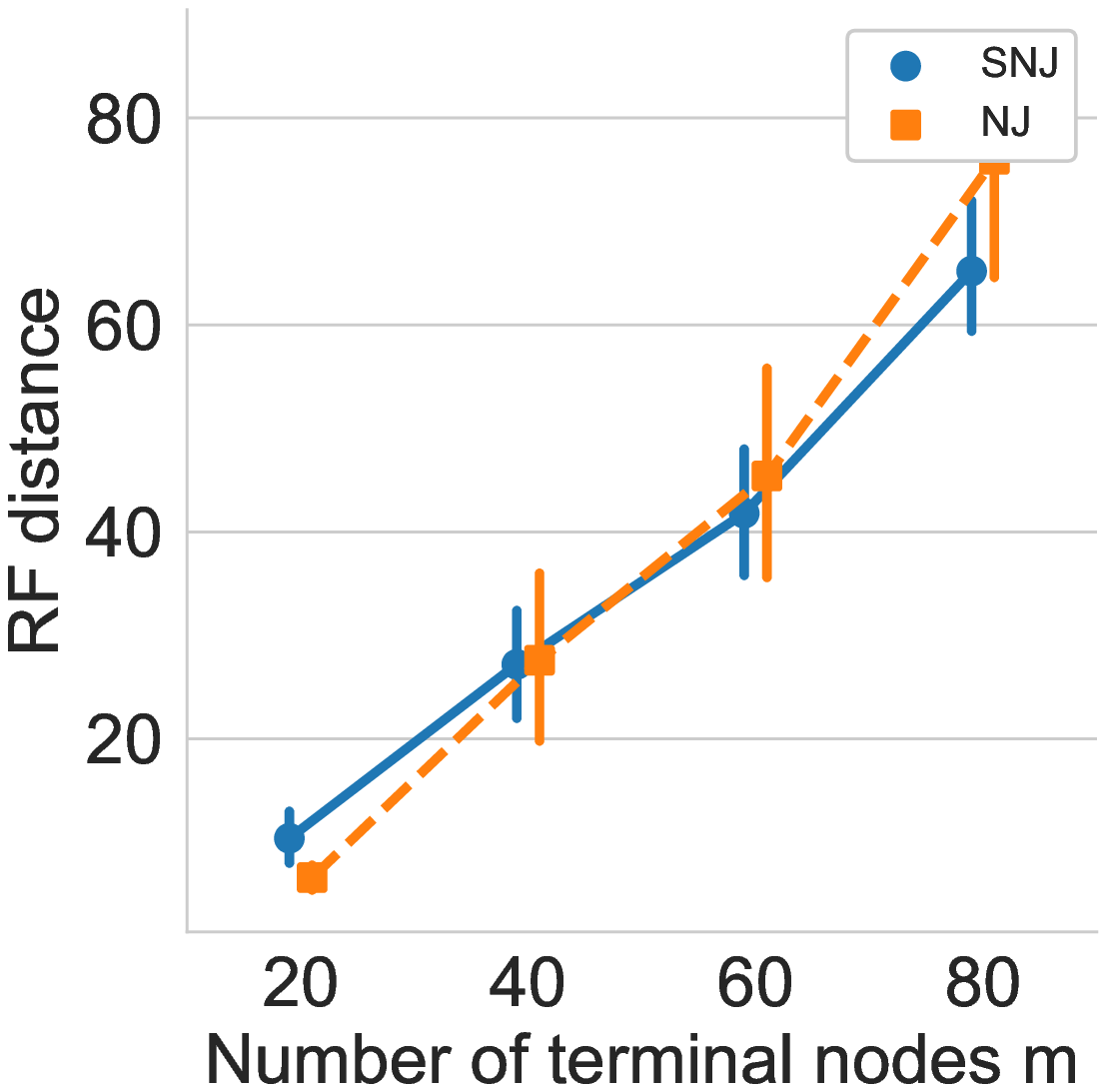}%
		\end{subfigure}
		~
		\begin{subfigure}[b]{0.49\textwidth}
		\includegraphics[width=1\textwidth]{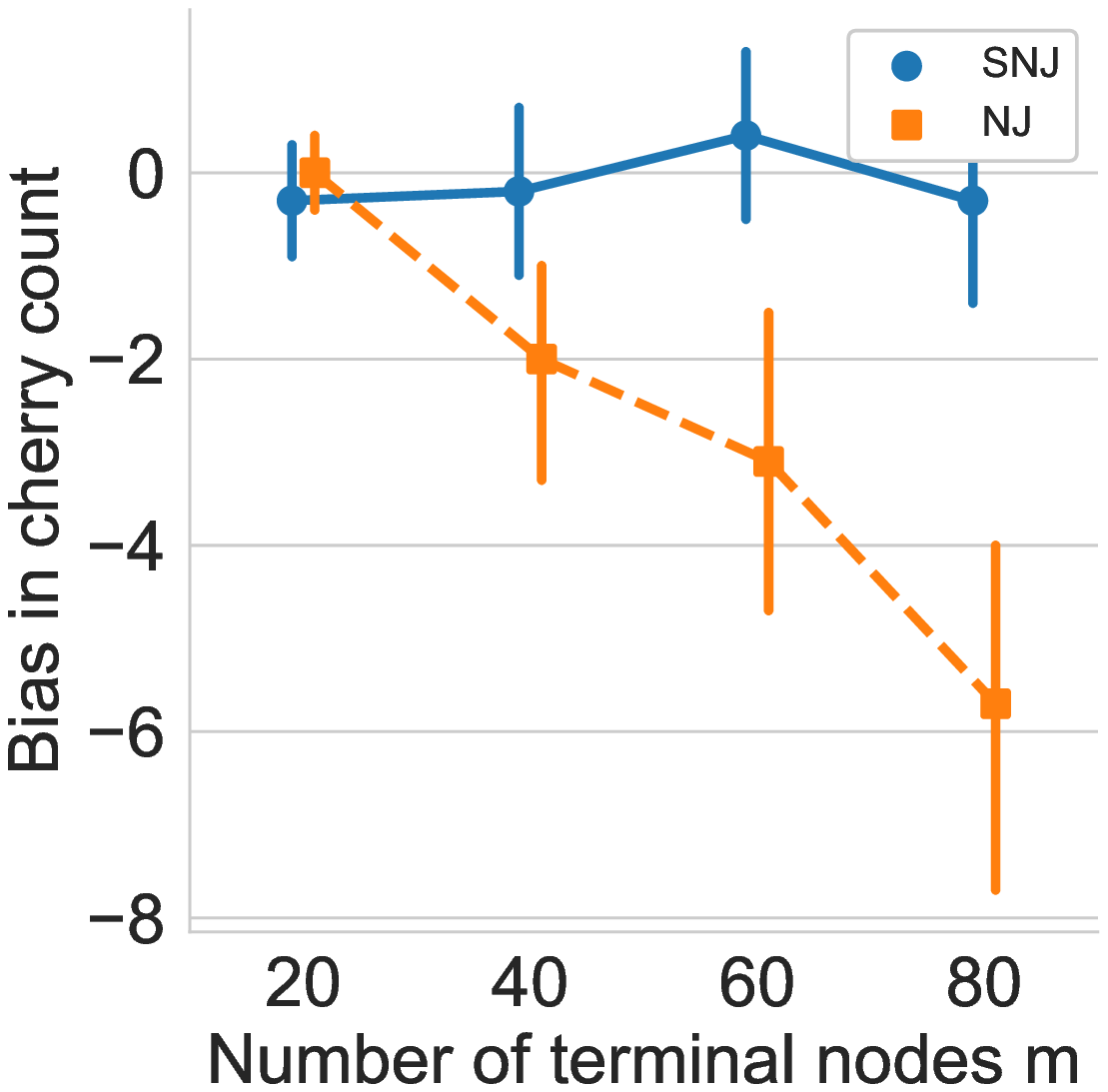}
		\end{subfigure}
			\vspace{-1.5\baselineskip}
		\caption{Simulation of trees generated according to the coalescent model. Left: The RF distance between the tree and its NJ and SNJ estimates. Right: The number of cherries in the estimated tree minus the number of cherries in the original tree. } 
		\label{fig:cherries}
	\end{figure}
		
\section{Analysis}\label{sec:analysis}

In this section we present a  theoretical analysis of the SNJ algorithm.
First, in section \ref{sec:population} we 
prove consistency of SNJ in the population setting where the similarity matrix $R$ is perfectly known, and assuming Eq. (\ref{eq:assumption_1}) holds. 
Next, we derive a sufficient condition
on the difference between the estimated and exact affinity matrices that guarantees correct tree reconstruction by SNJ. Finally, we derive an explicit expression for
the number of samples sufficient to guarantee exact tree reconstruction by SNJ with high probability 
under the Jukes-Cantor model. 
Proofs of auxiliary lemmas stated in this section appear in the appendix. 

\subsection{Consistency of SNJ in the population setting}\label{sec:population}

%
For SNJ to correctly recover the tree structure, at each iteration it must 
merge two adjacent clans of terminal nodes.
The following theorem characterizes the second eigenvalue  criterion in Eq. \eqref{eq:sec_eigv_criterion}, depending on whether two subsets are adjacent or not.
\begin{restatable}[]{theorem}{population}
\label{thm:population}
Let $C=A \cup B$, where $x_A$ and $x_B$ are disjoint subsets of terminal nodes such that each contains exactly the terminal nodes of a clan in $\T$. (i) If $x_A,x_B$ are adjacent clans then
\[
\sigma_2(R^C)=0.
\]
(ii) If $x_A,x_B$ are non adjacent clans then
\begin{equation}
	\label{eq:population}
\sigma_2(R^C) \geq 
\begin{cases}
\frac{1}{2} (2\delta^2)^{\log_2(m/2)}\delta(1-\xi^2) & \delta^2\ \leq 0.5, \\
\delta^3 (1-\xi^2) & \delta^2>0.5.
\end{cases}
\end{equation}
\end{restatable}

\noindent
For future use we define
\begin{equation}
	\label{eq:def_f}
f(m,\delta,\xi) = \frac{1}{2} (2\delta^2)^{\log_2(m/2)}\delta(1-\xi^2).
\end{equation}
Theorem \ref{thm:population} has several important implications, which we now discuss. First, 
as stated in the following corollary, it implies that SNJ is consistent. 
 \begin{corollary} 
 Let $\mathcal T$ be a tree which satisfies Eq. (\ref{eq:assumption_1}). 
 Then, SNJ with the exact affinity matrix $R$ is consistent and perfectly recovers $\mathcal T$. 
  \end{corollary}
 To see why the corollary is true, recall that at each iteration, SNJ merges two subsets with the smallest value of $\sigma_2(R^{A_i\cup A_j})$. By Theorem \ref{thm:population}, adjacent clans have $\sigma_2=0$, whereas if clans are non-adjacent, the second singular value corresponding to their union is strictly positive. Hence, given the exact affinity matrix, SNJ merges only adjacent clans until the whole tree has been perfectly reconstructed. 
 
A second important implication of Theorem \ref{thm:population} is that  
the bound in Eq. (\ref{eq:population}) yields insights into the ability of SNJ to correctly recover trees in the noisy setting, depending on the number of observed nodes $m$ and parameters $\delta,\xi$.  
Figure \ref{fig:bounds} shows the lower bound on $\sigma_2(R^C)$ for non-adjacent clans in Eq. (\ref{eq:population}), as a function of $\delta$ for $m = 4,8, 64, 128$, with $\xi=0.95$.  
Note that for $m=4$ the formulas for $\delta^2 \leq 0.5$ and $\delta^2>0.5$ coincide. 
For $\delta^2\leq 0.5$, which in a phylogenetic setting implies a high mutation rate, 
the bound decreases with a larger number of terminal leaves $m$. 
This implies that SNJ requires a higher number of samples to learn larger trees.


\begin{figure}[t]
    \centering
    \includegraphics[width=0.5\linewidth]{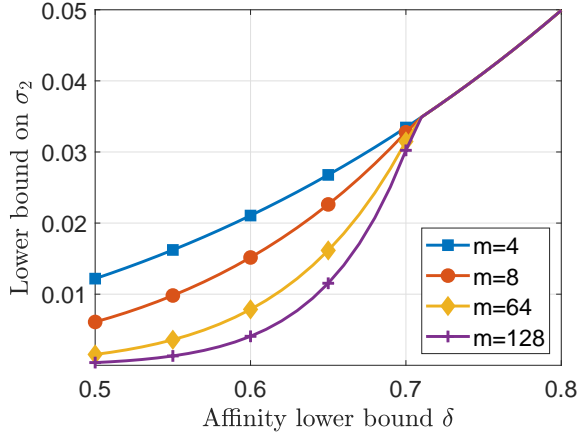}
    \caption{The lower bound \eqref{eq:population} on the second largest eigenvalue of $R^{A \cup B}$ for non adjacent clans $A,B$ as a function of the affinity lower bound $\delta$, at a fixed value $\xi=0.95$.}
    \label{fig:bounds}
\end{figure}

We remark that in general, 
the lower bounds in Eq. (\ref{eq:population}) are tight, up to a multiplicative factor of $1/\delta(1+\xi)$,
as described in the following lemma.  
\begin{lemma}
	\label{lem:tight_sigma_2}
	For $\delta^2 \leq 0.5$, there exists a tree and two non-adjacent clans $x_A,x_B$ 
	such that $\sigma_2(R^{A\cup B})=f(m,\delta,\xi) / \delta(1+\xi)$. 
	For $\delta^2>0.5$, there exists a tree with $m=4$ nodes for which 
	$\sigma_2(R^{A\cup B}) = 
	\delta^3 (1-\xi^2)\big/\delta(1+\xi)=
	\delta^2 (1-\xi)$. 
\end{lemma}

The first part of Theorem \ref{thm:population} follows directly from Lemma \ref{lem:affinity_spectral}. 
To prove the second part, we first introduce some notations and auxiliary lemmas. Let $x_A,x_B$ be two non adjacent  
clans in $\T$ and let $h_A,h_B$ be their corresponding root nodes. 
Since 
$x_A,x_B$ are not adjacent,
there are at least two additional hidden nodes on the path between $h_A$ and $h_B$. 
Let $h_{1},\ldots,h_{l}$ denote the $l$ hidden nodes on this path, see Fig. \ref{fig:rc_structure} for an example with $l=3$ intermediate nodes. 
We split the remaining $m-|A|-|B|$ terminal nodes to $l$ subsets as follows: Every terminal node in $(A \cup B)^c$ is assigned to the closest hidden node on the path between $h_A$ and $h_B$ (see Fig. \ref{fig:rc_structure}). 
The matrix $R^C$ can be rearranged in the following block structure,
\begin{equation}\label{eq:R_C}
R^C = 
\begin{bmatrix} R^A_1 & R^A_2 & \ldots & R_l^A\\ R^B_1 & R^B_2 & \ldots & R_l^B \end{bmatrix} =
\begin{bmatrix} R_1 & R_2 & \ldots & R_l\end{bmatrix},
\end{equation}
where $R_{i}^A$ is a matrix of $|A|$ rows with the pairwise affinities between the nodes in $x_A$ and the terminal nodes assigned to $h_i$. The matrix $R_{i}^B$ with $|B|$ rows is defined similarly. The matrix $R_i$ is the concatenation of $R_i^A$ and $R_i^B$. 
The following lemma shows that this block structure implies that the matrix $R^C$ has rank at most $2$.

\begin{figure}[t]
    \centering
    \includegraphics[width=0.7\textwidth]{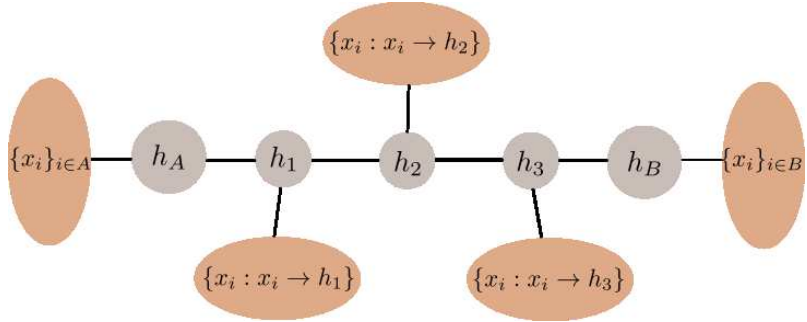}
    \caption{An example of two non adjacent clans $A$ and $B$. Every observed node in $(A \cup B)^c$ is assigned to the closest node on the path between $h_A$ and $h_B$.}
    \label{fig:rc_structure}
\end{figure}

\begin{restatable}[]{lemma}{rank2}
\label{lemma:rank2}
Let $R^C$ be the matrix of Eq. (\ref{eq:R_C}). Then $1\leq \text{rank}(R^C)\leq 2$.
\end{restatable}
\begin{proof}[Proof of Lemma \ref{lemma:rank2}]
Recall that $R^A$ denotes the affinity matrix between $x_A$ and $x_{A^c}$.  Under the assumption that $x_A$ contains the terminal nodes of a clan, by Lemma \ref{lem:affinity_spectral}, $R^A$ has rank one. 
The upper part of $R^C$ which includes $\{R_i^A\}_{i=1}^l$ is a submatrix of $R^A$ and hence has rank one as well. Similarly, the lower part of $R^C$, which includes $\{R_i^B\}_{i=1}^l$ is a submatrix of $R^B$ and also has rank one. The concatenation of two rank one matrices is at most rank two.
\end{proof}

Next, we present two auxiliary lemmas. The first concerns rank-2 matrices.

\begin{restatable}[]{lemma}{normseigs}
	\label{lemma:norms2eigs} 
	Let $M$ be a rectangular matrix with $1\leq rank(M)\leq 2$, and let $\sigma_2(M)$ be its second singular value. Then 
	\begin{equation}\label{eq:bound_quartet}
	\sigma_2(M)^2 \geq \frac12 \frac{\frob{M}^4 - \frob{M^TM}^2}{\frob{M}^2}.
	\end{equation}
\end{restatable}

The next auxiliary lemma expresses $\|R^C\|_F^4-\|(R^C)^TR^C\|_F^2$ 
in terms of the norms of the individual blocks $R_i^A,R_i^B$ of the matrix $R^C$.  
\begin{restatable}[]{lemma}{lemequality}\label{lem:equality_frobenius} 
	Let $R^C$ be the matrix of Eq. (\ref{eq:R_C}) with blocks $R_i^A$ and $R_i^B$. Then
	\begin{equation}\label{eq:frob_r_c}
	\frob{R^C}^4 - \frob{ (R^C)^TR^C}^2 = \sum_{j=1}^l \sum_{k=1}^l \big(\frob{R_j^A}\frob{R_k^B} - \frob{R_j^B}\frob{R_k^A}\big)^2. 
	\end{equation}
\end{restatable}

\begin{proof}[Proof of Theorem \ref{thm:population}, part (ii)]
Let $\bm u_A$ be the vector of affinities between $h_A$ and nodes in $x_A$,
and $\bm u_B$ be the vector of affinities between $h_B$ and $x_B$,
\begin{equation}\label{eq:u_a_def}
\bm u_A = \{r(x_i,h_A)\}_{i \in A}, \qquad \bm u_B = \{r(x_j,h_B)\}_{j \in B}.
\end{equation}
Similarly, let $\bm v_j$ be a vector of affinities between $h_j$ and the terminal nodes associated with it. By the multiplicative property of the affinity $r(x_i,x_j)$, the blocks $R_i^A$ and $R_j^B$ that are part of the matrix $R^C$ in \eqref{eq:R_C} have the following form,
\begin{equation}\label{eq:R_i_A}
R_i^A = \bm u_A r(h_A,h_i) \bm v_i^T \qquad R_i^B = \bm u_B r(h_B,h_i) \bm v_i^T,
\end{equation}
where $r(h_A,h_i)$ is the affinity between the hidden nodes $h_A$ and $h_i$.
The proof of the theorem is composed of the following three steps:
\begin{enumerate}
    \item Lower bound $\sigma_2(R^{A\cup B})$ in terms of $\|R_i^A\|_F$ and $\|R_i^B\|_F$. 
    \item Expand $\|R_i^A\|_F$ and $\|R_i^B\|_F$ in terms of   $\|\bm u_A\|,\|\bm u_B\|$ and $\|\bm v_i\|$.  
    \item Lower bound $\|\bm u_A\|,\|\bm u_B\|$ and $\|\bm v_i\|$ as a function of $m,\xi,\delta$.  
\end{enumerate}
\noindent {\bf Step 1:} Combining Lemmas \ref{lemma:rank2}, \ref{lemma:norms2eigs} and  \ref{lem:equality_frobenius} gives that
\[
\sigma_2(R^C) \geq \frac{\sum_{j=1}^l \sum_{k=1}^l \big(\frob{R_j^A}\frob{R_k^B} - \frob{R_j^B}\frob{R_k^A}\big)^2}{\|R^C\|_F^2}.
\] 
\noindent {\bf Step 2:} We express $\|R_i^A\|_F$ and $\|R_i^B\|_F$ in terms 
of $\|\bm u_A\|,\|\bm u_B\|$ and $\|\bm v_i\|$. This step follows directly from Eq. \eqref{eq:R_i_A}, 
\begin{equation}
	\label{eq:frob_R_j_a}
\|R_i^A\|_F = 
	r(h_A,h_i)\|\bm u_A\|  \|\bm v_i\| \qquad \|R_i^B\|_F = 
	r(h_B,h_i)\|\bm u_B\|  \|\bm v_i\|.
\end{equation}

\noindent {\bf Step 3:} The following 
auxiliary lemma provides a bound on $\|\bm u_A\|$ in terms of  $|A|$ and the affinity lower bound $\delta$.
\begin{restatable}[]{lemma}{normafinity}
\label{lemma:norm_u}
Let $x_A$ be equal to the terminal nodes of a clan in $\T$ and let $\bm u_A$ be the vector of Eq. (\ref{eq:u_a_def}). Then, 
\begin{equation}\label{eq:bound_norm_u}
\|\bm u_A\|^2 \geq 
\begin{cases} (2\delta^2)^{\log|A|} & \delta^2  \leq 0.5,\\
2\delta^2 & \delta^2>0.5.
\end{cases}
\end{equation}
\end{restatable}
Similar bounds hold for $\|\bm u_B\|^2$ and $\|\bm v_k\|^2$. 
Having described steps 1-3, we are now ready to conclude the proof of Theorem \ref{thm:population}. To this end, we use the following auxiliary lemma, which follows from steps 1 and 2. 
\begin{restatable}[]{lemma}{lemlowerboundua}
\label{lemma:lower_bound}
Let $C = A \cup B$, where $x_A$ and $x_B$ are non-adjacent clans in $\T$.  Then 
\begin{align}\label{eq:lower_bound_lemma}
\sigma_2(R_C)^2 &\geq  \frac{1}{4}\min\{\nrm{\bm u_A},\nrm{\bm u_B}\}^2 
\times \notag \\ 
&\min_j \min_{k \neq j} \nrm{\bm v_k}^2  (1-r(h_j,h_k)^2)^2
\min\{ \max_k r(h_A,h_k),\max_k r(h_B,h_k)\}^2 .
\end{align}
\end{restatable}
Next, we insert the lower bounds  in Eqs.
\eqref{eq:assumption_1} and \eqref{eq:bound_norm_u} into Eq.  \eqref{eq:lower_bound_lemma}.
For $\delta^2>0.5$, 
\begin{equation*}
\sigma_2^2(R^C) \geq \frac{1}{4}(2\delta^2) (2\delta^2) \delta^2 (1-\xi^2)^2 =
\delta^6 (1-\xi^2)^2.  
\end{equation*}
For $\delta^2 \leq 0.5$ we obtain,
\begin{equation}\label{eq:bound_norm_ab}
\sigma^2(R^C) \geq \frac14 (2\delta^2)^{\log|A|+\log|B|}\delta^2(1-\xi^2)^2 = \frac14(2\delta^2)^{\log|A||B|}\delta^2(1-\xi^2)^2.
\end{equation}
Since $x_A,x_B$ are non adjacent clans, there are at least two additional observed nodes that are not in $x_{A\cup B}$. It follows that $|A|+|B|\leq m-2$ and  hence $|A||B|<m^2/4$. Replacing $|A||B|$ with $m^2/4$ in \eqref{eq:bound_norm_ab} gives
\[
\sigma^2(R^C) \geq \frac{1}{4}(2\delta^2)^{\log(m^2/4)}\delta^2(1-\xi^2)^2=\frac{1}{4}(2\delta^2)^{2\log(m/2)}\delta^2(1-\xi^2)^2,
\]
which completes the proof of Theorem \ref{thm:population}.
\end{proof}




\subsection{Required number of samples for exact reconstruction}\label{sec:finite_sample}

We now focus on the finite sample setting, where we can only compute an approximate affinity matrix $\hat R$.
For NJ, the finite sample setting was addressed in \cite{atteson1999performance}, where NJ was proved to reconstruct the correct tree if the estimated distance matrix $\hat D$ satisfies Eq. \eqref{eq:atteson}. 
In the following theorem we derive an analogous result for SNJ. 
\begin{restatable}[]{theorem}{thmsufficientc}\label{thm:attenson_equivalent}
	Assume that Eq.  \eqref{eq:assumption_1} holds. Then a sufficient condition for spectral neighbor joining to recover the correct tree from $\hat R$ is that
	\begin{equation}\label{eq:snj_analoge}
	\|R-\hat R\| \leq 
	\begin{cases} \frac{f(m,\delta,\xi)}{2} & \delta^2 \leq 0.5 \\
	\frac12 \delta^3(1-\xi^2) & \delta^2 >0.5.
	\end{cases}
	\end{equation}
\end{restatable}

Next, we derive a concentration bound on the similarity matrix. This yields an upper bound on the number of samples required to obtain an estimated similarity matrix that satisfies Eq. \eqref{eq:snj_analoge}. 
For simplicity, the finite sample bound is derived for the Jukes-Cantor (JC) model, a popular model in phylogenetic inference, see \cite{felsenstein2004inferring}. Under the JC model, the probability over the $d$ states in all the nodes is uniform, and that the stochastic matrix between adjacent nodes $h_i,h_j$ is equal to
\[
\Pr(h_i|h_j)_{kl} = 
\begin{cases}
1-\theta(i,j) & k=l\\
\theta(i,j)/(d-1) & \text{otherwise,}
\end{cases}
\]
where $\theta(i,j)$ is the mutation rate between nodes $h_i$ and $h_j$. 
Under these assumptions, the affinity between terminal nodes in Eq. \eqref{eq:node_symmetric_affinity} simplifies to
\begin{equation}\label{eq:jc_similarity}
R(i,j) = \Big(1-\frac{d}{d-1}\theta(i,j)\Big)^{d-1}.
\end{equation}
By assumption \eqref{eq:assumption_1}  $R(i,j)$ is strictly positive, and hence $\theta(i,j)< (d-1)/d$.
Given $n$ i.i.d. realizations $\{x^l\}_{l=1}^n$ from the Jukes-Cantor model, we estimate $\hat \theta$ and $\hat R$ via

\begin{equation} \label{eq:jc_correction}
\hat{\theta}(i,j) = \min \Big\{\frac{1}{n} \sum_{l=1}^n \bm 1_{x^l_i \neq x^l_j},\frac{d-1}{d} \Big\} \qquad \hat R(i,j) = \Big(1 - \frac d{d-1} \hat \theta(i,j)\Big)^{d-1}.
\end{equation}
Applying SNJ to $\hat R$ estimated via Eq. \eqref{eq:jc_correction}, we have the following guarantee. 
\begin{theorem}\label{thm:finite_sample}
Assume the data was generated according to the Jukes-Cantor model. If the number of samples $n$ satisfies
\[ 
n \geq
\begin{cases}
 \frac{2d^2m^2}{f(m, \delta, \xi)^2}\log\Big(\frac{2m^2}{\epsilon}\Big) & \delta^2 \leq 0.5\\
 \frac{2d^2m^2}{\delta^6(1-\xi^2)^2}\log\Big(\frac{2m^2}{\epsilon}\Big) & \delta^2 > 0.5,
\end{cases}
 \]
where $f(m,\delta,\xi)$ was defined in \eqref{eq:def_f}, then SNJ will recover the correct tree topology with probability at least $1 - \epsilon$.
\end{theorem}
\noindent
To understand the dependency of $n$ on the number of terminal nodes $m$, we replace $f(m,\delta,\xi)$ with its definition   \eqref{eq:def_f}, and treat $\delta,\xi$ and $d$ as constants. For $\delta^2 \leq 0.5$,
\[ n = \Omega\Big(m^{4\log_2 (1/\delta)} \log(m / \epsilon)\Big). \]
If $\delta^2>0.5$, 
\[ n = \Omega\Big(m^2 \log(m / \epsilon)\Big). \]
Thus, up to a logarithmic factor, the number of samples required for an exact recovery of the tree is quadratic in $m$ for $\delta^2>0.5$, but can reach  $\Omega(m^\beta)$ with exponent $\beta \to \infty$ for very low values of $\delta$. 
Next, considering the dependence on $\xi$, Theorem \ref{thm:finite_sample} implies that $n$ scales as  $\Omega(1/(1-\xi^2))$. A high value of $\xi$ corresponds to a tree that has at least one very short edge, and is thus hard to reconstruct. A similar result appears in the guarantee derived by Atteson in Eq. \eqref{eq:atteson}, which depends on the minimal distance between adjacent nodes. In Section \ref{sec:simulations} we simulate trees with equal simiarity between all adjacent nodes such that $\delta = \xi$. The dependency of SNJ's performance on $\xi$ for these simulations is in accordance with this theoretical analysis.

The proof of Theorem \ref{thm:finite_sample} is based on the following auxiliary lemma, which states a concentration result on the estimated matrix $\hat R$.  
\begin{restatable}[]{lemma}{tailbound} \label{lem:tail_bound}
	Let $\hat R \in \R^{m \times m}$ be the matrix given by Eq. \eqref{eq:jc_correction}. Then
	\[ \Pr \Big( \nrm{\hat R - R} \leq t \Big) \geq 1 - 2m^2\exp\Big(-\frac{2nt^2}{d^2m^2}\Big). \]
\end{restatable}

\begin{proof}[Proof of Theorem \ref{thm:finite_sample}]

We prove the finite sample theorem by combining Theorem \ref{thm:attenson_equivalent} with the concentration bound on $\hat R$ in Lemma \ref{lem:tail_bound}.
For $\delta^2 \leq 0.5$, we replace $t$ with $f(m,\delta,\xi)/2$ in 
Lemma \ref{lem:tail_bound},
\begin{align*}
    \Pr \Big( \nrm{\hat R - R} \leq \frac{f(m, \delta, \xi)}2 \Big) 
    \geq 1 - 2m^2\exp\Big(-\frac{2n(f(m, \delta, \xi)/2)^2}{d^2m^2}\Big).
\end{align*}
Let $1 - \epsilon$ be a lower bound on this probability, such that
\begin{align*}
    1 - 2m^2\exp\Big(-\frac{2n(f(m, \delta, \xi)/2)^2}{d^2m^2}\Big) \geq 1 - \epsilon.
\end{align*}
Rearranging the above equation yields the following lower bound on $n$ in terms of $m,d$ and $\epsilon$,
\begin{align*}
    n \geq \frac{2d^2m^2}{f(m, \delta, \xi)^2} \log \Big( \frac{2m^2}{\epsilon} \Big),
\end{align*}
which concludes the proof for $\delta^2 \leq 0.5$. For $\delta^2>0.5$, we replace $f(m,\delta,\xi)$ with $\delta^3(1-\xi^2)$. 
\end{proof}

\subsection{Finite sample guarantees for alternative models of mutation}
\update{
The proof of Theorem \ref{thm:finite_sample} consists of two steps, corresponding to Theorem \ref{thm:attenson_equivalent} and Lemma \ref{lem:tail_bound}: (i) Given a sufficiently accurate similarity matrix, SNJ gives the correct tree, and (ii) An expression for the number of samples required for such an accurate estimate. 
}

\update{
The first step does not depend on any specific substitution model or any distribution of states at some node of the tree. 
The derivation of the second step, however, holds only for the Jukes-Cantor model, where 
a transition matrix $P_{x_i|x_j}$ is completely determined by a single mutation rate $\theta(i,j)$.
For this model, the affinity between terminal nodes simplifies to a polynomial in $\theta(i,j)$, see Eq. \eqref{eq:jc_similarity}.
}

With no assumptions on the structure of the transition matrices $P_{x_i|x_j}$, such a simplification is not possible. Here, we derive a bound that generalizes Lemma \ref{lem:tail_bound}, for unstructured transition matrices $P_{x_i|x_j}$. We make one simplifying assumption, that the transition matrices are symmetric with $P_{x_i|x_j}=P_{x_j|x_i}$. Thus, the similarity between terminal nodes $x_i,x_j$ in Eq. \eqref{eq:symmetric_affinity} simplifies to $R(i,j)=\text{det}(P_{x_i|x_j})$ where $\text{det}()$ denotes the matrix determinant.

Let $n_k(x_i)$ be the number of samples equal to state $k$ in terminal node $x_i$ and let $\gamma$ be equal to
\[
\gamma = \frac1n \min_{i \in [m],k \in [d]} n_k(x_i). 
\]
In words, $\gamma$ is the minimum proportion of one of the states $[d]$ in all terminal nodes $\{x_i\}_{i=1}^m$. The following lemma gives the number of samples required for an accurate estimate of the similarity matrix, for general transition matrices.

\begin{restatable}[]{lemma}{tailboundgeneral} 
\label{lem:tail_bound_general}
	Let $\hat R \in \R^{m \times m}$ be the matrix given by Eq. \eqref{eq:jc_correction}. Then
	\[ \Pr \Big( \nrm{\hat R - R} \leq t \Big) \geq 1 - 2d^2m^2\exp\Big(-\frac{2\gamma nt^2}{d^4m^2}\Big). 	\]
	
\end{restatable}

There are two important differences between the bounds in Lemmas \ref{lem:tail_bound} and \ref{lem:tail_bound_general}. First, the number of samples required is of order $O(d^4)$, rather than $O(d^2)$ in the JC model. This is expected due to lack of structure in the transition matrices. Second, if one or more of the states $\{1,\ldots,d\}$ appears with low frequency, than the required number of samples is increased, due to the dependency on $\gamma$. 

\section{The spectral criterion and a quartet based approach}\label{sec:quartet_link}
\update{In this section we show that the spectral criterion for merging subsets of terminal nodes is closely related to quartet based inference, a popular approach to recover latent tree models, see \cite{john2003performance,ranwez2001quartet,anandkumar2011spectral,pearl1986structuring,rhodes2019topological,snir2008short} and references therein. Quartet based inference is often  a two step procedure.
(i) estimate the topology for a large number of  quartets of terminal nodes. (ii) Based on the individual quartets, estimate the topology of the full tree.
} 

\update{
There are several approaches 
for the recovery of the full tree in step (ii). One approach is to find a tree that is consistent with the topology of the largest number of quartets, as estimated in step (i). The drawback of this approach is that in general it is a computationally hard problem, see \cite{day1986computational}. 
The quartet puzzling approach applies a greedy algorithm that first estimates the topology of a single quartet, and successively adds a single node at a time \cite{snir2008short,strimmer1996quartet}. 
An alternative method  \cite{rhodes2019topological} computes a pairwise distance matrix between all taxa based on the collection of quartets. The tree is then constructed via a distance based method.} 

\update{Mihaescu et. al. \cite{mihaescu2009neighbor} derived a link between quartet methods and NJ by proving a new guarantee for NJ. Let $ik;jl$ denote a quartet of terminal nodes $x_i,x_k,x_j,x_l$, with a topology as in Figure \ref{fig:model_subtree}, where the pairs $(x_i,x_k)$ and $(x_j,x_l)$ are siblings.
Informally, \cite{mihaescu2009neighbor} showed that NJ recovers the correct tree if the estimated distance matrix $D$ satisfies, for all quartets $ik;jl$, the following four point condition,
\begin{equation}\label{eq:four_point}
D(i,k)+D(j,l)  \leq \min\{D(i,j)+D(k,l),D(i,l)+D(j,k)\}.
\end{equation}
Here, we derive a similar connection between quartet based inference and SNJ. To this end, in Section \ref{sec:four_point_method} we define \textit{the quartet determinant} criterion and establish its relation to the four point condition in Eq.  \eqref{eq:four_point}. Next, in Section \ref{sec:quartet_and_merging_criterion} we prove that SNJ's spectral criterion is proportional to the normalized sum of squared quartet determinants. In Section  \ref{sec:max_quartet} we compare the finite sample guarantee in Theorem \ref{thm:finite_sample} to the guarantees obtained for quartet based methods in \cite{erdHos1999few,anandkumar2011spectral}. Based on the results of Section \ref{sec:quartet_and_merging_criterion}, we derive a quartet based approach by replacing SNJ's \textit{sum of squared quartets} merging criterion with a \textit{max quartet} criterion. With the new criterion, we prove that under the Jukes-Cantor model, the required number of samples for accurate reconstruction is similar to \cite{erdHos1999few,anandkumar2011spectral}. Comparing SNJ to the max-quartet approach, we discuss the trade off between statistical efficiency and computational complexity. 
}

\update{\subsection{The quartet determinant and the four point condition}\label{sec:four_point_method}
Let $w(ik ; jl)$ denote the following $2 \times 2$ determinant, 
\[
w(ik ; jl) = \left|\begin{matrix}
R(i,j) & R(i,l) \\
R(k,j) & R(k,l) \\
\end{matrix}\right|.
\]
By Lemma \ref{lem:affinity_spectral}, 
$w(ik;jl)=0$ if and only if the pairs $(x_i,x_k)$ and $(x_j,x_l)$ are siblings. Thus, one can use the value of $w(ik;jl)$  to determine the  topology of a quartet. 
Several works derived algorithms that recover latent tree models based on the quartet values $w(ik;jl)$. Anandkumar et. al. \cite{anandkumar2011spectral} developed spectral recursive grouping, which determines if  $x_i,x_k$ are siblings by computing $w(ik; jl)$ for all $j,l$.
To reconstruct a three layer tree,  \cite{jaffe2016unsupervised} applied spectral clustering to the following score matrix,
\[
S(i,k) = \sum_{k,l} |w(ik ; jl)|.
\]
Applying the spectral properties established in Lemma \ref{lem:affinity_spectral} to a tree of four nodes translates directly into the four point condition. If $x_i,x_k$ are siblings then $w(ik;jl)=0$ and hence
\begin{align*}
R(i,j)R(k,l) = R(i,l)R(k,j).
\end{align*}
Recall that by Eq.  \eqref{eq:distance2affinity} $D(i,j) = \log R(i,j)$. Taking logs on both sides yields
\begin{equation}\label{eq:four_point_equality}
    D(i,j) + D(k,l) = D(i,l) + D(k,j).
\end{equation}
In addition, $w(kl;ij)>0$, and hence 
\begin{equation}\label{eq:four_point_inequality}
D(i,k) + D(j,l) < D(i,j) + D(k,l).  
\end{equation}
Combining Eq. \eqref{eq:four_point_equality} and \eqref{eq:four_point_inequality} yields the four point condition in \eqref{eq:four_point}.
}

\update{
\subsection{The quartet determinant and the SNJ merging criterion}\label{sec:quartet_and_merging_criterion}
Let $A$ and $B$ be non-overlapping sets that are each equal to the observed nodes of a clan in a tree, and let $C=A \cup B$.
The following lemma relates the $\sigma_2(R^C)$ criterion for merging $A$ and $B$ and the sum over quartet values $w(ik;jl)$.
\begin{restatable}[]{lemma}{lemquartet}
For the population matrix $R$, The SNJ criterion $\sigma_2(R^C)$ can be written in terms of the quartet scores as follows,
\label{lem:quartets}
\[
\sigma_2(R^C)^2 = \frac{1}{4\sigma_1(R^C)^2}\sum_{i,k \in A \cup B}\;
\sum_{j,l \in (A \cup B)^c}
w(ik;jl)^2.
\]
\end{restatable}
Lemma \ref{lem:quartets} sheds new light on the spectral neighbor joining criterion for merging subsets of terminal nodes. At each iteration, SNJ merges two subsets $A,B$ that minimize a 
\textit{weighted quartet score},
where $w(ik;jl)^2$ serves as a measure of consistency between the quartet $i,j,k,l$ and the potential merge of $A$ and $B$. Thus, similar to quartet methods, the result of each step of SNJ is a merge that maximizes the consistency across all  possible quartets $i,k \in A \cup B$ and $j,l \in (A \cup B)^c$.
}

\update{
\subsection{The maximum quartet score and finite sample guarantees}\label{sec:max_quartet}
Inspired by Lemma \ref{lem:quartets}, we suggest the following criterion for merging subsets of terminal nodes, 
\begin{equation}\label{eq:max_quartet}
    M(A,B) = \max_{i,k \in A \cup B; j,l \notin A \cup B} |w(ik;jl)|.
\end{equation}
In words, we propose a different NJ type algorithm where we replace the \textit{sum of squared quartets} criterion in Lemma \ref{lem:quartets} with \textit{the max quartet} criterion. 
Clearly, the algorithm is consistent. Given the exact similarity matrix $R$, if $A \cup B$ forms a clan,
\[
w(ik;jl) = 0 \quad \forall(i \in A,k \in B,j,l \notin A,B),
\]
and hence $M(A,B)=0$. On the other hand, if $A \cup B$ does not form a clan, there is at least one pair of nodes $k,l \notin A \cup B$ such that for any pair $i,j \in A \cup B$ the topology is $ik;jl$, see illustration in Figure \ref{fig:max_quartet}. Let $h_1,h_2$  be the two nodes that 
split between  $(i,k)$ and $(j,l)$ as in the right panel of Figure \ref{fig:max_quartet}. The  criterion $|w(ij;kl)|$ is equal to
\begin{align*}
|w(ij ; k,l)| &= 
|R(i,k)R(j,l)-R(i,l)R(k,j)| \\
&=
R(i,h_1)R(k,h_1)R(j,h_2)R(l,h_2)(1-R(h_1,h_2)^2)>0.    
\end{align*}
Thus, if $A \cup B$ is not a clan, the criterion is proportional to the product of similarities between $h_1,h_2$ and the four taxa.
To further analyse this expression, we 
denote by $\text{depth}(\T)$ the 
\textit{depth of a tree} $\T$, which was defined in \cite{erdHos1999few} in the following way. 
For an edge $e(h_i,h_j)$, let $\bm x_A(h_i,h_j),x_B(h_i,h_j)$ denote a partition of the taxa 
induced by $e(h_i,h_j)$.
We denote by $g(h_i,h_j)$ the maximum between two values: (i) the number of edges from $h_i$ to the closest taxon in $\bm x_A$ and (ii) the number of edges from $h_j$ to the closest taxon in $\bm x_B$. 
Finally, the depth of a tree $\T$ is defined by
\begin{equation}\label{eq:depth}
\text{depth}(\T) = \max_{e(h_i,h_j) \in \T} g(h_i,h_j).
\end{equation}
The following theorem addresses the statistical efficiency of the max quartet NJ algorithm.
\begin{theorem}\label{thm:max_quartet}
Assume that the similarity between adjacent nodes is bounded as in Eq. \eqref{eq:assumption_1} and that the data is generated according to the Jukes-Cantor model. The number of samples sufficient for an accurate reconstruction of the tree by the max quartet approach scales as
\begin{equation}\label{eq:max_quartet_nsamples}
n = O\big(\log(m)/\delta^{4(\text{depth}(\T)+1)}\big).
\end{equation}
\end{theorem}
Similar guarantees to Theorem \ref{thm:max_quartet} were derived for quartet based approaches such as \cite{erdHos1999few} and \cite{anandkumar2011spectral}, and the Recursive Grouping algorithm \cite{choi2011learning}. In addition, \cite{erdHos1999few} showed that under two common tree distributions, the depth of almost all random trees scales as $O(\log \log m)$.
Indeed for such a tree, the guarantee in Theorem \ref{thm:max_quartet} is polynomial in $\log m$. However, for cases such as binary symmetric trees where $\text{depth}(\T) =  \log m$, then $n = O(\log(m)/ m^{4\log \delta})$, which, for $\delta^2<0.5$ is similar to the bound for SNJ in Theorem \ref{thm:finite_sample}.
}

\update{
The finite sample guarantee 
of $O(\log(m)/\delta^{4\text{depth}(T)})$ 
for quartet based methods such as \cite{erdHos1999few,anandkumar2011spectral} is
achieved 
by analyzing only \textit{short quartets}  \cite{erdHos1999few} where the distance between siblings is smaller than $2\text{depth}(\T)$. The drawback is that 
finding short quartets requires a costly search of all combinations of four terminal nodes. For example, \cite{erdHos1999few} prove that their computational complexity is $O(m^5\log m)$.
Similarly, the computation of the max quartet criterion requires, for some subsets, a search of $O(m^4)$, making the algorithm intractable for large trees. In contrast, computing the sum of quartets in Lemma \ref{lem:quartets}  can be done efficiently by computing the singular values. 
}

\update{
Figure \ref{fig:max_quartet_performance} shows the RF distance and runtime of both approaches on trees generated according to the coalescent model. The accuracy of SNJ is similar to the max-quartet approach, with a much lower runtime.   
}

\begin{figure*}[t]
		\begin {tikzpicture}[-latex ,auto ,node distance =4 cm and 5cm ,on grid ,
		semithick ,
		state/.style ={ circle ,top color =white , bottom color = blue!20 ,
			draw,blue , text=blue , minimum width = 0.75 cm}]
		\node[state] (x1) at (0,6) {};
		\node[state] (x2) at (0,5) {};	
		\node[state] (xk) at (0,3) {$x_k$};	
		\node[state] (x5) at (0,2) {};
		\node[state] (x6) at (0,1) {};	
		\node[state] (x7) at (0,0) {};	
		
		\node[state] (xi) at (0,4) {$x_i$};
		\node[state] (h1) at (1,5.5) {};
		\node[state] (h2) at (1,2.5) {};	
		\node[state] (h3) at (1,0.5) {};	
		
		\node[state] (ha) at (2,4.75) {$h_a$};
		\node[state] (hb) at (2,1.5) {$h_c$};
		
		\node[state] (ha1) at (3,3) {$h_1$};
		\node[state] (ha2) at (4,3) {$h_2$};
		
		\node[state] (hc) at (5,4.75) {$h_b$};
		\node[state] (hd) at (5,1.5) {$h_d$};
		
		\node[state] (h4) at (6,4) {};
		\node[state] (xj) at (7,5.5) {$x_j$};
		\node[state] (x9) at (7,2.5) {$x_l$};	
		\node[state] (x10) at (7,0.5) {};	
		
		\node[state] (x11) at (7,4.5) {};
		\node[state] (xl) at (7,3.5) {};	
		
		\draw [ultra thick, draw=black, fill=gray, opacity=0.2, rounded corners]
       (6.5,3) -- (7.5,3) -- (7.5,6) -- (6.5,6) -- cycle;
		
	    \draw [ultra thick, draw=black, fill=gray, opacity=0.2, rounded corners]
       (6.5,0) -- (7.5,0) -- (7.5,2.9) -- (6.5,2.9) -- cycle;	
		\node[draw] at (-1,5) {$A$};
		\node[draw] at (-1,2) {$C$};
		\node[draw] at (8,5) {$B$};
		\node[draw] at (8,2) {$D$};

		\path[-] (x1) edge node {}(h1);
		\path[-] (x2) edge node {}(h1);
		\path[-] (h1) edge node {}(ha);
		\path[-] (xi) edge node {}(ha);
		\path[-, color = red, line width = 0.075 cm] (ha) edge node {}(ha1);
		\path[-, color = red, line width = 0.075 cm] (ha1) edge node {}(ha2);
		\draw [ultra thick, draw=black, fill=gray, opacity=0.2, rounded corners]
       (-0.5,-0.5) -- (0.5,-0.5) -- (0.5,3.4) -- (-0.5,3.4) -- cycle;
		
		\path[-] (xk) edge node {}(h2);
		\path[-] (x5) edge node {}(h2);
		\path[-] (x6) edge node {}(h3);
		\path[-] (x7) edge node {}(h3);
		\path[-] (h2) edge node {}(hb);
		\path[-] (h3) edge node {}(hb);
		\path[-, color = red, line width = 0.075 cm] (hb) edge node {}(ha1);
		\draw [ultra thick, draw=black, fill=gray, opacity=0.2, rounded corners]
       (-0.5,3.5) -- (0.5,3.5) -- (0.5,6.5) -- (-0.5,6.5) -- cycle;

		\path[-, color = red, line width = 0.075 cm] (ha2) edge node {}(hc);
		\path[-, color = red, line width = 0.075 cm] (ha2) edge node {}(hd);

		\path[-] (hc) edge node {}(xj);
		\path[-] (hc) edge node {}(h4);
		\path[-] (hd) edge node {}(x9);
		\path[-] (hd) edge node {}(x10);
		\path[-] (h4) edge node {}(x11);
		\path[-] (h4) edge node {}(xl);
		
		\node[state] (x2i) at (9,4) {$x_i$};
		\node[state] (x2k) at (9,1) {$x_k$};
		\node[state] (x2j) at (11,4) {$x_j$};	
		\node[state] (x2l) at (11,1) {$x_l$};
		\node[state] (h21) at (9.5,2.5) {$h_1$};	
		\node[state] (h22) at (10.5,2.5) {$h_2$};	
		
		\path[-] (x2i) edge node {}(h21);
		\path[-] (x2k) edge node {}(h21);
		\path[-] (x2j) edge node {}(h22);
		\path[-] (x2l) edge node {}(h22);
		\path[-] (h21) edge node {}(h22);
	\end{tikzpicture}
	\caption{Computing the max quartet score for merging  subset $A$ and subset $B$. We can find at least one pair $k,l \in C \cup D$ and one pair  $i,j \in A \cup B$ that together satisfy two properties: (i) The topology of the quartet is as in the right subtree, and (ii) The number of edges from the splitting edge $e(h_1,h_2)$ to the quartets is at most $\text{depth}(\T)+1$.}
	\label{fig:max_quartet}
\end{figure*}
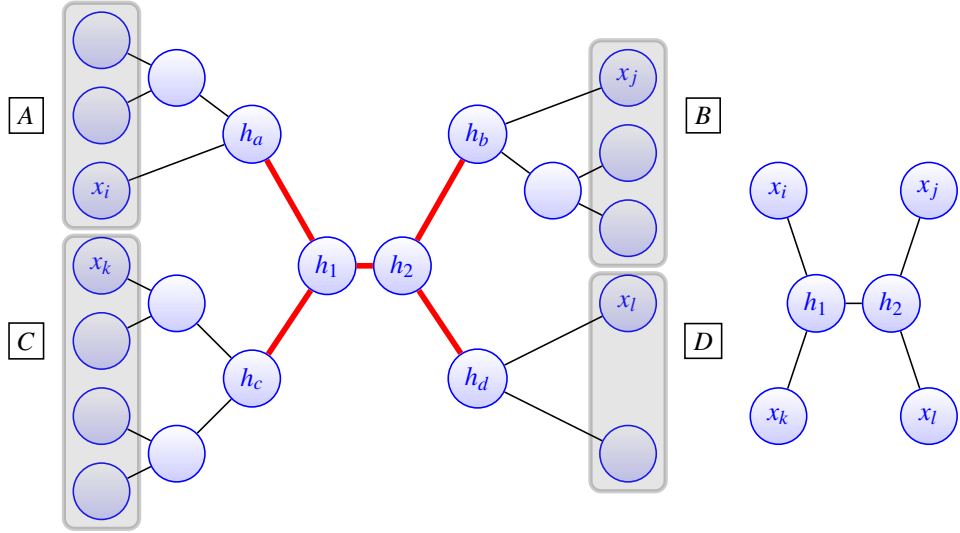

\begin{figure}
    \centering
    \includegraphics[width = 0.45\linewidth]{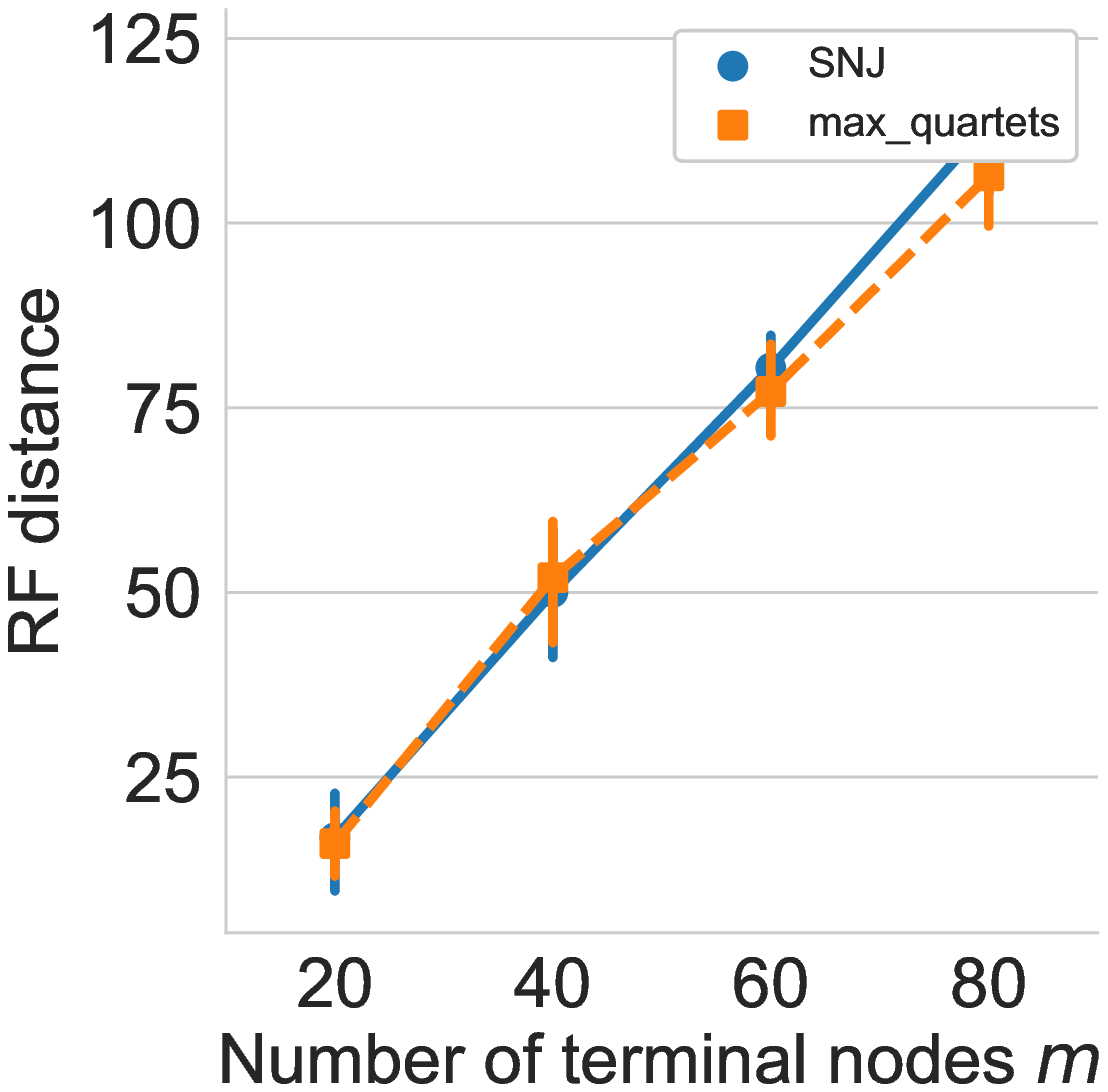}
    \includegraphics[width = 0.45\linewidth]{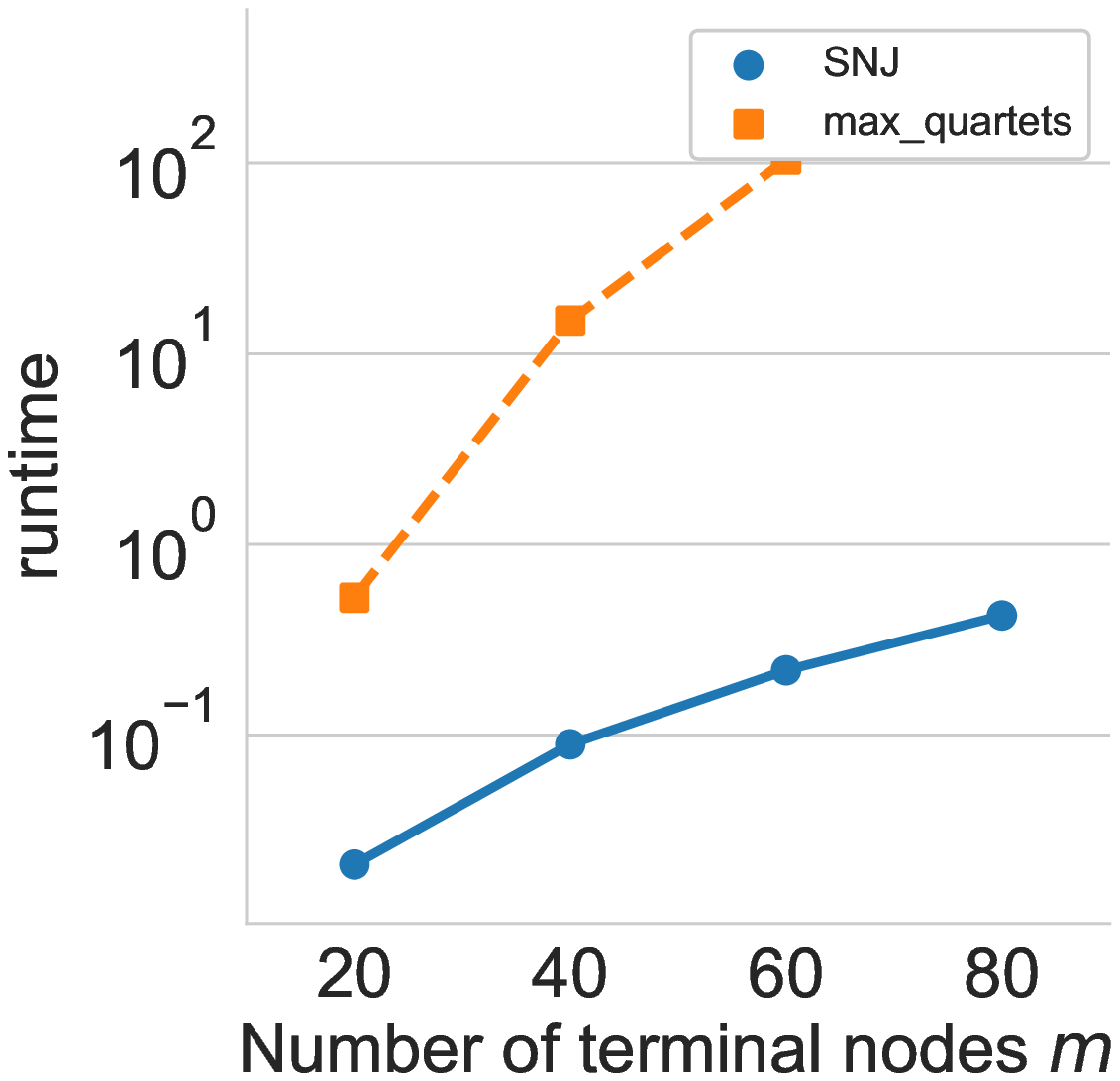}
    \caption{Comparison between SNJ and the max quartet method for recovering trees generated according to the coalescent model. The laft panel shows the RF distance  between the reconstructed and ground truth tree as a function of the number of terminal nodes $m$. The right panel shows the runtime of both methods.}
    \label{fig:max_quartet_performance}
\end{figure}

\section{Comparison between Atteson's NJ guarantee and its SNJ analogue}
\label{sec:theoretical_comparison}
Here, we make a qualitative comparison between the NJ sufficient condition for perfect tree recovery in Eq. \eqref{eq:atteson} and its SNJ analogue in Theorem \ref{thm:attenson_equivalent}.
We make two simplifying assumptions:
(i) the affinity between all adjacent nodes is equal to $\delta$, and (ii) $\delta^2\geq 0.5$. 
Our main insight is that to guarantee perfect recovery for trees with a large diameter, 
SNJ requires fewer samples than NJ.

The comparison between the two guarantees is done in two steps. First, 
in Eqs. \eqref{eq:sufficient_condition_nj} and \eqref{eq:nj_symmetric} we derive    requirements for the accuracy of $\hat R$, that are \textit{less strict} than Eq. \eqref{eq:atteson} (Atteson's condition). In other words, if $R$ satisfies \eqref{eq:atteson} , it also satisfies Eqs. \eqref{eq:sufficient_condition_nj} and \eqref{eq:nj_symmetric}. 
Then, these requirements are  compared to  Theorem \ref{thm:attenson_equivalent}.

Under the assumption that the similarity between all adjacent nodes is $\delta$, the NJ sufficient condition \eqref{eq:atteson} simplifies to
\begin{equation}\label{eq:aatteson_simplified}
|\log \hat R(i,j) - \log R(i,j) | = \Big|\log \frac{\hat R(i,j)}{R(i,j)}\Big| \leq -\frac{\log \delta }{2} = \log \delta^{-0.5} \qquad \forall i,j.
\end{equation}
Taking an exponent on both sides and simple algebraic manipulations give
\begin{equation}\label{eq:nj_sym_condition}
(1-\delta^{-0.5})R(i,j)<R(i,j) - \hat R(i,j) <(1-\delta^{0.5})R(i,j) \qquad \forall i,j.
\end{equation}
Since $0<\delta<1$, if Eq. \eqref{eq:nj_sym_condition} holds, then
\begin{equation} \label{eq:sufficient_condition_nj}
|R(i,j)-\hat R(i,j)|\leq \delta^{-0.5} R(i,j) \qquad \forall i,j.
\end{equation}
Let $\text{diam}(\T)$ denote the diameter of $\T$, defined as the maximal number of edges between a pair of terminal nodes. Let $ i^\ast,j^\ast$ be a pair of terminal nodes with $\text{diam}(\T)$ edges on the path between them such that $R(i^\ast,j^\ast)= \delta^{\text{diam}(\T)}$.
The requirement in Eq. \eqref{eq:sufficient_condition_nj} is for all pairs $i,j$, and hence 
a necessary condition for 
$\hat R(i^\ast,j^\ast)$ is 
\begin{equation}\label{eq:nj_symmetric}
|R(i^\ast,j^\ast)-\hat R(i^\ast,j^\ast)|\leq \delta^{\text{diam}(\T)-0.5}.
\end{equation}

\vspace{0.5em}

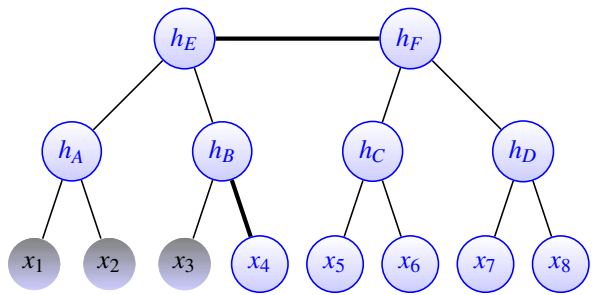
\begin{figure*}[t]
		\centering
		\begin {tikzpicture}[-latex ,auto ,node distance =4 cm and 5cm ,on grid ,
		semithick ,
		state/.style ={ circle ,top color =white , bottom color = blue!20 ,
			draw,blue , text=blue , minimum width = 0.75 cm}
		]
        \node[state] (he) at (3.5,3) {$h_E$};
        \node[state] (hf) at (6.5,3) {$h_F$};
		\node[state] (ha) at (2,1.5) {$h_A$};
		\node[state] (hb) at (4,1.5) {$h_B$};
		\node[state] (hc) at (6,1.5) {$h_C$};
		\node[state] (hd) at (8,1.5) {$h_D$};
		\node[circle, bottom color = blue!20] (x1) at (1.5,0) {$x_1$};
		\node[circle, bottom color = blue!20] (x2) at (2.5,0) {$x_2$};
		\node[circle, bottom color = blue!20] (x3) at (3.5,0) {$x_3$};
		\node[state] (x4) at (4.5,0) {$x_4$};
		\node[state] (x5) at (5.5,0) {$x_5$};
		\node[state] (x6) at (6.5,0) {$x_6$};
		\node[state] (x7) at (7.5,0) {$x_7$};
		\node[state] (x8) at (8.5,0) {$x_8$};
		
		\path[-] (he) edge node [above =0.15 cm,left = 0.15cm] {}(ha);
		\path[-,ultra thick] (he) edge node [above =-0.45 cm,left = 0cm] {}(hf);
		\path[-] (he) edge node [above =0.15 cm,left = 0.15cm] {}(hb);
		\path[-] (hf) edge node [above =0.15 cm,left = 0.15cm] {}(hc);
		\path[-] (hf) edge node [above =0.15 cm,left = 0.15cm] {}(hd);
		\path[-] (ha) edge node [above =0.15 cm,left = 0.15cm] {}(x1);
		\path[-] (ha) edge node [above =0.15 cm,left = 0.15cm] {}(x2);
		\path[-] (hb) edge node [above =0.15 cm,left = 0.15cm] {}(x3);
		\path[-,ultra thick] (hb) edge node [above =0.15 cm,left = 0.35cm] {}(x4);
		\path[-] (hc) edge node [above =0.15 cm,left = 0.15cm] {}(x5);
		\path[-] (hc) edge node [above =0.15 cm,left = 0.15cm] {}(x6);
		\path[-] (hd) edge node [above =0.15 cm,left = 0.15cm] {}(x7);
		\path[-] (hd) edge node [above =0.15 cm,left = 0.15cm] {}(x8);
	\end{tikzpicture}
	\caption{
	A perfect binary tree model with $m=8$ terminal nodes.
	For the proof of Lemma \ref{lem:equivalent_statements}, the terminal nodes in $x_A$ are colored in darker shade of gray. The thick edges form the minimal set that separates $x_A$ from $x_{A^c}$. The quartet $i = 1, k = 3,j = 4$ and $l = 5$ satisfies  $i,k \in A, j,l \in A^c$ but its topology is not as in Figure \ref{fig:model_subtree}.}
	\label{fig:tightness}
\end{figure*}

Next we recall SNJs theoretical guarantee in Theorem \ref{thm:attenson_equivalent}.
In our setting $\delta=\xi$ and $\delta^2 \geq 0.5$, hence SNJ recovers the tree if
\begin{equation}\label{eq:snj_simple}
\|R-\hat R\|\leq \frac12 \delta^3(1-\delta^2).    
\end{equation}
We point out two differences between Eq. \eqref{eq:snj_simple} and the corresponding NJ requirements in   \eqref{eq:sufficient_condition_nj} and \eqref{eq:nj_symmetric}. 
First, the inequality in Eq. \eqref{eq:snj_simple} is on the spectral norm, while in \eqref{eq:sufficient_condition_nj} it is on every element in the similarity matrix.
Second, the requirement for SNJ does not depend on the number of terminal nodes $m$ or the tree topology. In contrast, the NJ guarantee requires an accuracy of order $O(\delta^{\text{diam}(\T)})$.






 Let us consider two extreme cases. For trees similar to the caterpillar tree, the diameter is of order $O(m)$. In this case 
 the entries $\hat R(i,j)$ must be extremely accurate as the
 right hand side in Eq. \eqref{eq:nj_symmetric} decays \textit{exponentially} in $m$, a significantly stricter condition than for SNJ. 
At the other end, consider a tree similar to the binary symmetric tree, with a diameter of $B\log m$, for some constant $B$.   
In this case, the required accuracy in Eq. \eqref{eq:nj_symmetric} is of order $O(m^{B\log \delta})$. This condition is comparable to SNJ for low values of $B$ and high values of $\delta$ which corresponds, respectively, to trees with a small diameter and low mutation rate. For cases with high mutation rate, or if $B$ is large, we expect SNJ to have an advantage over NJ. 

In Figures \ref{fig:sim_comparison_cat_11},\ref{fig:sim_comparison_cat_12} we compare SNJ to NJ for caterpillar trees with $\text{diam}(\T)=m-1$. The results show that the SNJ is considerably more accurate than NJ for this setting. In Figures \ref{fig:sim_comparison_bin_11},\ref{fig:sim_comparison_bin_12} we compare SNJ to NJ to the binary symmetric tree. Here, the advantage of SNJ is not as significant as in the case of the caterpillar tree, but increases with higher mutation rate. 
Thus, the simulation results match the qualitative comparison of the two guarantees. A more rigorous comparison between the two methods 
may be an interesting direction for future research.

\section{Simulation results}\label{sec:simulations}
We compare the performance of SNJ to the following methods: (i) standard neighbor joining, equipped with the log-determinant distance (ii) Recursive Grouping (RG) \cite{choi2011learning} (iii) the Binary Forrest algorithm \cite{harmeling2010greedy} and (iv) the Tree SVD algorithm  \cite{eriksson2005tree}.
The algorithms are tested on the following tree models: (i) perfect binary trees with equal similarity between all adjacent nodes, and (ii) caterpillar trees, where the non terminal nodes form a path graph.
Due to the prohibitive runtime of some of these methods when applied to large trees, we
divided the simulation section to three parts:
\begin{enumerate}
    \item Comparing  SNJ and NJ for large trees and $d=4$ states. 
    \item Comparing SNJ, NJ and Recursive Grouping for  medium sized trees and $d=4$ states. For this part, in addition to perfect binary and caterpillar trees, we test the methods on trees generated according to Kingman's coalescent model \cite{wakeley2009coalescent}, a common model in phylogeny.  
    \item Comparing SNJ,NJ, Tree SVD and Binary Forrest for  small trees and $d=2$ states.
    \item Comparing SNJ and NJ for data generated according to the Gamma model of heterogeneity in mutation rate along a sequence.
\end{enumerate}
In all experiments, the transition matrices between adjacent nodes follow the Jukes-Cantor model.
The code for SNJ and scripts to reproduce our results can be found at\\ \url{https://github.com/NoahAmsel/spectral-tree-inference}. All simulations were done with the Python phylogenetic computing library Dendropy \cite{sukumaran2010dendropy}. The accuracy of a recovered tree is evaluated by the Robinson-Foulds (RF) distance \cite{estabrook1985comparison}, a popular measure for comparison between trees. The RF distance between two trees $\T_1$ and $\T_2$ counts the number of partitions in $\T_1$ that are not present in $\T_2$ and the number of partitions in $\T_2$ not present in $\T_1$. 

\textbf{Comparison to NJ for large trees and $d=4$  states.}
Figure \ref{fig:sim_comparison_bin_11} shows, for the case of a perfect binary tree with $m=512$ terminal nodes, the RF distance  between the tree and its NJ and SNJ estimates, as a function of the sequence length $n$. The similarity between adjacent nodes is $\delta = 0.85,0.9$. The results are averaged over $5$ realizations of the tree model. 
As expected from the theoretical analysis in Section \ref{sec:theoretical_comparison}, the advantage of SNJ over NJ increases for trees with high mutation rates. 

Next, we consider  caterpillar trees. In general, these trees are considered more challenging to recover than balanced ones, see \cite{lacey2006signal}. As shown in Figure \ref{fig:sim_comparison_cat_11}, the advantage of SNJ over NJ, for both high and low mutation rates is much more apparent in these trees compared to the perfect binary trees.  
Figure \ref{fig:sim_comparison_bin_12} and \ref{fig:sim_comparison_cat_12}
show the RF distance as a function of the number of terminal nodes $m$, 
on perfect binary and caterpillar trees, respectively. 
The number of samples $n$ is fixed to $400$ and $800$ for the binary and caterpillar trees, respectively and  the similarity between adjacent nodes is $\delta=0.85,0.9$. The advantage of SNJ increases with the tree size. For perfect binary small trees, the performance of SNJ and NJ is similar. 

	\begin{figure}[htb]
	\begin{subfigure}[b]{0.49\textwidth}
		\includegraphics[width=0.75\textwidth]{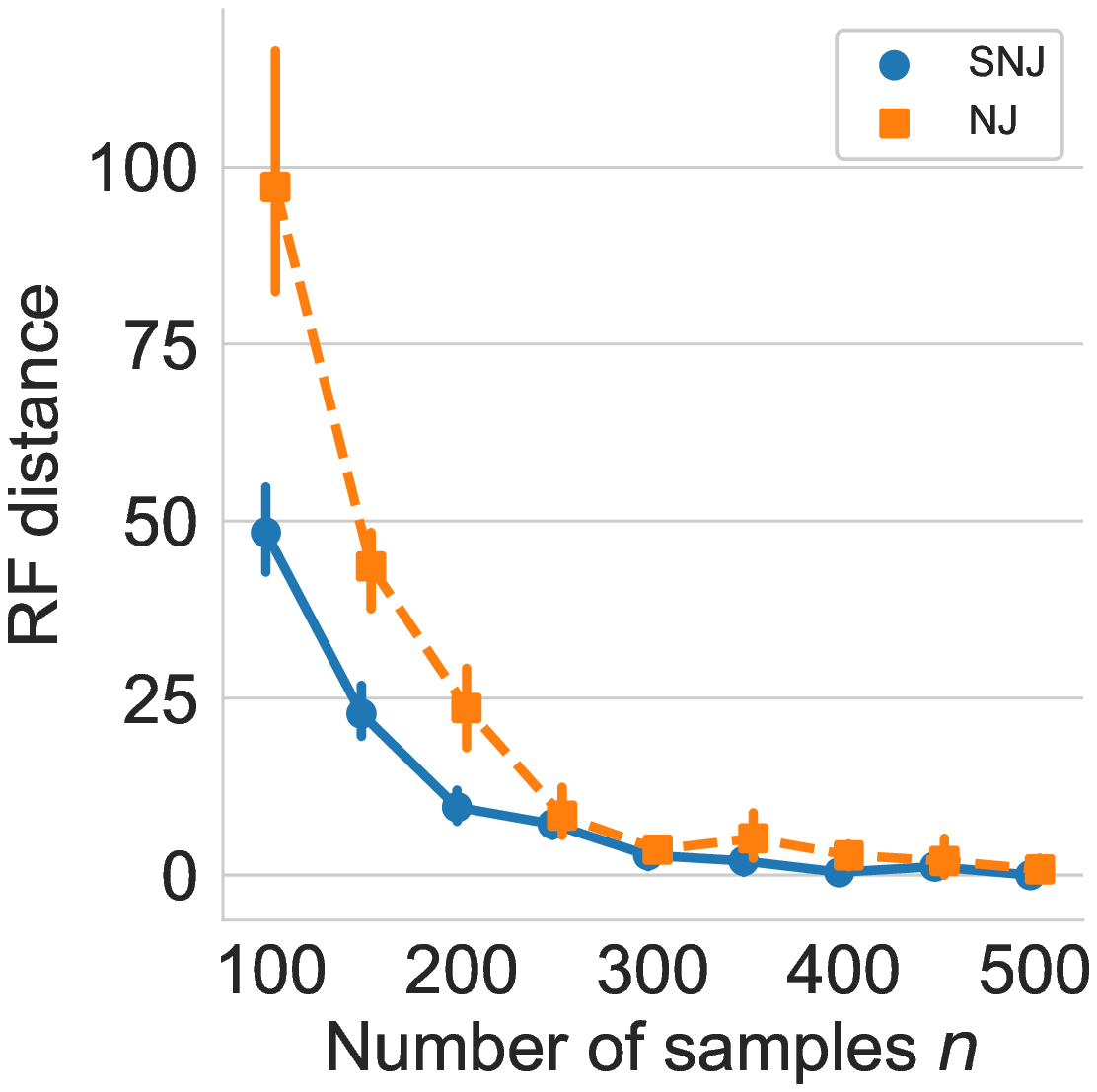}%
	\end{subfigure}
		~
	\begin{subfigure}[b]{0.49\textwidth}
		\includegraphics[width=0.75\textwidth]{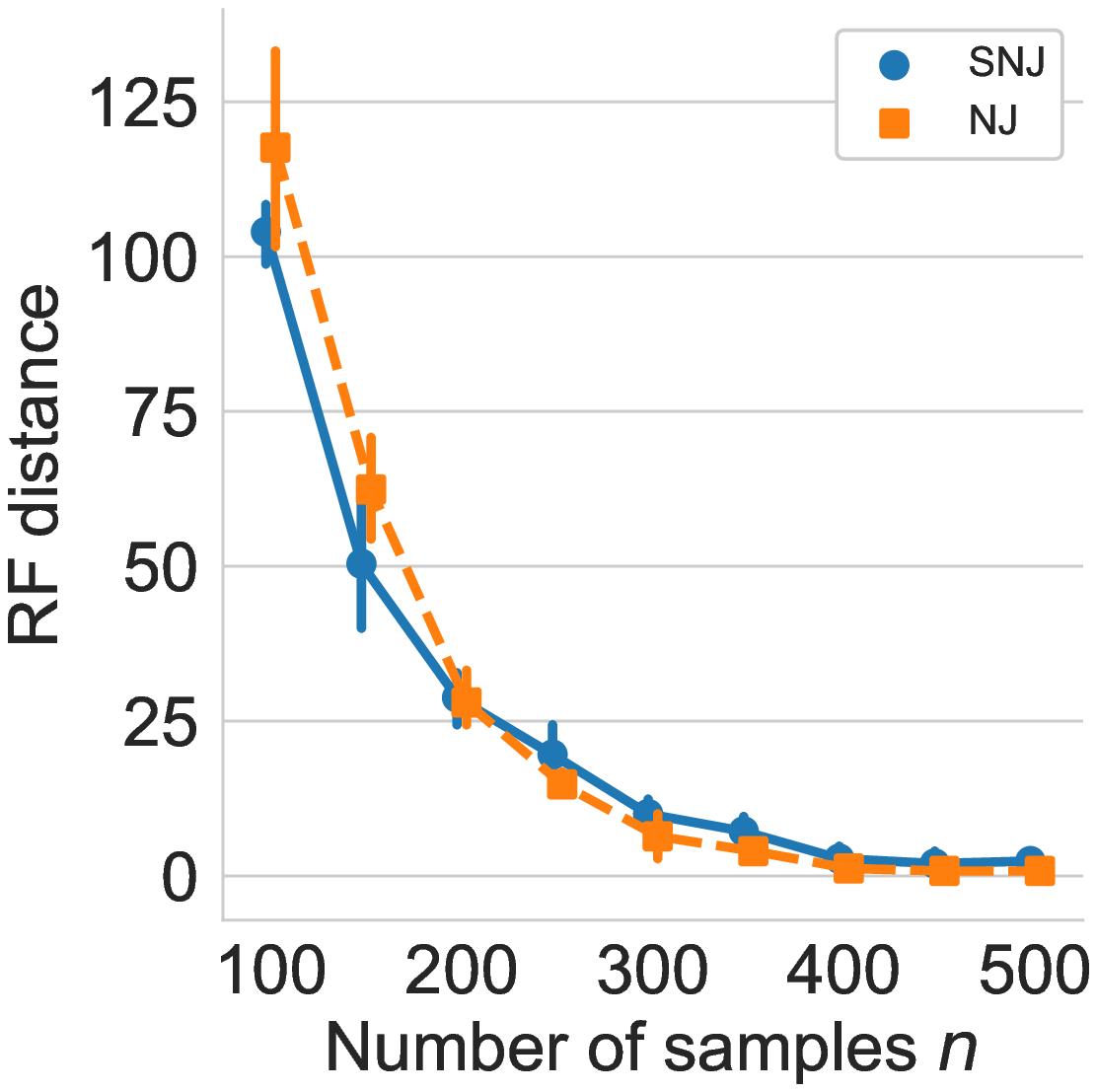}
	\end{subfigure}
	\vspace{-2.5\baselineskip}
		\caption{Comparison between NJ and SNJ for perfect binary trees with $m=512$ nodes,  $\delta=0.85$ (left) and $\delta=0.9$ (right). 
		}. 
		\label{fig:sim_comparison_bin_11}
	\end{figure}
	
\begin{figure}[htb]
    \begin{subfigure}[b]{0.49\textwidth}
		\includegraphics[width=0.75\textwidth]{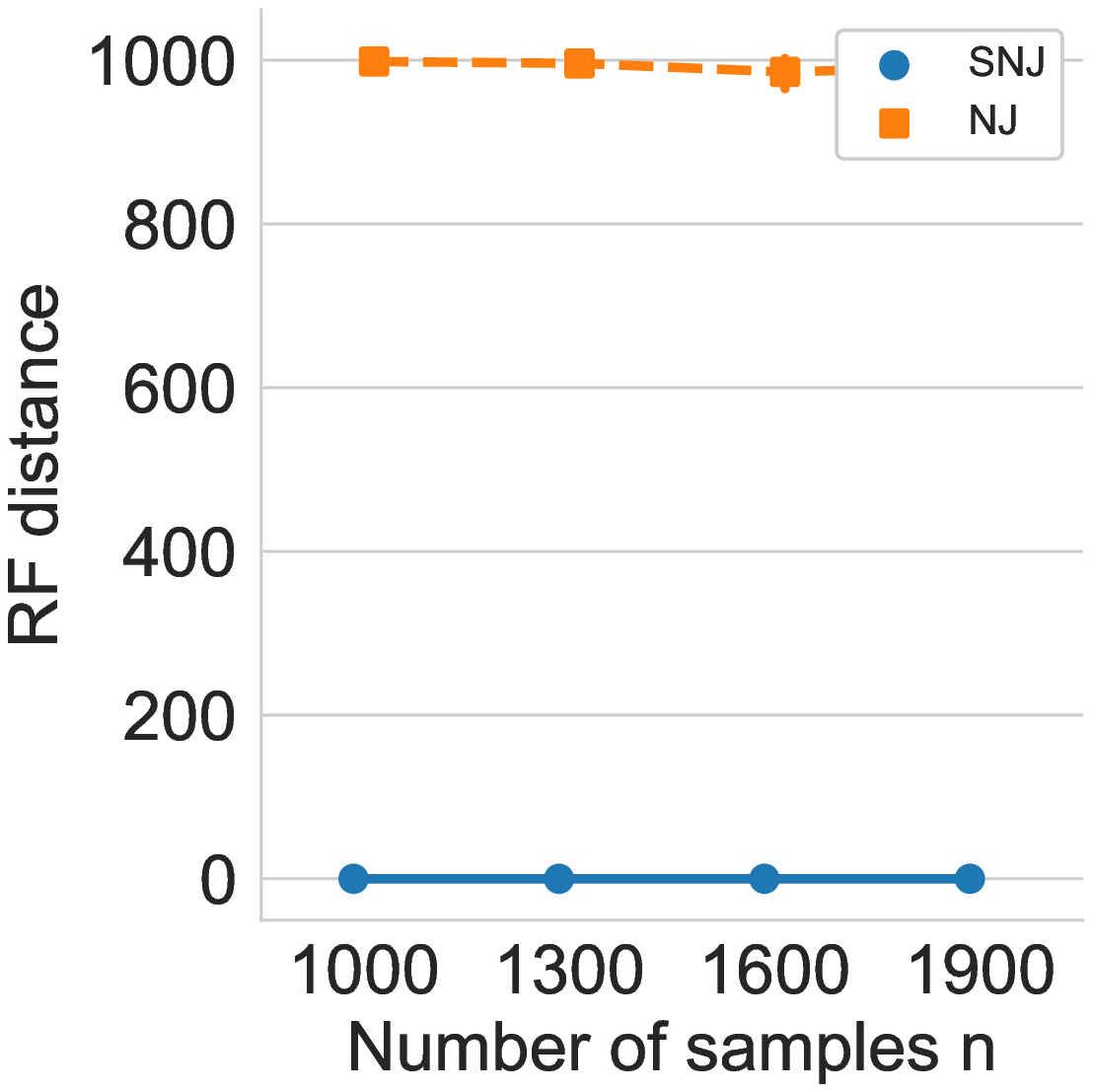}%
	\end{subfigure}
		~
	\begin{subfigure}[b]{0.49\textwidth}
		\includegraphics[width=0.75\textwidth]{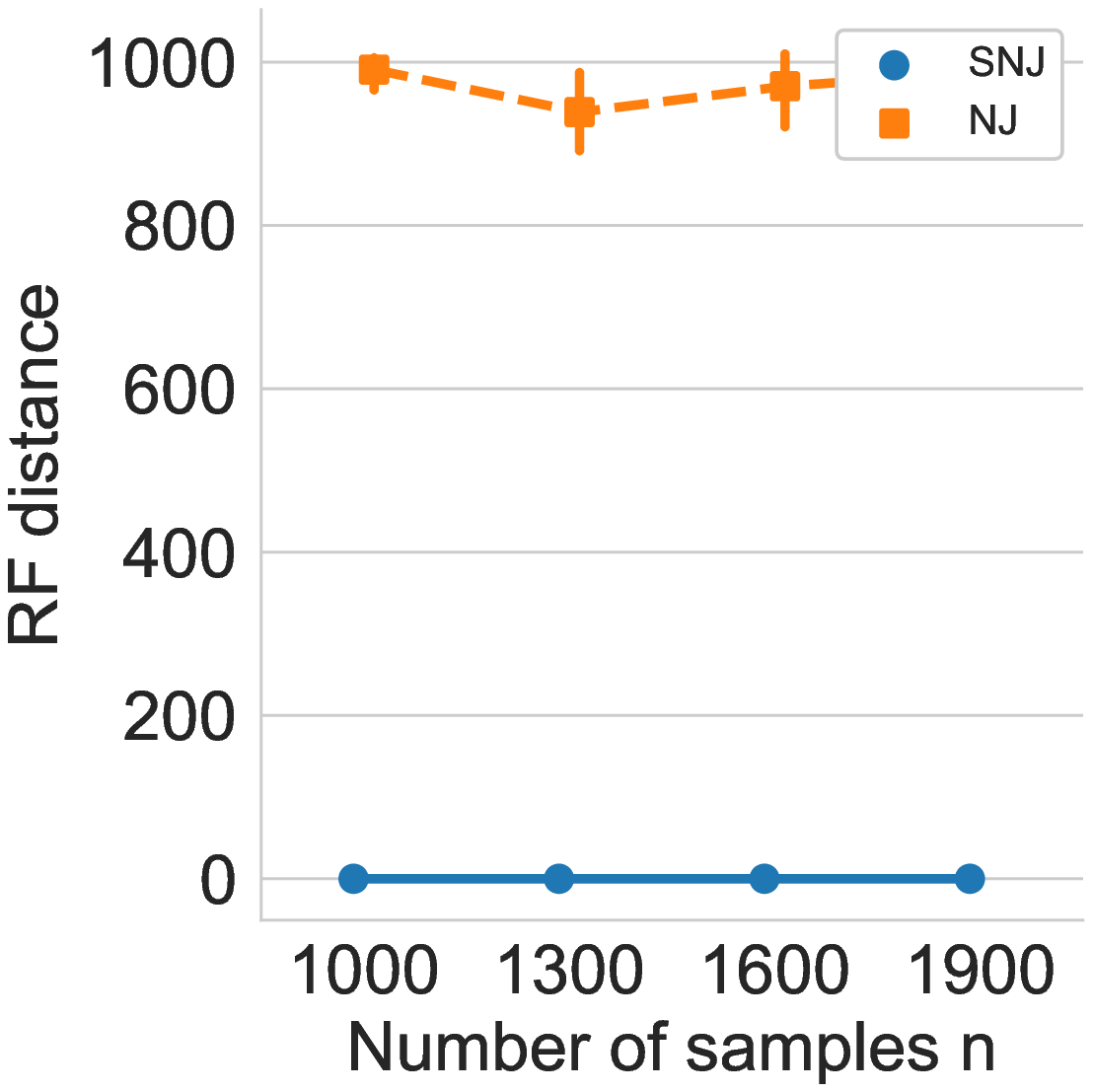}
	\end{subfigure}
	\vspace{-1.5\baselineskip}
		\caption{Comparison between NJ and SNJ for caterpillar trees with $m=512$ nodes  $\delta=0.85$ (left) and $\delta=0.9$ (right).} 
		\label{fig:sim_comparison_cat_11}
	\end{figure}
	\begin{figure}[htb]
	\begin{subfigure}[b]{0.49\textwidth}
		\includegraphics[width=0.75\textwidth]{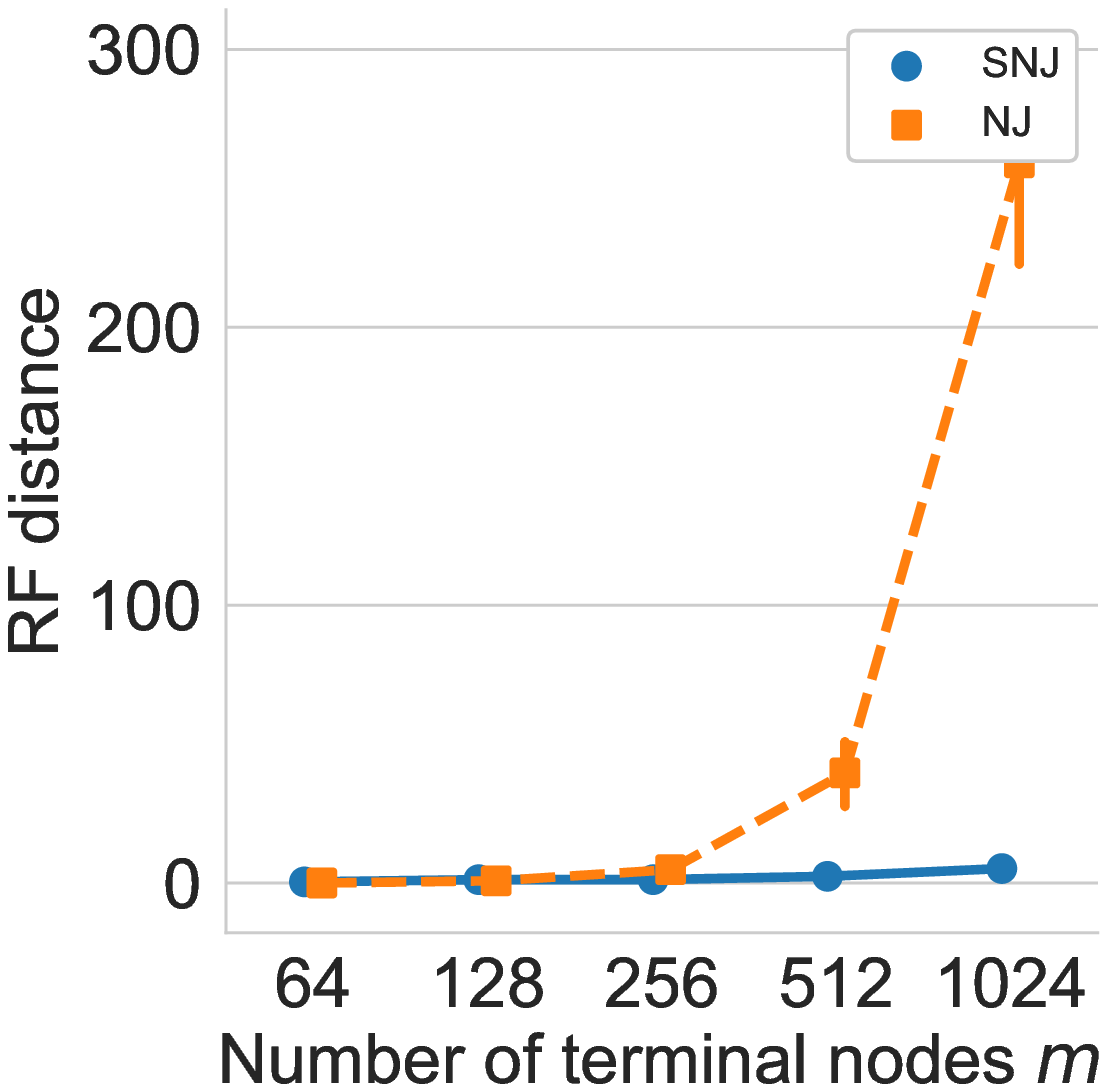}%
		\end{subfigure}
		~
	\begin{subfigure}[b]{0.49\textwidth}
		\includegraphics[width=0.75\textwidth]{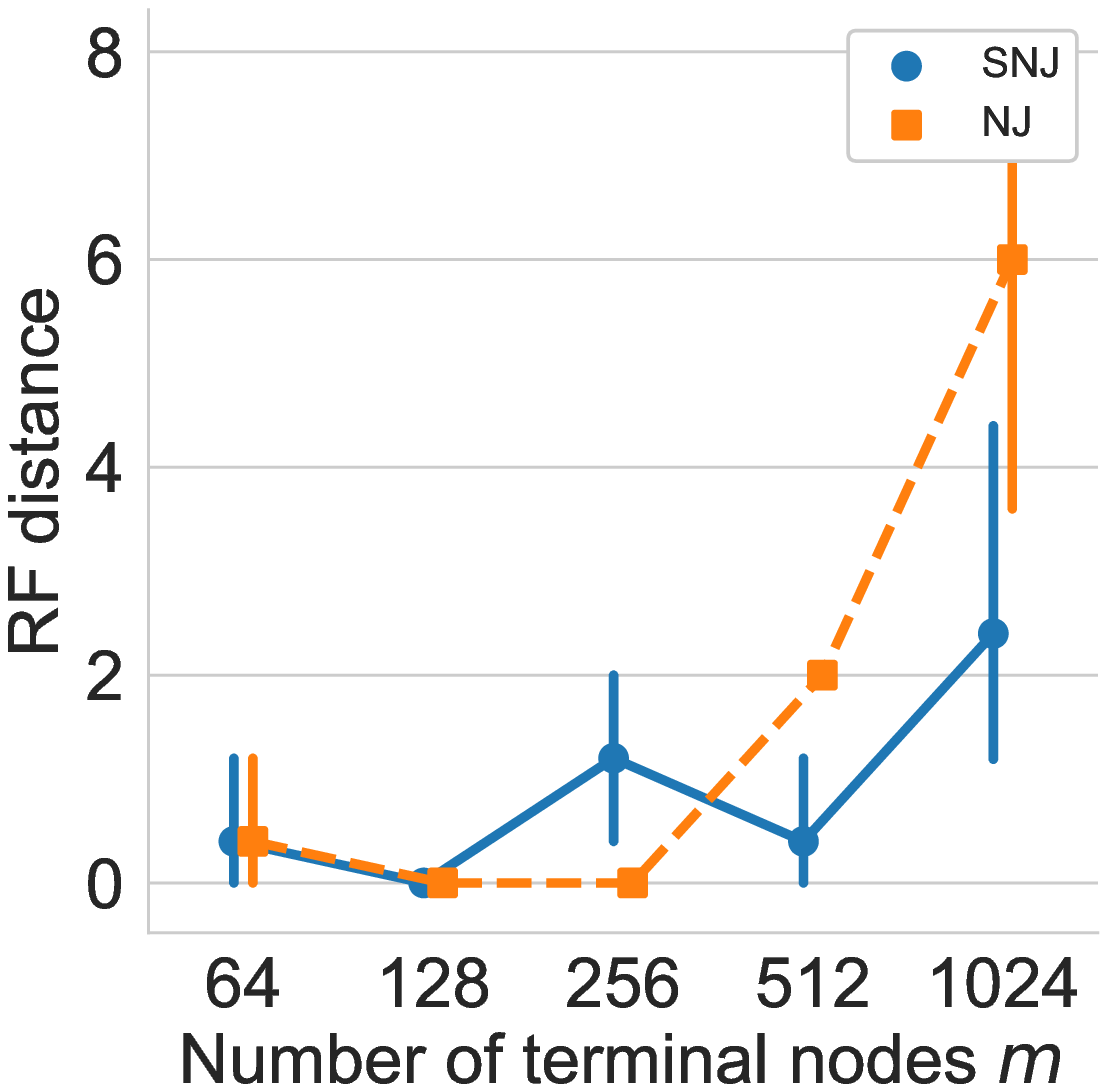}
	\end{subfigure}
	\vspace{-2.5\baselineskip}
		\caption{Comparison between NJ and SNJ for binary trees of different size,  $\delta=0.85$ (left) and $\delta=0.9$ (right). The number of samples is fixed to $n =400$.
		}
		\label{fig:sim_comparison_bin_12}
	\end{figure}
	\begin{figure}[htb]
	\begin{subfigure}[b]{0.49\textwidth}
		\includegraphics[width=0.75\textwidth]{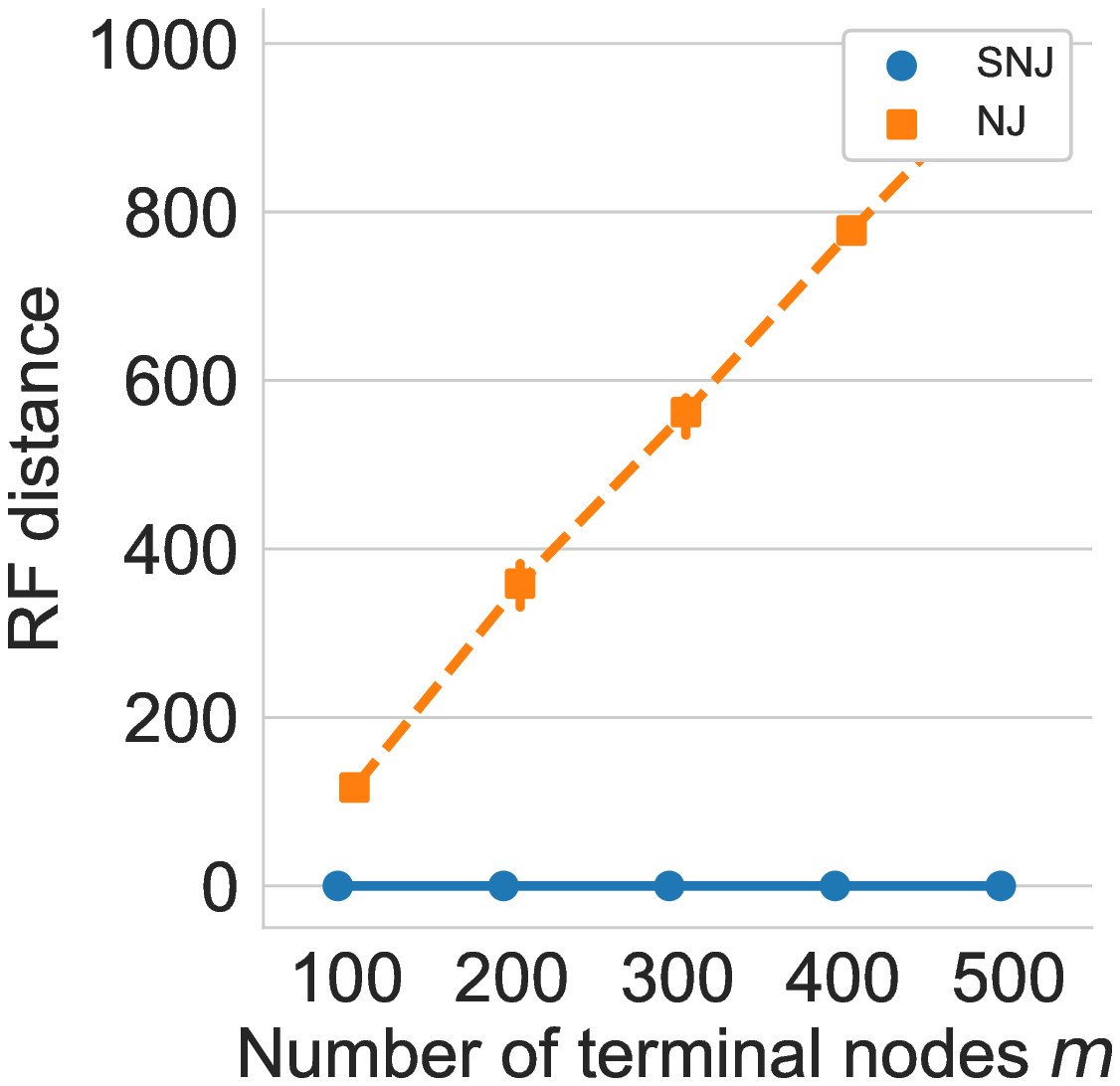}%
		\end{subfigure}
		~
	\begin{subfigure}[b]{0.49\textwidth}
		\includegraphics[width=0.75\textwidth]{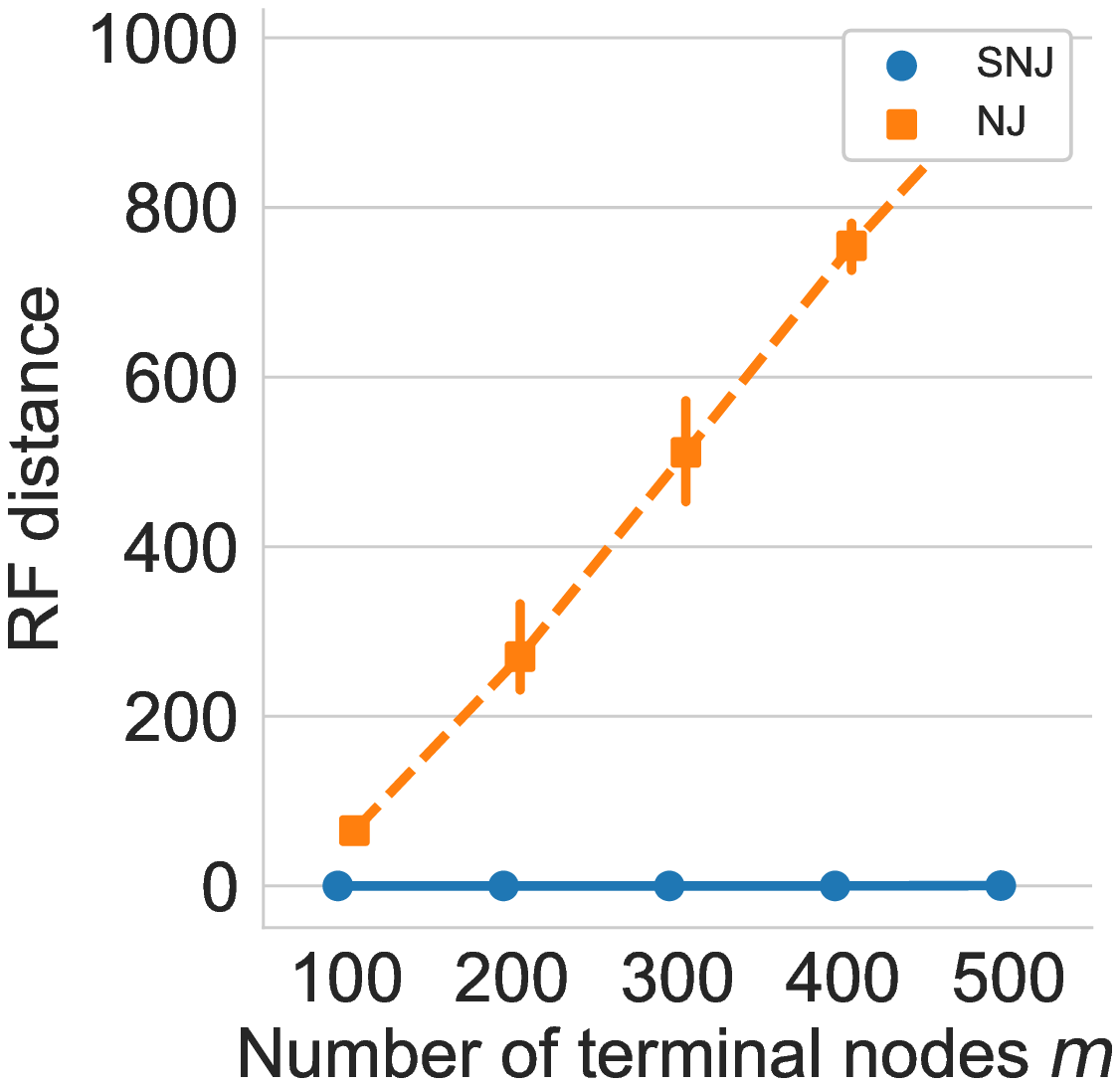}
	\end{subfigure}
	\vspace{-2.5\baselineskip}
		\caption{Comparison between NJ and SNJ for caterpillar trees of different sizes  $\delta=0.85$ (left) and $\delta=0.9$ (right). The number of samples is fixed to $n =800$.
		}
		\label{fig:sim_comparison_cat_12}
	\end{figure}
\textbf{Comparison to NJ and RG for medium size trees with $d=4$ states.}

Figure \ref{fig:sim_bin_21} shows, for the case of a perfect binary tree with $m=128$ terminal nodes, the RF distance  between the tree and its NJ, SNJ and RG estimates, as a function of the sequence length $n$ and for $\delta = 0.85,0.9$. The results are averaged over $5$ realizations. 
The SNJ and NJ algorithms both outperform RG for this tree.  

Next, in Figure \ref{fig:sim_cat_21} we show the results for caterpillar trees with $m=128$ terminal nodes and $\delta=0.85,0.9$. Here, RG outperforms NJ for high mutation rate.
The SNJ method, however, outperforms RG even in this case. 
As discussed in Section \ref{sec:quartet_link}, the required number of samples of quartet based algorithms increase exponentially with the depth of the tree as defined in \eqref{eq:depth}. For perfect binary trees, the depth is of order $O(\log m)$, and for caterpillar trees is equal to one.  Thus, we expect quartet methods such as RG to require more samples for accurately recover a perfect binary tree, compared to caterpillar tree. 

Figure \ref{fig:sim_king_21} shows the results for trees generated according to the coalescent model. The SNJ slightly outperforms NJ with low mutation rate. For higher mutation rate - the results for both methods are similar. 
Figures \ref{fig:sim_bin_22}, \ref{fig:sim_cat_22} and \ref{fig:sim_king_22}
show the performance, as a function of number of terminal nodes $m$ for perfect binary, caterpillar, and coalescent trees respectively.  
The number of samples $n$ is fixed to $400,800 $ and $1000$ for the binary, caterpillar and coalescent trees, respectively. 
For this range of tree sizes, the performance of SNJ and NJ is similar. 
Finally, Figure \ref{fig:sim_runtime_22}
compares the runtime of NJ, SNJ and RG as a function of the number of terminal nodes on a logarithmic scale. As expected, the runtime of RG is much higher than the runtimes of NJ and SNJ. 

\begin{figure}[htb]
		\begin{subfigure}[b]{0.49\textwidth}
		\includegraphics[width=0.75\textwidth]{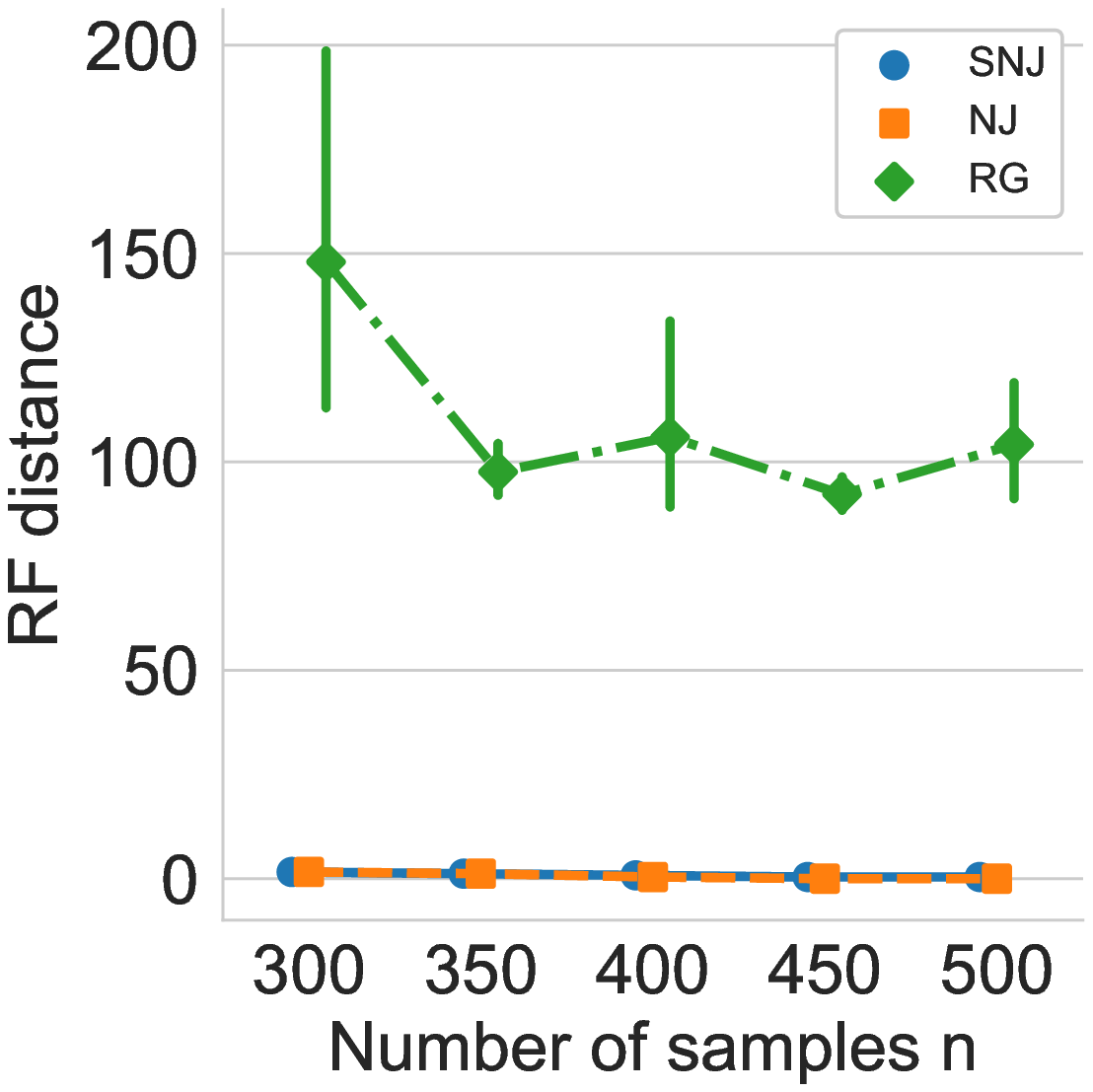}%
	\end{subfigure}
		~
		\begin{subfigure}[b]{0.49\textwidth}\includegraphics[width=0.75\textwidth]{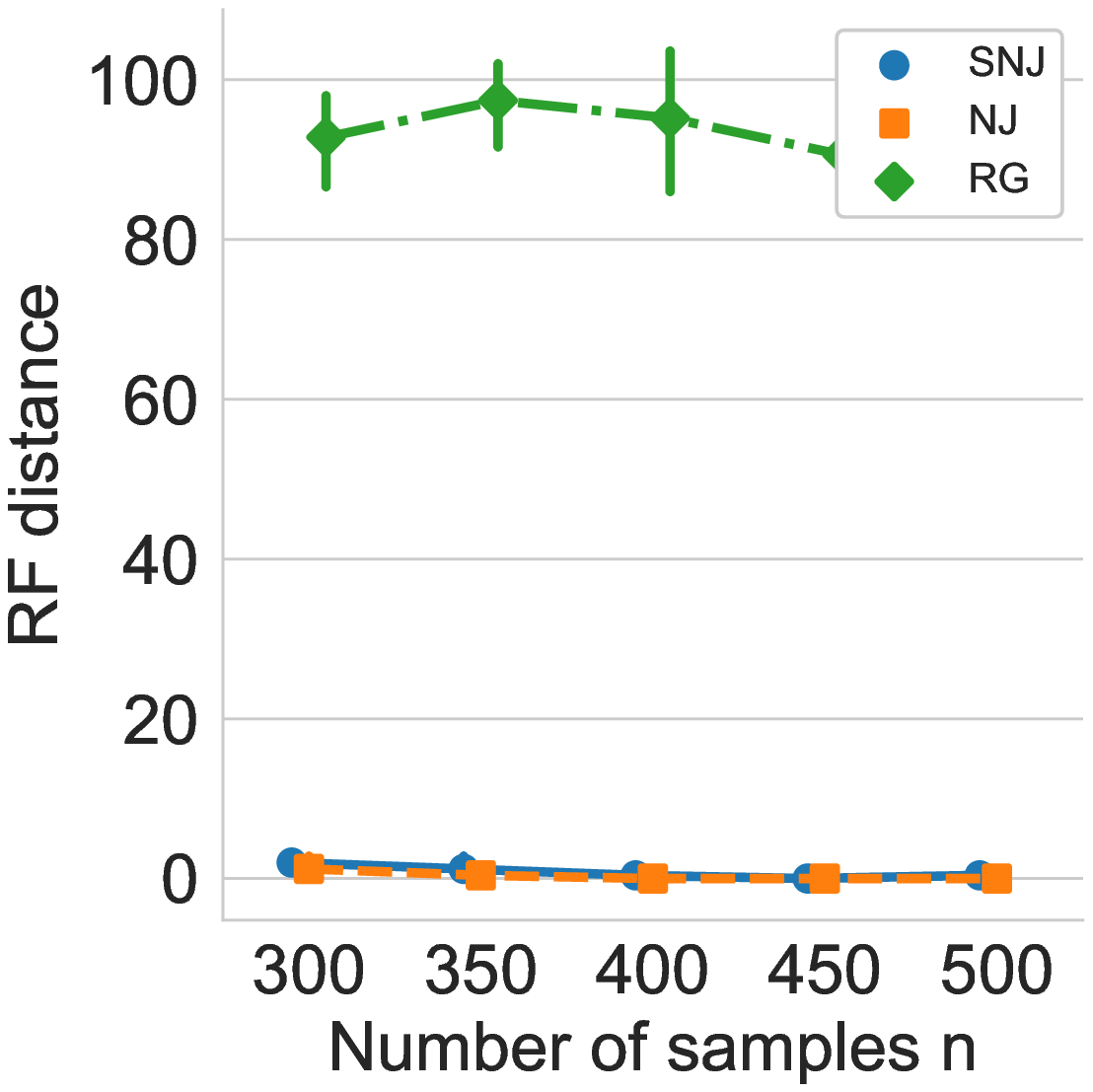}
	\end{subfigure}
	\vspace{-2.5\baselineskip}
		\caption{Comparison between NJ, SNJ and RG for perfect binary trees with $m=128$ nodes  $\delta=0.85$ (left) and $\delta=0.9$ (right).} 
		\label{fig:sim_bin_21}
	\end{figure}
\begin{figure}[htb]
		\begin{subfigure}[b]{0.49\textwidth}
		\includegraphics[width=0.75\textwidth]{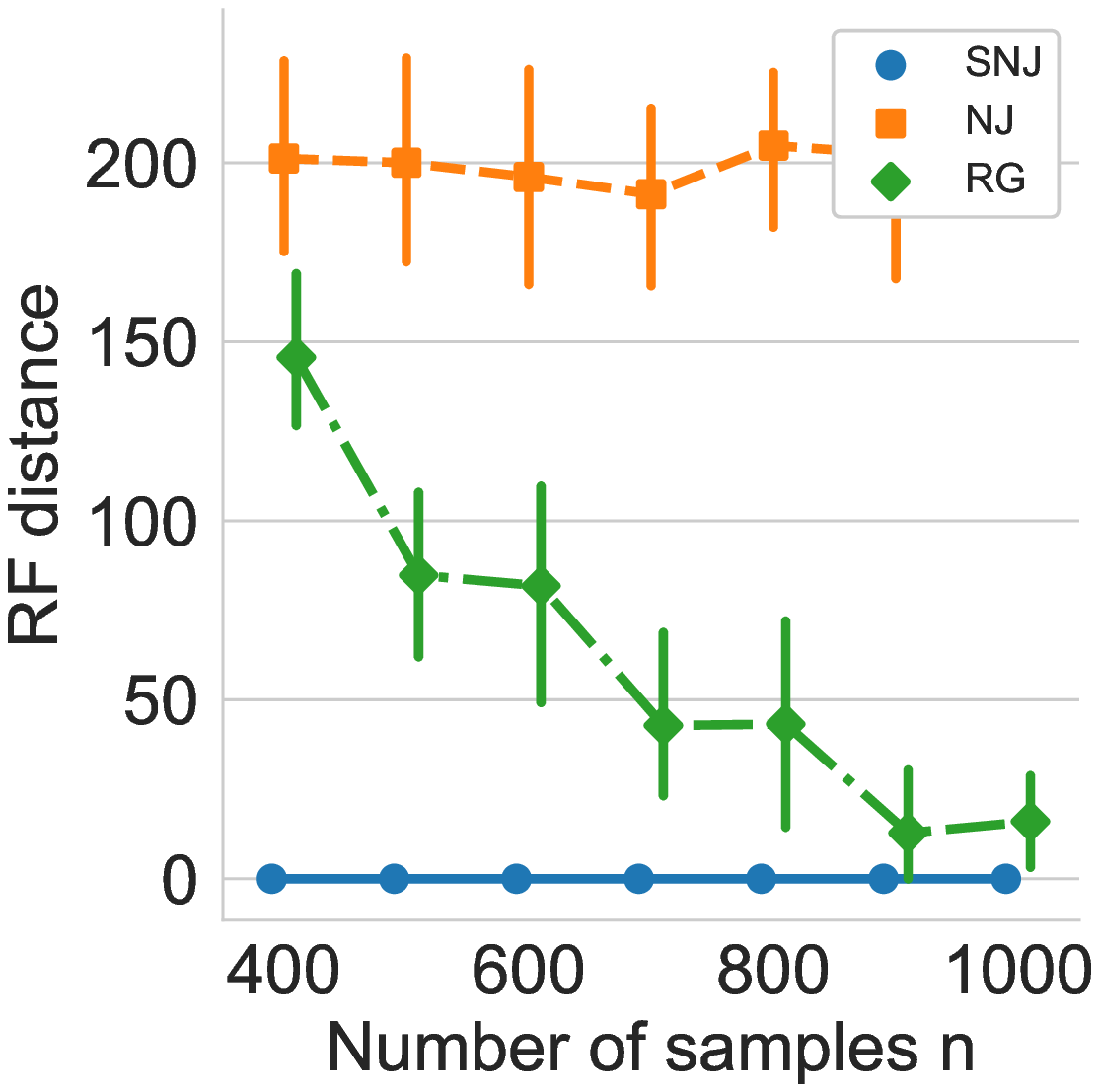}%
	\end{subfigure}
		~
		\begin{subfigure}[b]{0.49\textwidth}\includegraphics[width=0.75\textwidth]{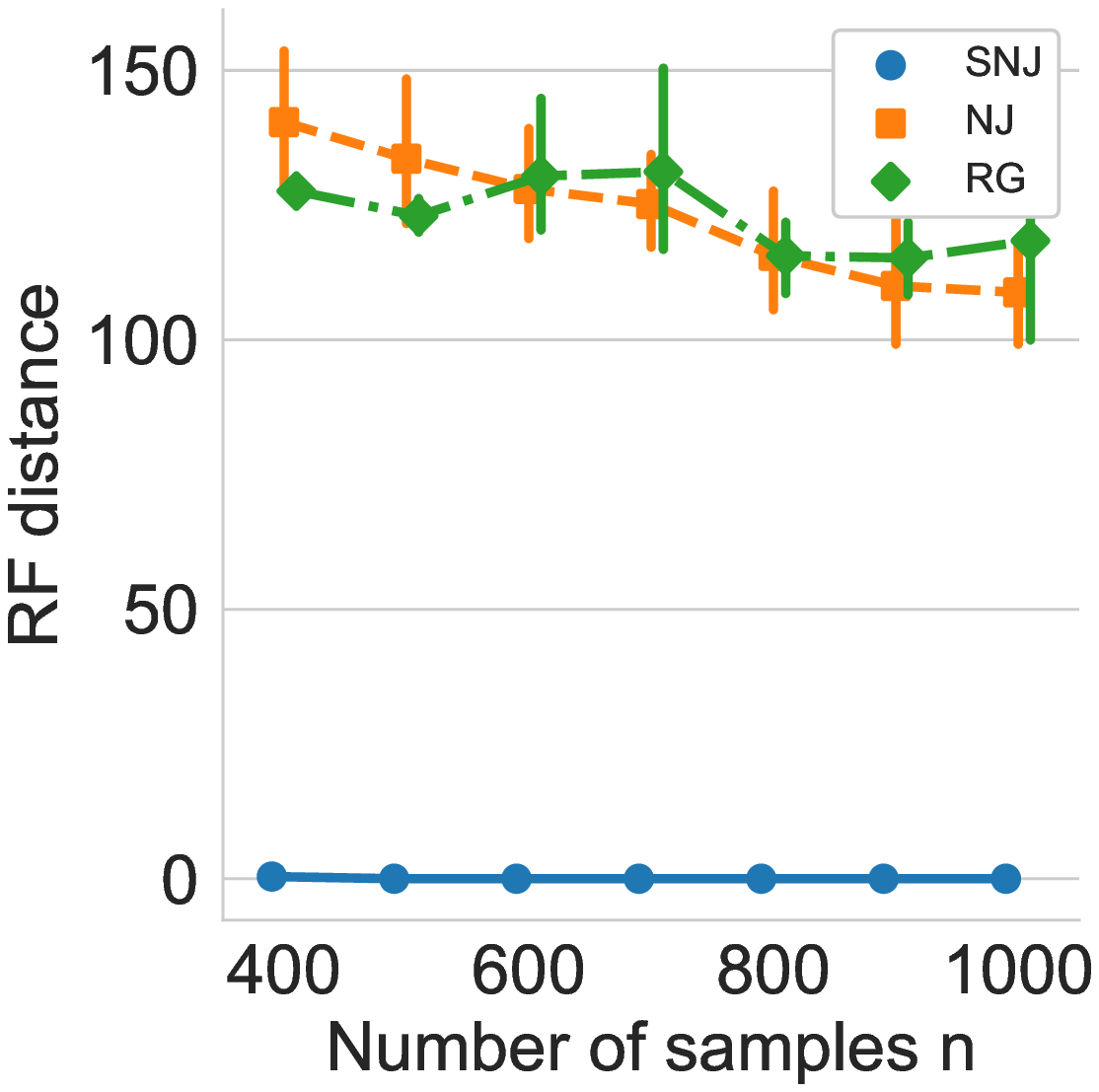}
	\end{subfigure}
	\vspace{-2.5\baselineskip}
		\caption{Comparison between NJ,RG and SNJ for caterpillar trees with $m=128$ nodes  $\delta=0.85$ (left) and $\delta=0.9$ (right).}
		\label{fig:sim_cat_21}
	\end{figure}
	
\begin{figure}[htb]
		\begin{subfigure}[b]{0.49\textwidth}
		\includegraphics[width=0.75\textwidth]{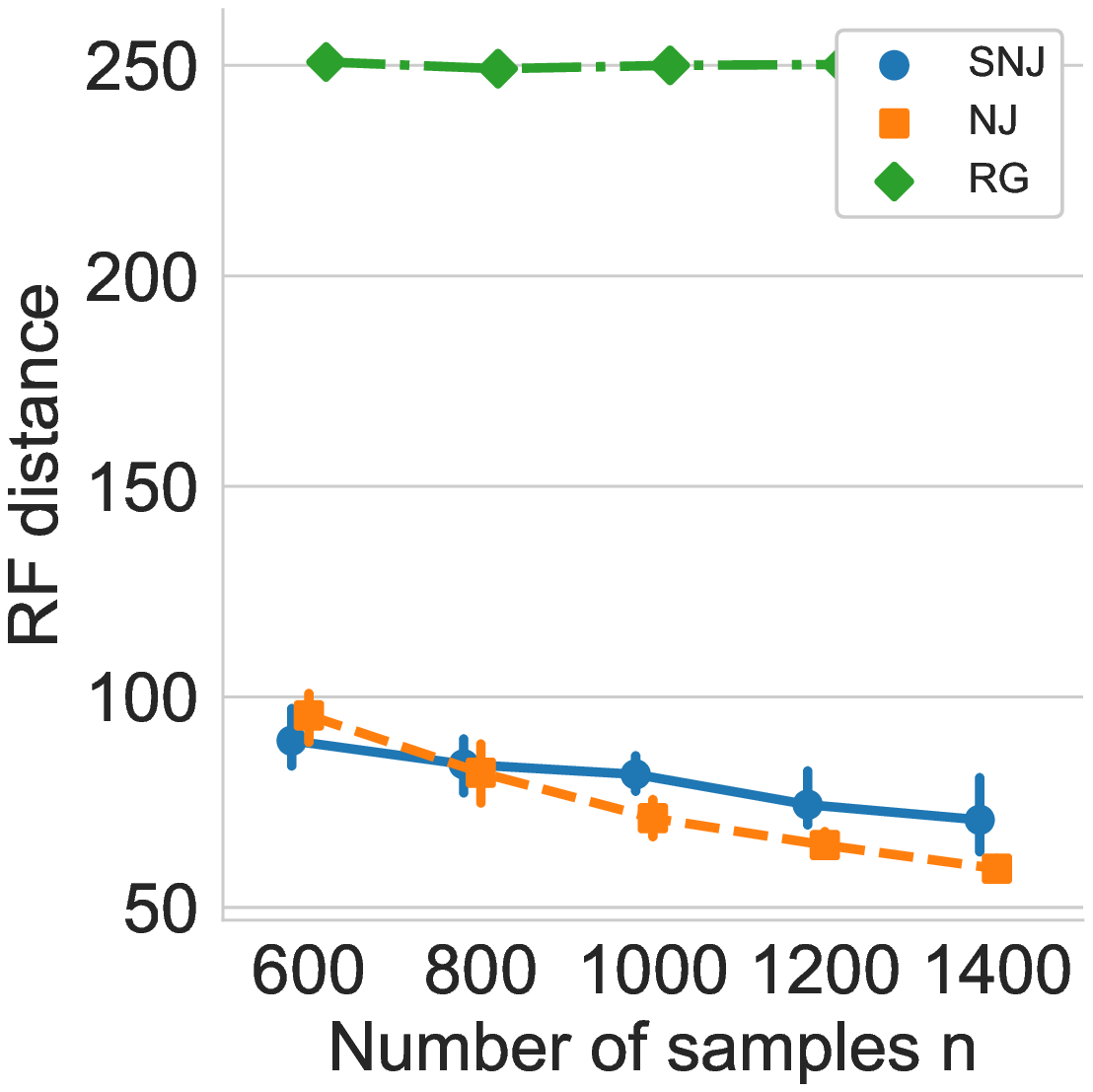}%
	\end{subfigure}
		~
		\begin{subfigure}[b]{0.49\textwidth}\includegraphics[width=0.75\textwidth]{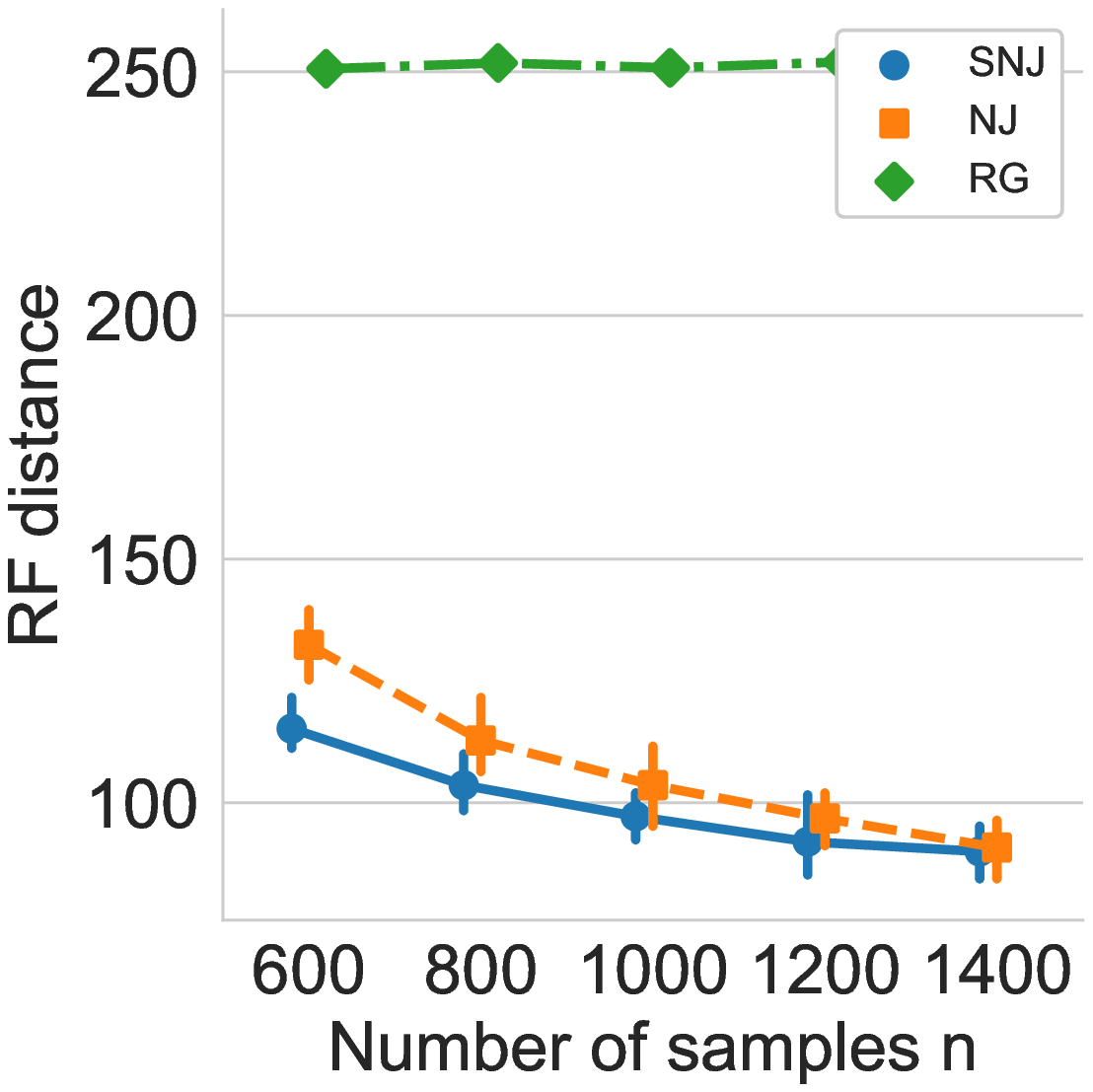}
	\end{subfigure}
	\vspace{-2.5\baselineskip}
		\caption{Comparison between NJ,RG and SNJ for trees with $m=128$ nodes, generated according to the coalescent model  $\delta=0.85$ (left) and $\delta=0.9$ (right). } 
		\label{fig:sim_king_21}
	\end{figure}
	
	\begin{figure}[htb]
		\begin{subfigure}[b]{0.49\textwidth}
		\includegraphics[width=0.75\textwidth]{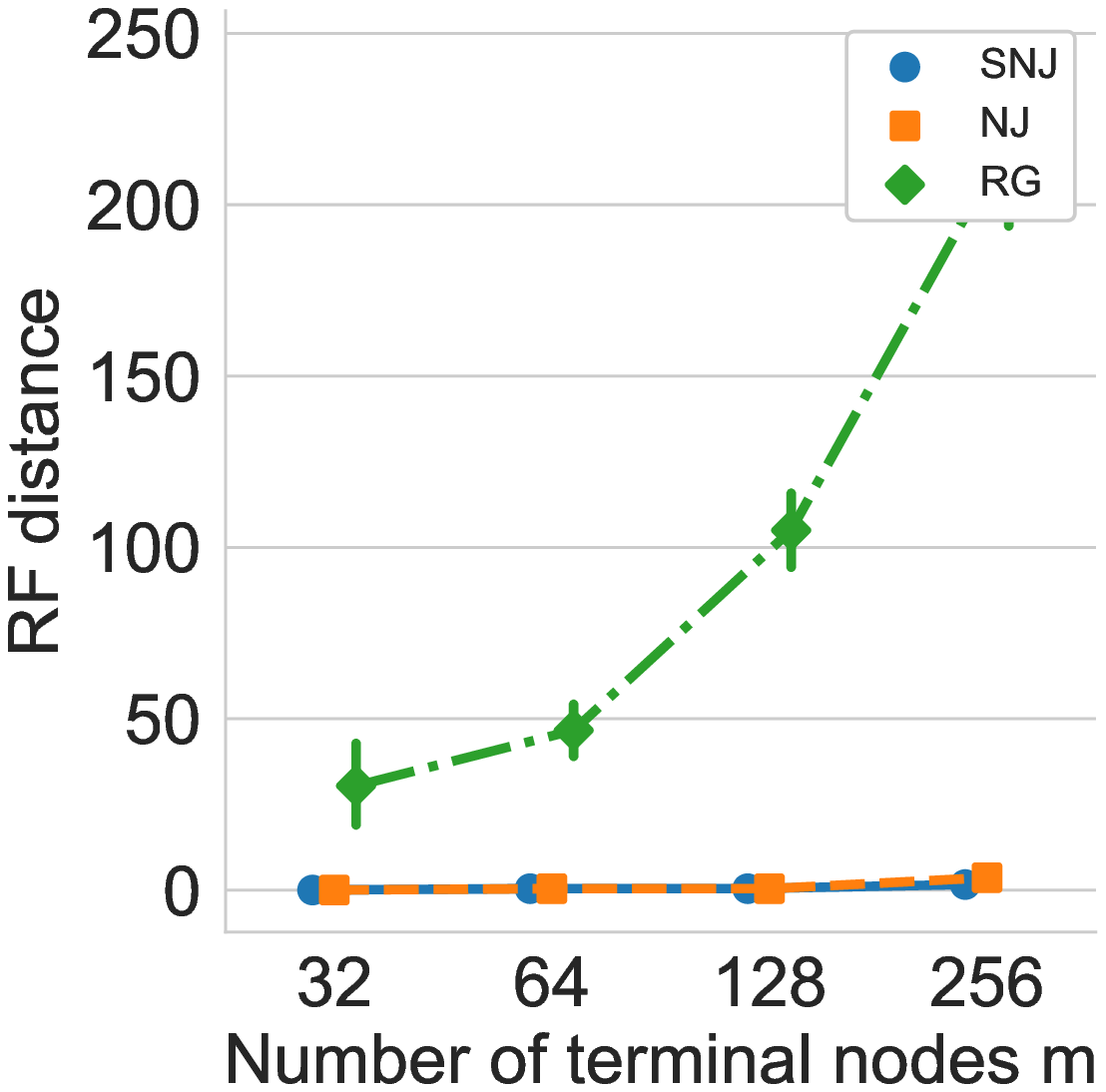}%
	\end{subfigure}
		~
		\begin{subfigure}[b]{0.49\textwidth}\includegraphics[width=0.75\textwidth]{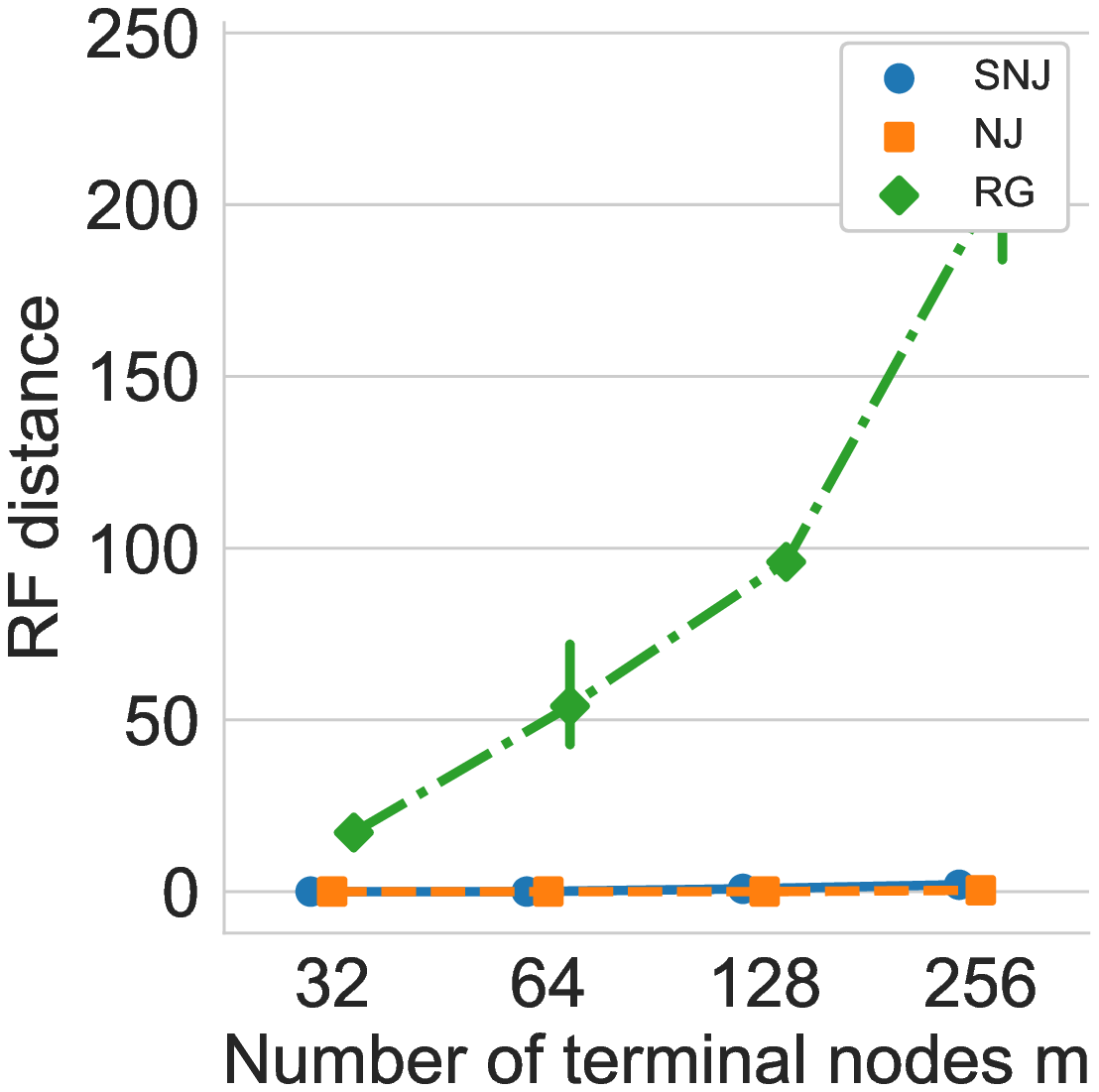}
	\end{subfigure}
	\vspace{-2.5\baselineskip}
		\caption{Comparison between NJ, SNJ and RG for binary trees of different sizes,  $\delta=0.85$ (left) and $\delta=0.9$ (right). The number of samples is fixed with $n =400$ samples.
		} 
		\label{fig:sim_bin_22}
	\end{figure}
	\begin{figure}[htb]
		\begin{subfigure}[b]{0.49\textwidth}\includegraphics[width=0.75\textwidth]{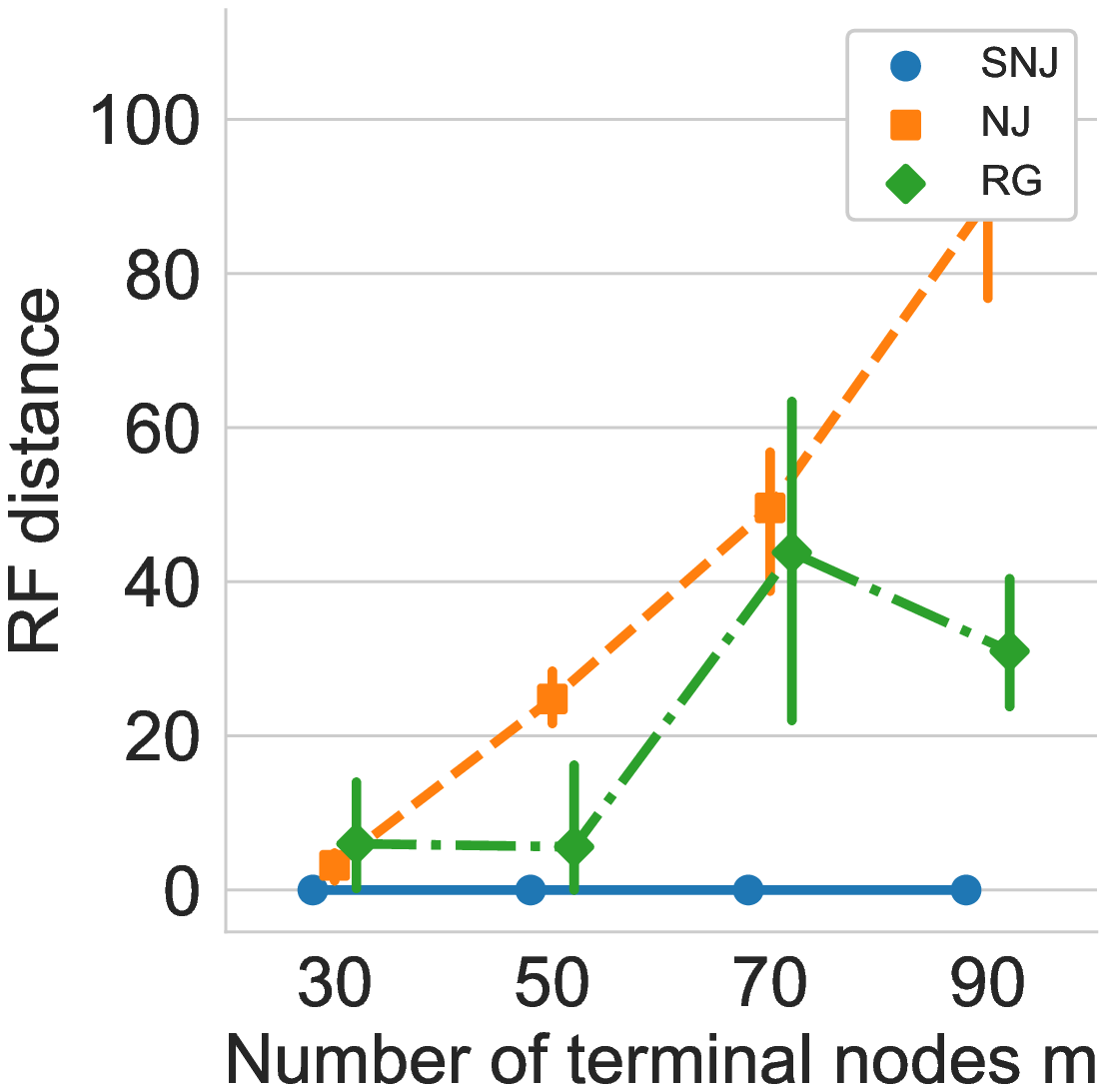}%
	\end{subfigure}
		~
		\begin{subfigure}[b]{0.49\textwidth}\includegraphics[width=0.75\textwidth]{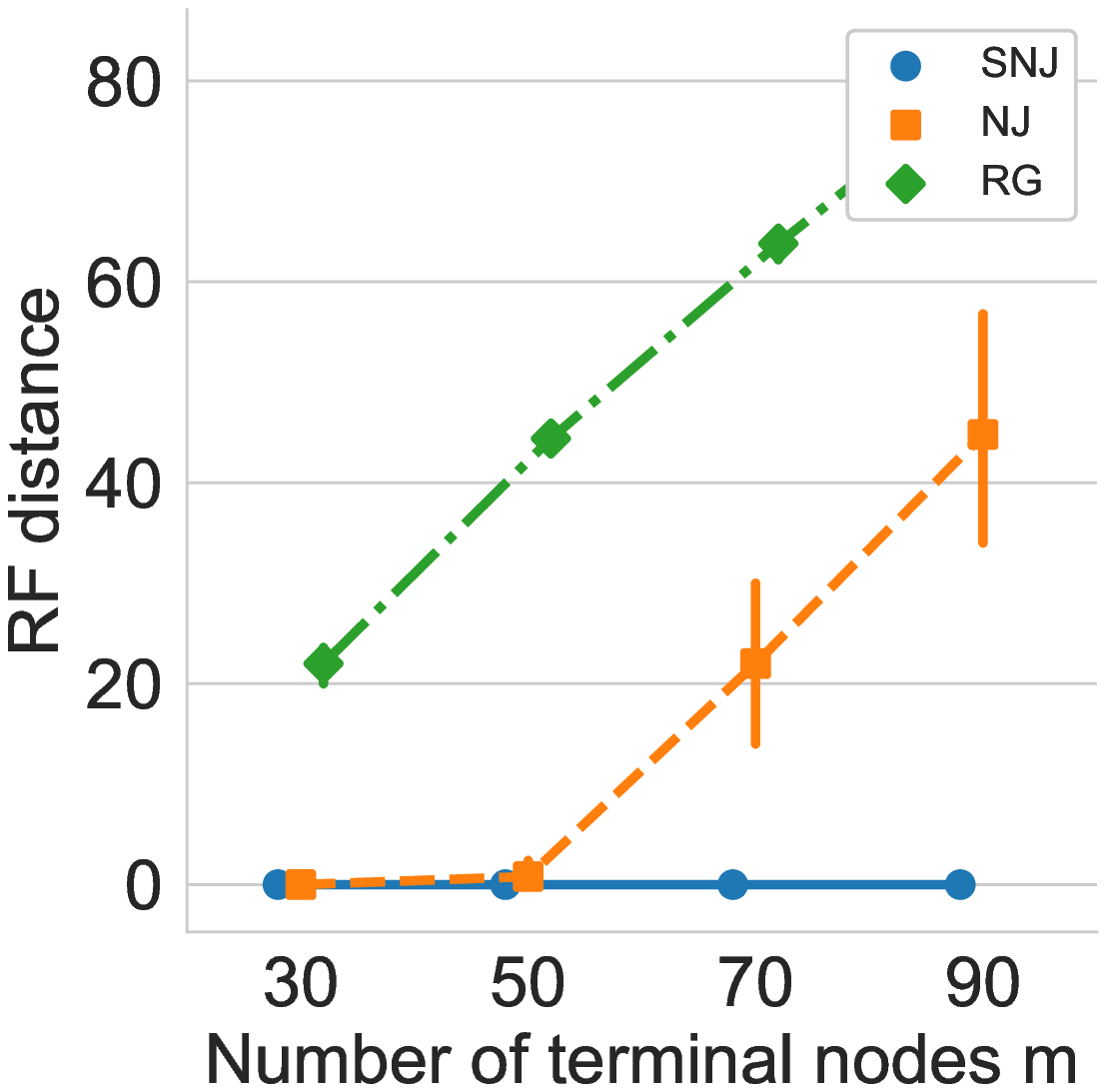}
	\end{subfigure}
	\vspace{-2.5\baselineskip}
		\caption{Comparison between NJ, SNJ and RG for caterpillar trees of different sizes  $\delta=0.85$ (left) and $\delta=0.9$ (right). The number of samples is fixed with $n =800$ samples.
		}
		\label{fig:sim_cat_22}
	\end{figure}
	
\begin{figure}[htb]
		\begin{subfigure}[b]{0.49\textwidth}\includegraphics[width=0.75\textwidth]{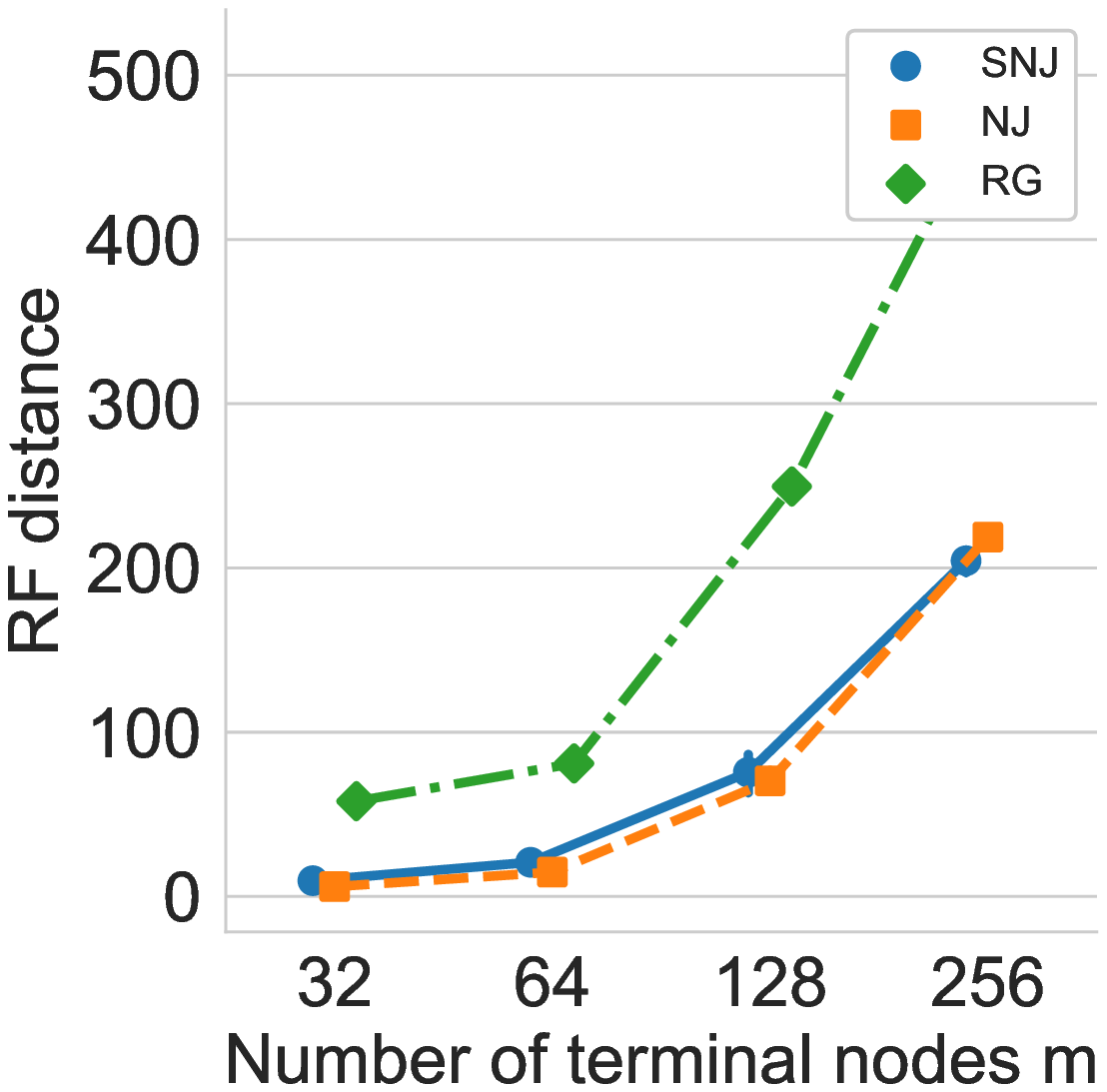}%
	\end{subfigure}
		~
		\begin{subfigure}[b]{0.49\textwidth}\includegraphics[width=0.75\textwidth]{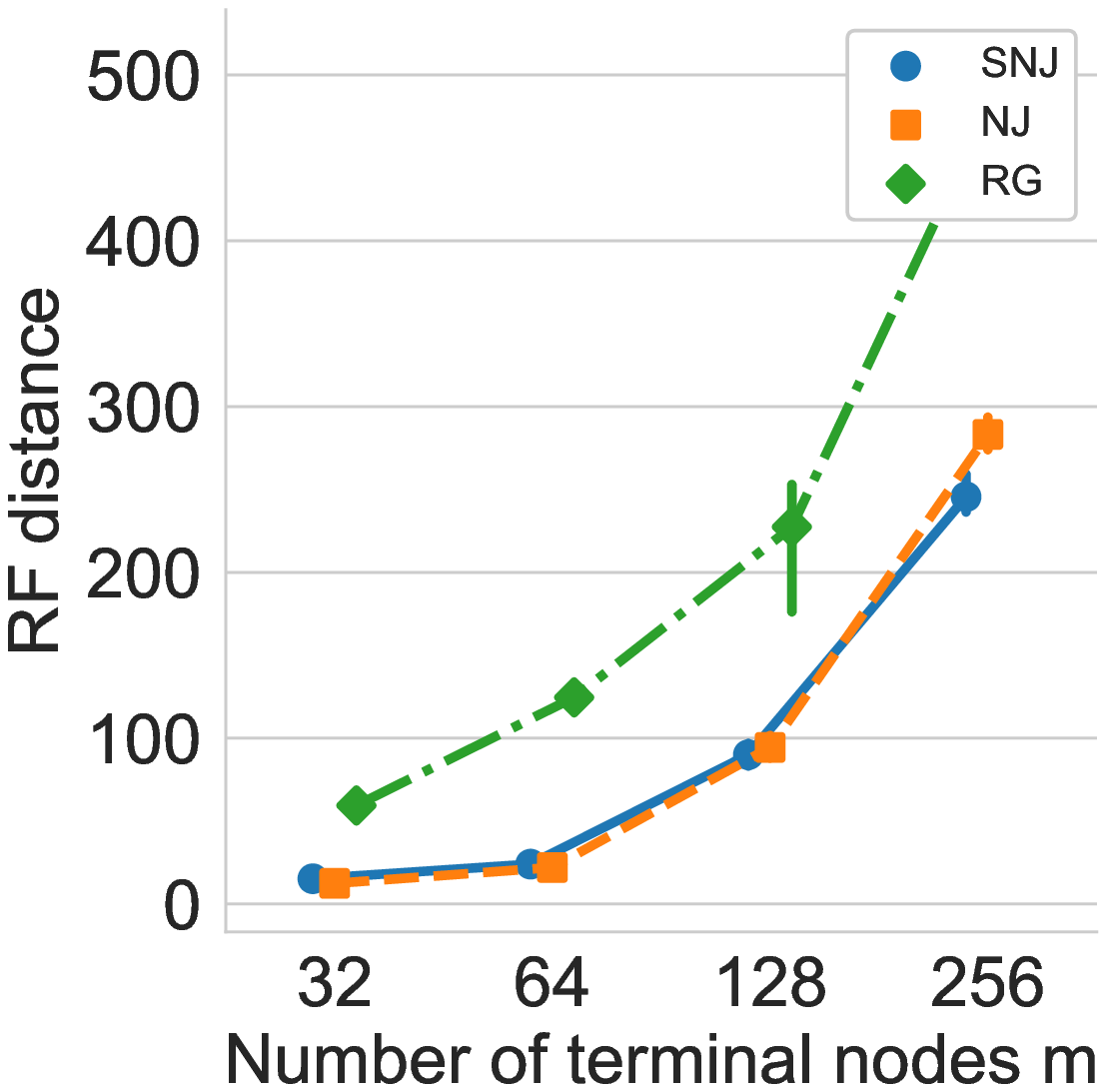}
	\end{subfigure}
	\vspace{-2.5\baselineskip}
		\caption{Comparison between NJ,SNJ and RG for trees with different sizes, generated according to the coalescent model,  $\delta=0.85$ (left) and $\delta=0.9$ (right). The number of samples is fixed to $n=1000$. } 
		\label{fig:sim_king_22}
	\end{figure}

	\begin{figure}[htb]
		\begin{subfigure}[b]{0.49\textwidth}\includegraphics[width=0.75\textwidth]{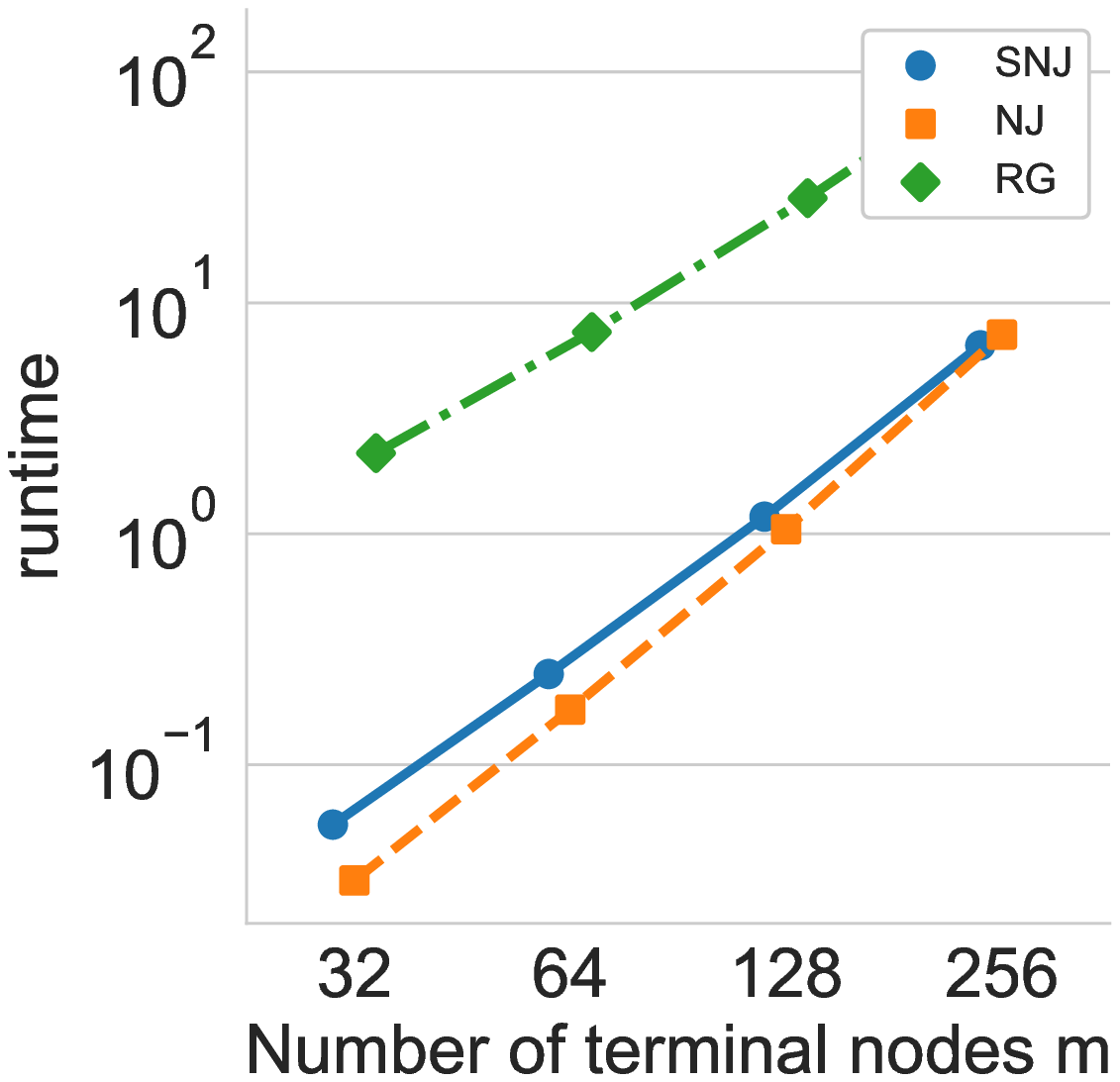}%
	\end{subfigure}
		~
		\begin{subfigure}[b]{0.49\textwidth}\includegraphics[width=0.75\textwidth]{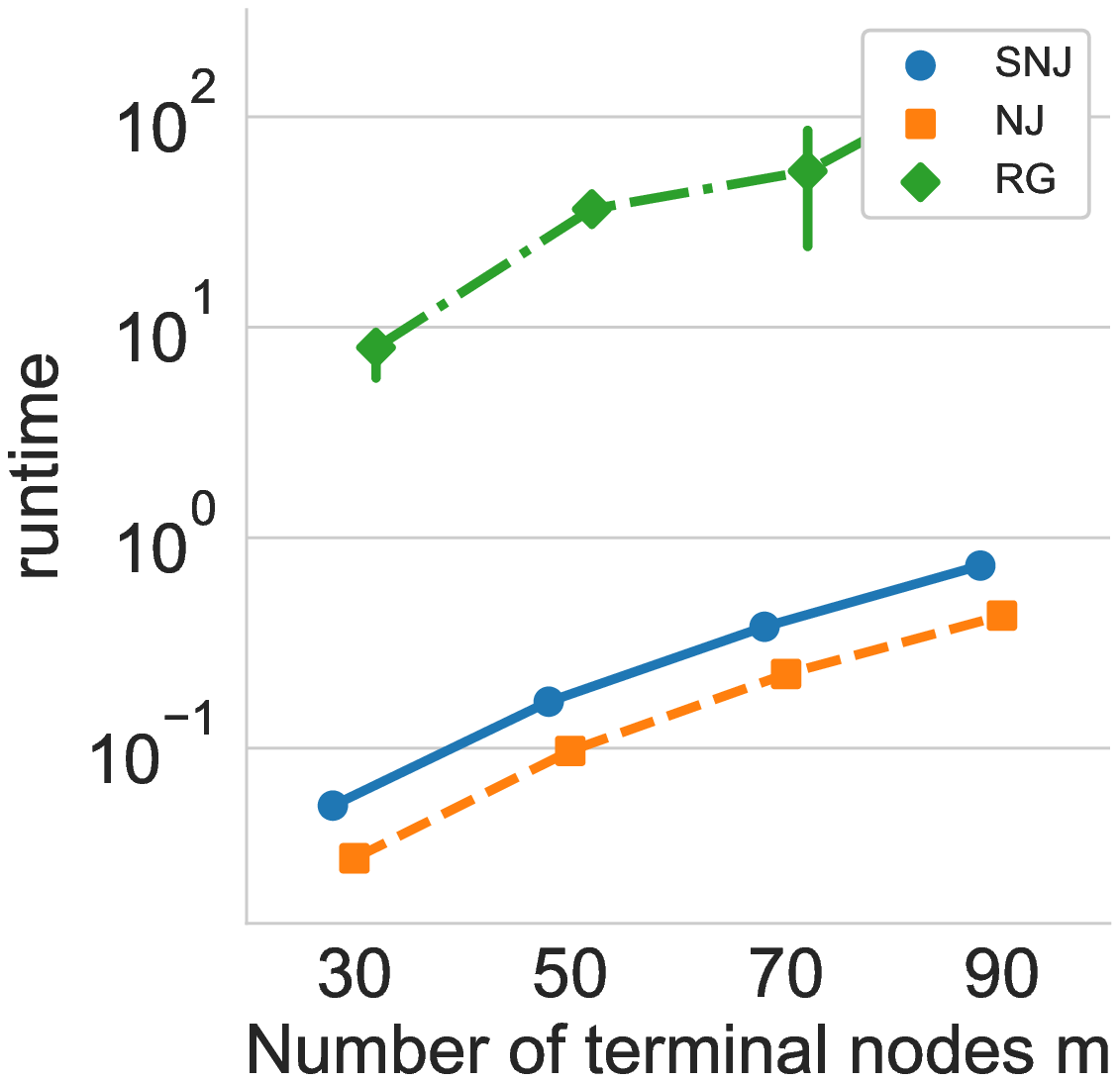}
	\end{subfigure}
	\vspace{-2.5\baselineskip}
		\caption{Comparison between the runtime of NJ,SNJ and RG for the experiment on binary and caterpillar trees shown in Figure \ref{fig:sim_bin_22} and \ref{fig:sim_cat_22}.}. 
		\label{fig:sim_runtime_22}
	\end{figure}


\textbf{Comparison to Tree SVD and Binary Forrest on small trees.}

We compared SNJ, NJ, Tree SVD and Binary Forrest on small trees with binary data.
Figure \ref{fig:sim_comparison_bin_31} and Figure \ref{fig:sim_comparison_cat_31} show the RF distance between the tree and its estimates for perfect binary and caterpillar trees, respectively. The number of terminal nodes is $m=16$ and the similarity between adjacent nodes is $\delta = 0.85,0.9$. The results are averaged over $5$ realizations of the tree model. 

The performance of the Binary Forrest algorithm is comparable to NJ and SNJ for 
high values of $\delta$, and long sequence length, but is inferior to both methods for low $\delta$. Both NJ and SNJ perform better than the Tree SVD algorithm. 

In terms of runtime, even for such small trees, there is a difference of more than two orders of magnitude between the runtime of NJ and SNJ, and those of Binary Forrest and Tree SVD

\begin{figure}[htb]

		\begin{subfigure}[b]{0.49\textwidth}
		\includegraphics[width=0.75\textwidth]{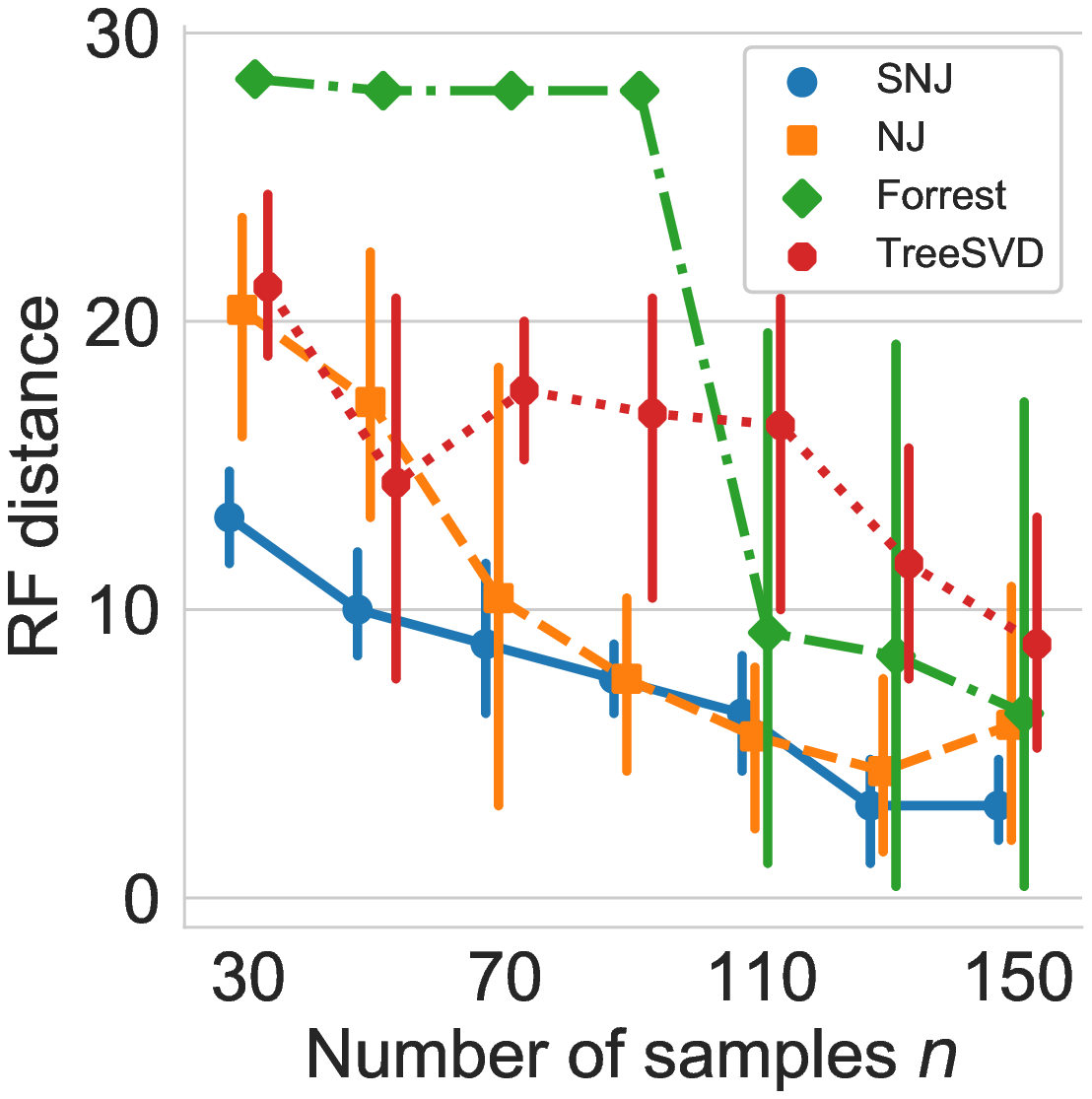}%
	    \end{subfigure}
		~
		\begin{subfigure}[b]{0.49\textwidth}\includegraphics[width=0.75\textwidth]{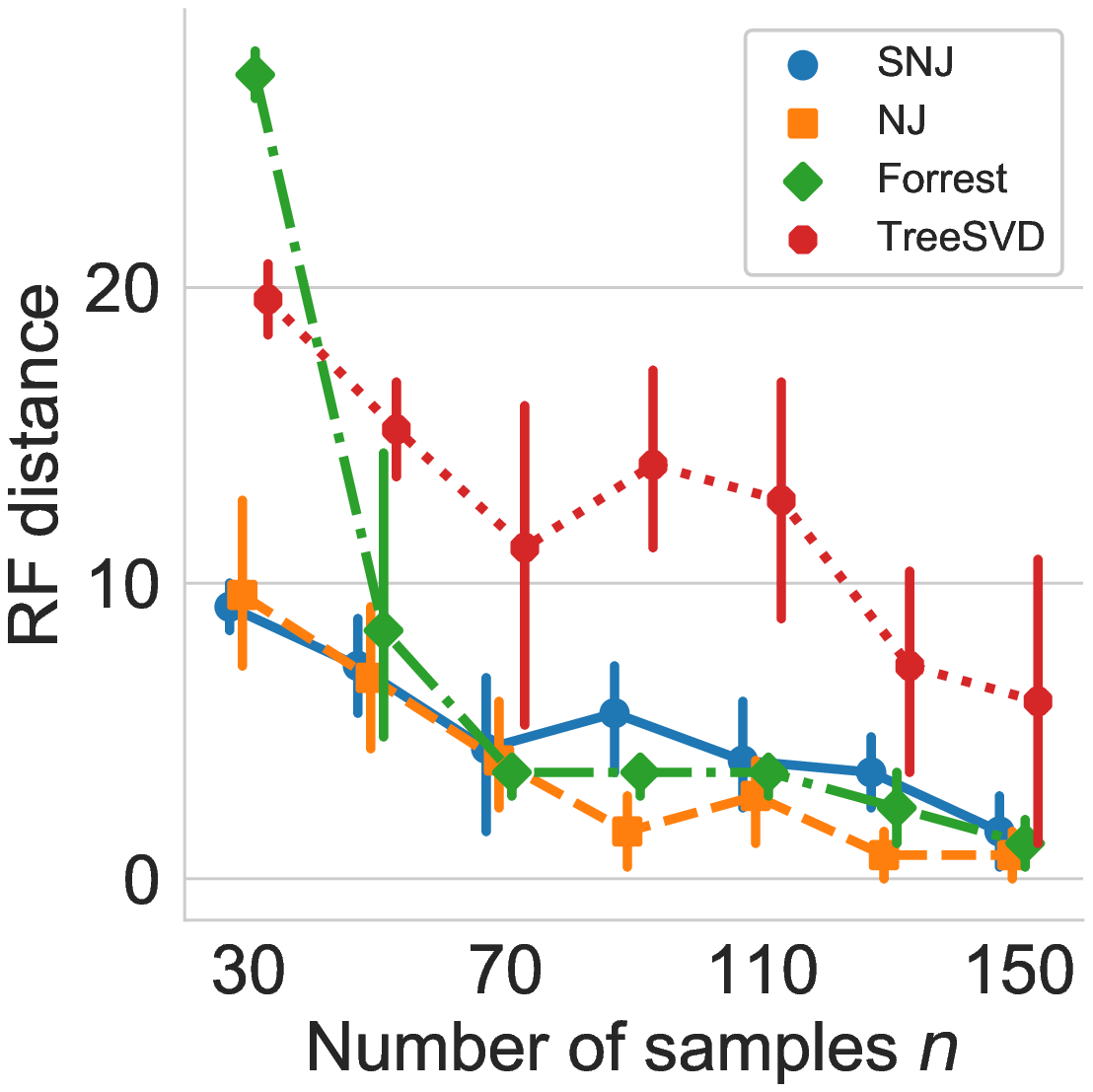}
	\end{subfigure}
	\vspace{-2.5\baselineskip}
		\caption{Comparison between NJ, SNJ, Binary forrest and Tree SVD for a balanced binary tree with $m=16$ terminal nodes  $\delta=0.85$ (left) and $\delta=0.9$ (right).} 
		\label{fig:sim_comparison_bin_31}
\end{figure}

	\begin{figure}[htb]
	    \begin{subfigure}[b]{0.49\textwidth}
	    \includegraphics[width=0.75\textwidth]{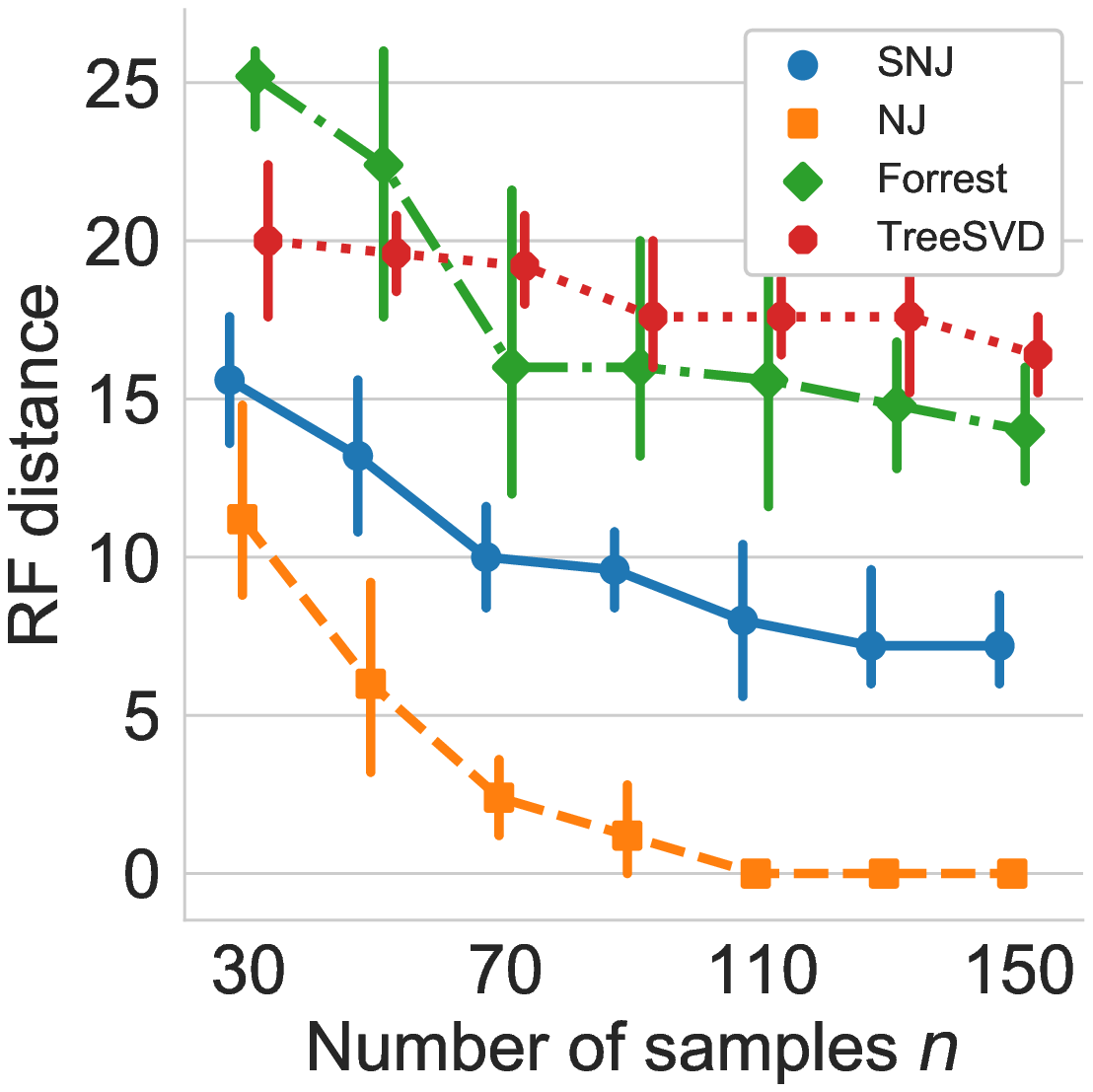}%
		\end{subfigure}
		~
		\begin{subfigure}[b]{0.49\textwidth}
		\includegraphics[width=0.75\textwidth]{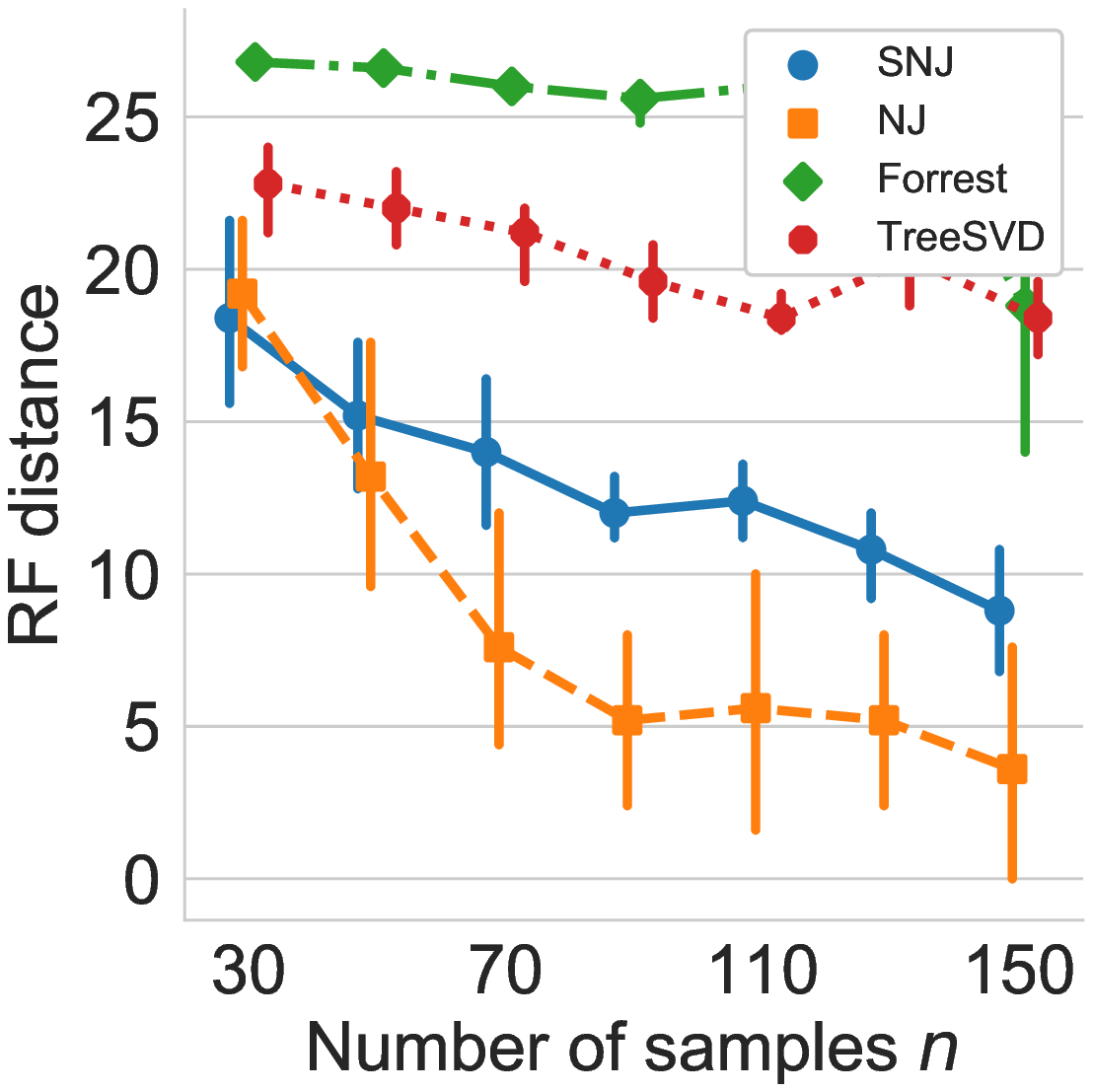}
		\end{subfigure}
			\vspace{-2.5\baselineskip}
		\caption{Comparison between NJ, SNJ, Binary forrest and Tree SVD for a caterpillar tree with $m=16$ terminal nodes,  $\delta=0.85$ (left) and $\delta=0.9$ (right).} 
		\label{fig:sim_comparison_cat_31}
	\end{figure}


\textbf{Heterogeneity in mutation rates: comparison between NJ and SNJ}

\begin{figure}[htb]
	    \begin{subfigure}[b]{0.49\textwidth}
	    \includegraphics[width=0.75\textwidth]{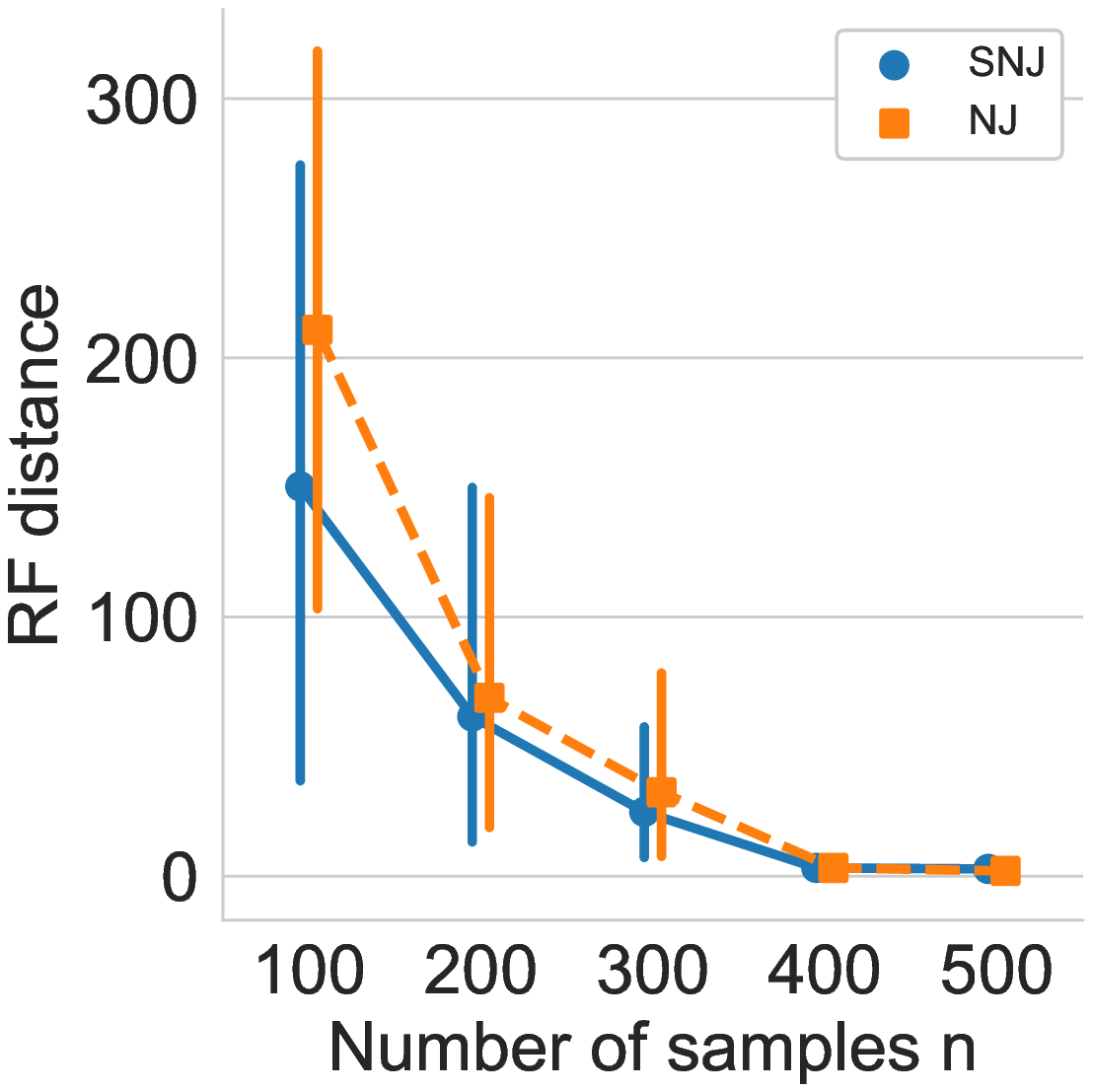}%
		\end{subfigure}
		~
		\begin{subfigure}[b]{0.49\textwidth}
		\includegraphics[width=0.75\textwidth]{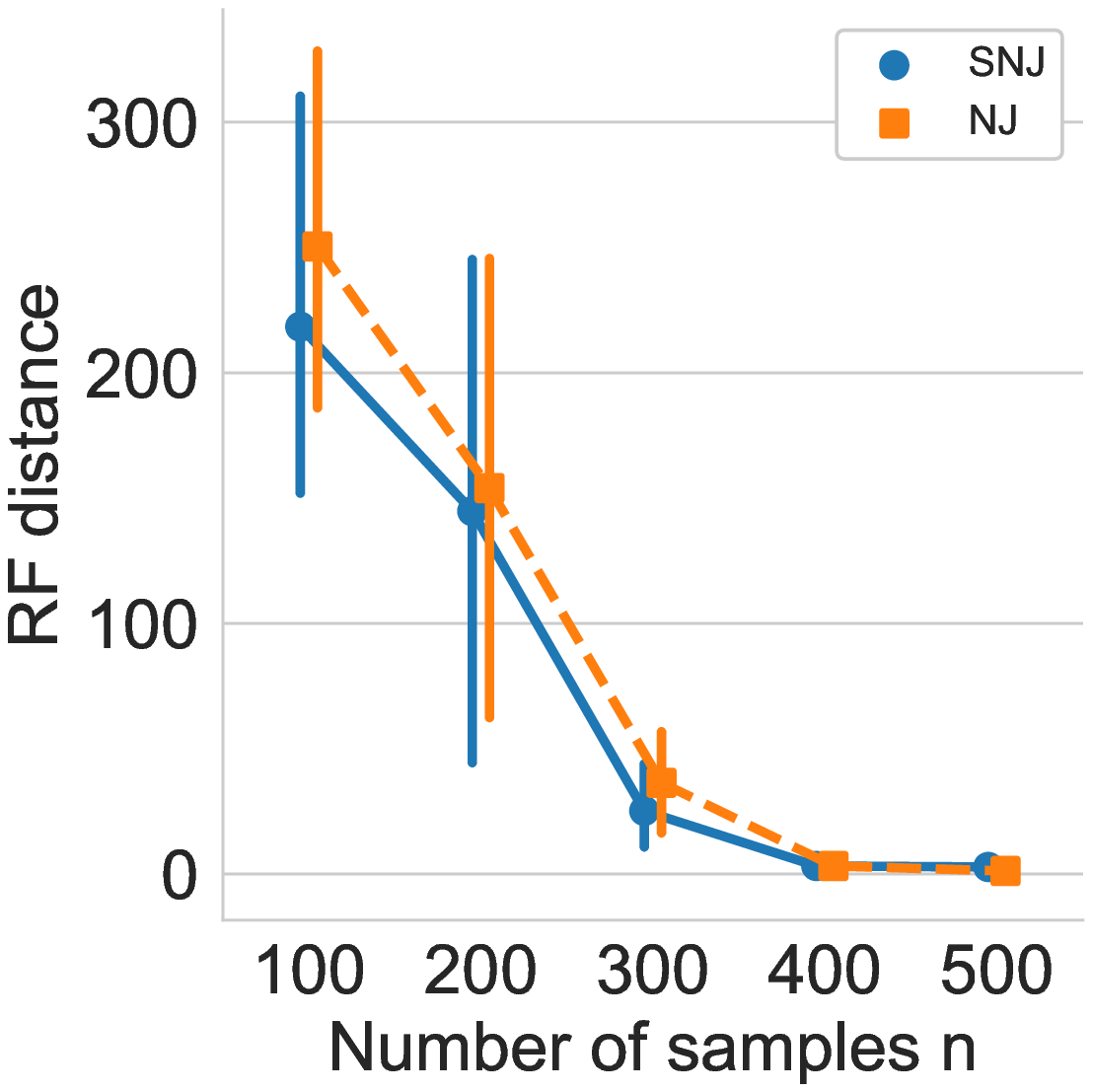}
		\end{subfigure}
			\vspace{-2.5\baselineskip}
		\caption{Heterogeneity in mutation rates. Comparison between NJ and SNJ on a binary tree with $m=256$ terminal nodes. The mutation rates were sampled according to the Gamma distribution with a shape parameter $a=5$ (left) and $a=10$ (right). } 
		\label{fig:sim_hetero}
	\end{figure}
Our last simulation compares the performance of NJ and SNJ for the case of heterogeneity in mutation rates along the sequence. The simulation was done on a binary symmetric tree with $m=128$ terminal nodes. The  similarity between all pairs of adjacent nodes is equal to $\delta r$, with $\delta$ fixed at $0.95$ and $r$
sampled according to a Gamma distribution with a mean value of one. 
Figure \ref{fig:sim_hetero} shows the RF distance as a function of number of samples $n$ for two values of $\beta$, which denotes the \textit{shape} of the Gamma distribution. A high value of $\beta$ indicates a higher degree of concentration in the mutation rates. Based on the observed data, the distance matrix $D$ was estimated via RAxML \cite{stamatakis2006raxml}. 
Finally NJ was applied based on the estimated distance matrix, and SNJ based on the corresponding similarity $R(i,j) = \exp(-D(i,j))$. The SNJ algorithm outperforms NJ in this scenario for both values of $\beta$. 

\vspace{0.5em}
\paragraph{Acknowledgements}
Y.K. acknowledges support by NIH grants R01GM131642, R01HG008383, UM1 DA051410, and 2P50CA121974.
BN is incumbent of the William Petschek professorial chair of mathematics. Part of this work was done while BN was on sabbatical at the Institute for Advanced Study at Princeton. He gratefully acknowledges the support from the Charles Simonyi Endowment.
The authors would like to thank Junhyong Kim, Stefan Steinerberger and Ronald Coifman for their help in various aspects of the paper.

\appendix

\section{Proof of Lemma \ref{lem:equivalent_statements}}
\label{sec:lemma_equivalent}

\begin{proof}
First, we assume that the subset $A$ forms the terminal nodes of a clan. Hence, there is a single edge in the tree such that all paths between $\{x_i\}_{i \in A}$
and $\{x_i\}_{i \in A^c}$ pass through it. We denote this edge by $e(h_A,h_B)$. Let $i,k \in A$ and $j,l \in A^c$. Then all paths $i \to j, i \to l,k \to j$ and $k \to l$ pass through the common edge $e(h_A,h_B)$.
It follows that the topology of the quartet is as in Figure  \ref{fig:model_subtree}.

For the other direction, assume that all quartets $x_i,x_j,x_k,x_l$, where $(i,k) \in A$ and $(j,l) \in A^c$ have a topology as in Figure \ref{fig:model_subtree}. By way of contradiction, assume that $A$ \textit{is not equal} to the terminal nodes of a clan. Consider a set of edges in the tree such that all paths between $A$ from $A^c$ pass through at least one of the edges in the set. 
If $A$ is not a clan, there is no unique edge in the tree such that all paths between $A$ and $A^c$ pass through it, and hence any such set must contain at least two edges, which we denote by $e(h_1,h_2)$ and $e(h_3,h_4)$. Note that these edges might connect between two non terminal nodes, or between one terminal and one non terminal node, see illustration in Figure \ref{fig:tightness}. Assume w.l.o.g. that $h_1$ is closer to nodes in $A$ than $h_2$ and that $h_3$ is closer to $A$ than $h_4$.
We pick a quartet of nodes $x_i,x_j,x_k,x_l$ in the following way: $x_i$ is chosen such that $i \in A$ and is closest to $h_1$ among $h_1,h_2,h_3$ and $h_4$. Similarly, $x_k,x_j,x_l$ are chosen such that $k \in A$ and $j,l \in A^c$ and they are closest to $h_3,h_2,h_4$ respectively. The topology of this quartet is not as in Figure \ref{fig:model_subtree}, which contradicts our assumption. Thus, $A$ must equal the terminal nodes of a clan.
\end{proof}

\section{Proofs of lemmas of Section \ref{sec:analysis}}\label{sec:aux_equality}

We first present the following auxiliary lemma.

\begin{lemma}
\label{lem:aux_frobenius_product}
Let $R_j^A$ and $R_k^A$ be two different blocks of the matrix $R^C$ given in Eq. (\ref{eq:R_C}). Then,
\[
\frob{(R_k^A)^T R_j^A} = \frob{R_k^A}\frob{R_j^A}.
\]
Similarly, with $R_j^B$ and $R_k^B$ also two blocks of $R^C$ corresponding to the subtree $B$,
\[
\|(R_k^A)^TR_j^A(R_j^B)^TR_k^B\|_F =\frob{R_j^A}\frob{R_k^A}\frob{R_j^B}\frob{R_k^B}.
\]
\end{lemma}
\begin{proof}[Proof of Lemma \ref{lem:aux_frobenius_product}]
By Eq. (\ref{eq:R_i_A}), $R_j^A$ and $R_k^A$ are rank one matrices, with the {\em same} left singular vector $\bm u_A$.
Thus,
\begin{align*}
\|(R_k^A)^T R_j^A\|_F &= \|\bm v_k r(h_A,h_k) \bm u_A^T \bm u_A r(h_A,h_j) \bm v_j^T\|_F 
\\
&= \|\bm u_A\|^2 r(h_A,h_j)r(h_A,h_k)\|\bm v_k \bm v_j^T\|_F \\
&=\|\bm u_A\|^2\|\bm v_j\|\|\bm v_k\| r(h_A,h_j)r(h_A,h_k)
= \|R_j^A\|_F \|R_k^A \|_F.
\end{align*}
Similarly,
\begin{align*}
(R_k^A)^TR_j^A(R_j^B)^TR_k^B\ &=\bm v_k r(h_A,h_k) \bm u_A^T \bm u_A r(h_A,h_j) \bm v_j^T \bm v_j r(h_B,h_k) \bm u_B^T \bm u_B r(h_B,h_j) \bm v_j^T. 
\end{align*}
Taking the Frobenius norm yields the second equation of the lemma.
\end{proof}

\begin{proof}[Proof of Lemma \ref{lem:tight_sigma_2}]
Consider a perfect binary tree, as in Figure \ref{fig:tightness}.  The affinity between all adjacent nodes is $\delta$, except the edge that splits the tree into two subsets of size $m/2$, whose affinity is $\xi$. 
We assume that the four clans attached to $h_A,h_B,h_C$ and $h_D$ were correctly reconstructed during the first iterations of the algorithm.
The last step to reconstruct the tree is to estimate the inner topology of $h_A,h_B,h_C$ and $h_D$.
The paths between terminal nodes in $A \cup B$ and $C \cup D$ contain $2 \log_2(m/2)$ edges with affinity $\delta$ and a single edge with affinity $\xi$. In contrast, paths connecting terminal nodes in $A$ and terminal nodes in $B$ are shorter, with only $2\log_2(m/2)$ edges with affinity $\delta$. A similar property holds for paths connecting terminal nodes in $C$ to terminal nodes in $D$. Consider the matrix $R^{A \cup C}$
of size  $m/2 \times m/2$, that contains the affinities between nodes in $A \cup C$ and $B \cup D$.
This matrix has the following block structure,
\[
R^{A \cup C} =
\delta^{2\log_2(m/2)}
\begin{bmatrix}
\bm 1 \bm 1^T & \xi\bm 1 \bm 1^T \\
\xi\bm 1 \bm 1^T & \bm 1 \bm 1^T
\end{bmatrix},
\]
where $\bm 1$ is a vector of ones of length $m/4$. The second eigenvalue of $R^{A \cup C}$ is equal to
\begin{equation}\label{eq:symmetric_tree}
\sigma_2(R^{A \cup C}) =
\frac{m}{4}\delta^{2\log_2(m/2)}(1-\xi) = \frac{1}{2}(2\delta^2)^{\log_2(m/2)}(1-\xi) = \frac{f(m,\delta,\xi)}{\delta(1+\xi)},
\end{equation}
which concludes the proof for $\delta^2\leq0.5$.
For $\delta^2>0.5$, consider a tree with $4$ terminal
nodes. 
Inserting $m=4$ in \eqref{eq:symmetric_tree} we obtain
\[
\sigma_2(R^{A \cup C}) =
\delta^{2}(1-\xi),
\]
which concludes the proof for $\delta^2>0.5$.
\end{proof}

\begin{proof}[Proof of Lemma \ref{lemma:norms2eigs}]
Denote the largest singular value of $M$ by $\sigma_1$.
Then,
\begin{equation} \label{eq:square_svd}
\frob{M}^2 = \sigma_1^2 + \sigma_2^2, \qquad
\frob{M^TM}^2 = \sigma_1^4 + \sigma_2^4,
\end{equation}
where we used the fact that the singular values of $M^TM$ are equal to the square of the singular values of $M$.
The numerator in Eq. \eqref{eq:bound_quartet} is thus equal to
\begin{equation}\label{eq:R_C_frob}
\frob{M}^4 - \frob{M^TM}^2 = (\sigma_1^2 + \sigma_2^2)^2 - (\sigma_1^4 + \sigma_2^4) = 2\sigma_1^2 \sigma_2^2.
\end{equation}
Since $0\leq\sigma_2\leq\sigma_1$,
\begin{equation}
\label{eq:auxiliary_bounds}
\sigma_2^2 \geq \frac12 \frac{2\sigma_1^2 \sigma_2^2}{\sigma_1^2 + \sigma_2^2}.
\end{equation}
Combining \eqref{eq:square_svd}, \eqref{eq:R_C_frob} and \eqref{eq:auxiliary_bounds} 
proves Eq. \eqref{eq:bound_quartet}.
\end{proof}

\begin{proof}[Proof of Lemma \ref{lem:equality_frobenius}]
Recall that the matrix $R^C$ has a block form given in Eq. (\ref{eq:R_C}).
Thus,
\begin{equation}\label{eq:RC_frob_blocks}
\frob{R^C}^2 = \sum_{j=1}^l (\frob{R_j^A}^2 + \frob{R_j^B}^2)
\end{equation}
and
\begin{align}\label{eq:Rc_frob_fourth}
\frob{R^C}^4 &= \sum_{j,k=1}^l
\frob{R_j^A}^2 \frob{R_k^A}^2 +
\frob{R_j^A}^2 \frob{R_k^B}^2 +
\frob{R_j^B}^2 \frob{R_k^A}^2 +
\frob{R_j^B}^2 \frob{R_k^B}^2.
\end{align}
To compute $\|(R^C)^TR^C\|_F^2$, by Eq. (\ref{eq:R_C}),
\begin{align*}
    (R^C)^TR^C = \begin{bmatrix}R_1^T \\ R_2^T \\ \vdots \\ R_l^T \end{bmatrix} \begin{bmatrix}R_1 & R_2 & \ldots & R_l \end{bmatrix}
    = \begin{bmatrix}
    R_1^TR_1 & R_1^TR_2 & \ldots & R_1^TR_l \\
    R_2^TR_1 & R_2^TR_2 & \ldots & R_2^TR_l \\
    \vdots & \vdots & \ddots & \vdots \\
    R_l^TR_1 & R_l^TR_2 & \ldots & R_l^TR_l
    \end{bmatrix}, 
\end{align*}
where $R_i$ was defined as the concatenation of $R_i^A$ and $R_i^B$.
Thus,
\begin{align}
\label{eq:frobenius_Rc_squared}
\|(R^C)^TR^C\|_F^2 &= \sum_{j,k=1}^l \frob{R_j^TR_k}^2 =
\sum_{j,k=1}^l  \| (R_j^A)^TR_k^A + (R_j^B)^TR_k^B\|_F^2  \notag\\
&= \sum_{j,k=1}^l \Big(\| (R_j^A)^TR_k^A\|_F^2 + \|(R_j^B)^TR_k^B\|_F^2 + 2\|(R_k^A)^TR_j^A(R_j^B)^TR_k^B\|_F^2\Big).
\end{align}
Applying Lemma \ref{lem:aux_frobenius_product} to the various terms in Eq. \eqref{eq:frobenius_Rc_squared} gives
\begin{equation} 
\|(R^C)^TR^C\|_F^2 = \sum_{j,k=1}^l \frob{R_j^A}^2\frob{R_k^A}^2 + \frob{R_j^B}^2\frob{R_k^B}^2 + 2\frob{R_j^A}\frob{R_k^A}\frob{R_j^B}\frob{R_k^B}.
	\nonumber
\end{equation}
Combining the above equation and (\ref{eq:Rc_frob_fourth}) yields
\begin{align*}
\frob{R^C}^4 - \frob{(R^C)^TR^C}^2
= \sum_{j,k=1}^l  \Big( & \frob{R_j^A}^2\frob{R_k^B}^2 + \frob{R_k^A}^2\frob{R_j^B}^2  \\
 & - 2\frob{R_j^A}\frob{R_k^A}\frob{R_j^B}\frob{R_k^B}\Big) \\
= \sum_{j,k=1}^l  \Big( & \frob{R_j^A}\frob{R_k^B} -  \frob{R_k^A}\frob{R_j^B} \Big)^2.
\end{align*}
\end{proof}

\begin{proof}[Proof of Lemma \ref{lemma:norm_u}]
	Let $\T_A$ be a clan of the tree $\T$ that contains $|A|$ terminal nodes and
	let $h_A$ be the root of the clan.
	We say that a terminal node $x_i$ is of depth $k$ if the path between $x_i$ and $h_A$ contains exactly $k$ edges.
	Let $\bm u_A$ be the vector of affinities between the terminal nodes of $\T_A$ and its root $h_A$. 
	Given the multiplicative property of the affinity function along the paths as discussed in Section \ref{sec:definitions} and assumption \eqref{eq:assumption_1}, $\|u_A\|$
	is clearly at least as large as its norm if we assume all edge affinities are exactly $\delta$. 
	
	Next, considering all possible trees with $|A|$ terminal nodes, we show that if $\delta^2 \leq 0.5$, the norm $\|\bm u\|$ is minimal for a perfect binary tree. In contrast, if $\delta^2 >0.5$, the norm is minimized
	for a caterpillar tree.	
	For both cases, our proof is based on altering the tree $\T_A$ by removing a pair of adjacent terminal nodes $x_1,x_2$ of depth $j$, and attaching them to a terminal node $x_i$ of depth $k$. We denote by $\bm u_2$ the vector of affinities between the terminal nodes and $h$ in the altered tree.
	The difference between $\|\bm u_1\|^2$ and $\|\bm u_2\|^2$ is equal to
	\begin{equation}\label{eq:diff_similar}
	\|\bm u_2\|^2-\|\bm u_1\|^2 = 2\delta^{2(k+1)}-2\delta^{2j} + \delta^{2(j-1)}-\delta^{2k} .
	\end{equation}
	The first two terms are due to the shift of $x_1,x_2$ from depth $j$ to depth $k+1$. The last two terms are due to the non terminal node attached to $x_1,x_2$ becoming terminal, while $x_i$ becoming non terminal.
	We can rewrite Eq. \eqref{eq:diff_similar} as
	\begin{equation}\label{eq:alter_diff}
	\|\bm u_2\|^2-\|\bm u_1\|^2 = (\delta^2)^{j-1}(1-2\delta^2) - (\delta^2)^{k}(1-2\delta^2) =
	(1-2\delta^2)((\delta^2)^{j-1}-(\delta^2)^{k}).
	\end{equation}
	For $\delta^2 < 0.5$, the above expression is negative if $j-1>k$. We can thus decrease the norm of the affinity vector by shifting pairs of adjacent terminal nodes of depth $j$ to depth $k+1$ where $k+1<j$. Repeating this step will decrease the norm up to a point where such a change is no longer possible. The extreme case is when the depth of all terminal nodes is equal to $\log_2(|A|)$. In this case, the norm of $\bm u$ is equal to,
	\[
	\|\bm u\|^2 = |A|\delta^{2\log_2(|A|)} = (2\delta^2)^{\log_2(|A|)}.
	\]
	For $\delta^2 > 0.5$, Eq. \eqref{eq:alter_diff} is negative if $k+1<j$. We can thus decrease the vector norm by increasing the depth of $x_1,x_2$. Repeating this step will decrease the norm up to a point where the tree contains exactly one terminal node of depth $i$ for $i = 1,\ldots,|A|-2$, and 2 terminal nodes of depth $|A|-1$.
	The squared norm of the affinity vector is bounded by
	\begin{align*}
	\|\bm u_A\|^2 &= \sum_{i=1}^{|A|-1}\delta^{2i} + \delta^{2(|A|-1)} = \delta^2\Big(\sum_{i=0}^{|A|-2}\delta^{2i} + \delta^{2(|A|-2)}\Big) \\ &\geq \delta^2\Big(\sum_{i=0}^{|A|-2}(0.5)^i + 0.5^{|A|-2}\Big)
	= 2\delta^2,
	\end{align*}
	which concludes the proof.
\end{proof}

\begin{proof}[Proof of Lemma \ref{lemma:lower_bound}]
Combining Lemmas \ref{lemma:norms2eigs} and  \ref{lem:equality_frobenius} with Eq. 
\eqref{eq:RC_frob_blocks} gives
\begin{equation}\label{eq:bound_sigma_frob}
\sigma_2(R^C)^2 \geq
\frac12 \frac{\sum_{j,k=1}^l \big(\frob{R_j^A}\frob{R_k^B} - \frob{R_j^B}\frob{R_k^A}\big)^2}{\sum_{j=1}^l \frob{R^A_j}^2 + \frob{R^B_j}^2}.
\end{equation}
Inserting Eq. \eqref{eq:frob_R_j_a} into Eq. \eqref{eq:bound_sigma_frob}
yields,
\begin{eqnarray}
\label{eq:bound_vector_norms}
\sigma_2(R^C)^2 &\geq&
\frac12 \nrm{\bm u_A}^2\nrm{\bm u_B}^2 \times \notag \\
& &\frac {  \sum_{j,k=1}^l\Big(\nrm{\bm v_j}^2\nrm{\bm v_k}^2\big(r(h_A,h_j)r(h_B,h_k) - r(h_A,h_k)r(h_B,h_j)\big)^2\Big)}
{\sum_{j=1}^l \Big(\nrm{\bm u_A}^2\nrm{\bm v_j}^2r(h_A,h_j)^2 +\nrm{\bm u_B}^2\nrm{\bm v_j}^2r(h_B,h_j)^2\Big)}.
\end{eqnarray}
We bound the ratio of sums in Eq. \eqref{eq:bound_vector_norms} by the minimum over individual ratios,
\begin{eqnarray}\label{eq:bound_minimum_ratio}
\sigma_2(R^C)^2 &\geq& \frac12 \nrm{\bm u_A}^2\nrm{\bm u_B}^2 \times \notag \\
& &\min_j
\frac { \nrm{\bm v_j}^2 \sum_{k=1}^l\Big(\nrm{\bm v_k}^2\big(r(h_A,h_j)r(h_B,h_k) - r(h_A,h_k)r(h_B,h_j)\big)^2\Big)}
{\nrm{\bm u_A}^2\nrm{\bm v_j}^2r(h_A,h_j)^2 +\nrm{\bm u_B}^2\nrm{\bm v_j}^2r(h_B,h_j)^2}. \notag
\\
&= & \frac12 \nrm{\bm u_A}^2\nrm{\bm u_B}^2 \min_j
\frac { \sum_{k=1}^l\Big(\nrm{\bm v_k}^2\big(r(h_A,h_j)r(h_B,h_k) - r(h_A,h_k)r(h_B,h_j)\big)^2 \Big)}
{\nrm{\bm u_A}^2r(h_A,h_j)^2 +\nrm{\bm u_B}^2r(h_B,h_j)^2}.
\end{eqnarray}
Let us focus on the term
\[
r(h_A,h_j)r(h_B,h_k) - r(h_A,h_k)r(h_B,h_j).
\]
Recall that $h_j$ and $h_k$ are nodes on the path from $h_A$ to $h_B$.
Obviously, if $k=j$ then this term vanishes. Else, if $h_j$ is on the path traversing from $h_A$ to $h_k$ (i.e., the path is $h_A \to h_j \to h_k \to h_B$) the affinity multiplicative property implies
\[
r(h_A,h_k) = r(h_A,h_j)r(h_j,h_k)\qquad  r(h_B,h_j) = r(h_B,h_k)r(h_k,h_j),
\]
and hence
\begin{equation}\label{eq:path_hj_hk}
r(h_A,h_j)r(h_B,h_k) - r(h_A,h_k)r(h_B,h_j) = r(h_A,h_j)r(h_B,h_k)\big(1-r(h_k,h_j)^2\big).
\end{equation}
Similarly, if $h_j$ is closer to $h_B$ (i.e., the path is $h_A \to h_k \to h_j \to h_B$) then
\begin{equation}\label{eq:path_hk_hj}
r(h_A,h_j)r(h_B,h_k) - r(h_A,h_k)r(h_B,h_j) = r(h_A,h_k)r(h_B,h_j)\big(r(h_k,h_j)^2-1\big).
\end{equation}
Combining Eq. \eqref{eq:path_hj_hk} and Eq. \eqref{eq:path_hk_hj},
\begin{align}\label{eq:inner_quartet}
\big(r(h_A,h_j)r(h_B,h_k) &- r(h_A,h_k)r(h_B,h_j)\big)^2 \notag =\\
 &\max\{r(h_A,h_j)r(h_B,h_k),r(h_A,h_k)r(h_B,h_j)\}^2(1-r(h_j,h_k)^2)^2.
\end{align}
Inserting Eq. \eqref{eq:inner_quartet} into Eq.   \eqref{eq:bound_minimum_ratio} we obtain,
\begin{align*}
&\sigma_2(R^C)^2 \geq \frac12 \nrm{\bm u_A}^2\nrm{\bm u_B}^2 \times
\\
&  \min_j \frac {  \sum_{k=1}^l\nrm{\bm v_k}^2\max\{r(h_A,h_j)r(h_B,h_k),r(h_A,h_k)r(h_B,h_j)\}^2(1-r(h_j,h_k)^2)^2}
{\nrm{\bm u_A}^2r(h_A,h_j)^2 +\nrm{\bm u_B}^2r(h_B,h_j)^2}
\\
&  \geq \frac12\nrm{\bm u_A}^2\nrm{\bm u_B}^2 \times
\\
 &\min_j \frac {  \sum_{k=1}^l\nrm{\bm v_k}^2\max\{r(h_A,h_j)r(h_B,h_k),r(h_A,h_k)r(h_B,h_j)\}^2(1-r(h_j,h_k)^2)^2}
{2\max\{\nrm{\bm u_A}r(h_A,h_j),\nrm{\bm u_B}r(h_B,h_j)\}^2}.
\end{align*}
Note that for $k=j$ we have $r(h_j,h_k)=1$.
Next, we lower bound the sum over $k$ by the maximal term $k \neq j$. Using the inequality for any non-negative elements $\max_k \{x_ky_k\}\geq \min_k\{x_k\}\max_k\{y_k\}$ yields
\begin{align}\label{eq:frac_max_max}
&\sigma_2(R^C)^2 \geq \frac12\nrm{\bm u_A}^2\nrm{\bm u_B}^2 \times
\\
&  \min_{j}\min_{k \neq j}\nrm{\bm v_k}^2 (1-r(h_j,h_k)^2)^2\max_{k \neq j} \frac { \max\{r(h_A,h_j)r(h_B,h_k),r(h_A,h_k)r(h_B,h_j)\}^2}
{2\max\{\nrm{\bm u_A}r(h_A,h_j),\nrm{\bm u_B}r(h_B,h_j)\}^2}. \notag
\end{align}
For the numerator in \eqref{eq:frac_max_max}, we apply the  following inequality,
\[
\max\{x_1y_1,x_2y_2\} \geq \max\{x_1,x_2\}\min\{y_1,y_2\}. 
\]
It follows that,
\begin{align}\label{eq:numerator_max}
&\max_{k \neq j} \max\{r(h_A,h_j)r(h_B,h_k),r(h_A,h_k)r(h_B,h_j)\}^2 \notag \\
&=  \max\{\max_{k \neq j} r(h_B,h_k)r(h_A,h_j),\max_{k \neq j}r(h_A,h_k)r(h_B,h_j)\}^2  \notag \\
&
\geq \max\{r(h_A,h_j),r(h_B,h_j)\}^2 \min\{\max_{k \neq j} r(h_B,h_k),\max_{k \neq j} r(h_A,h_k)\}^2.
\end{align}
For the denominator we have,
\begin{align}\label{eq:denominator}
\max\{\nrm{\bm u_A}r(h_A,h_j),&\nrm{\bm u_B}r(h_B,h_j)\}^2 \leq \notag \\
&\max\{\nrm{\bm u_A},\nrm{\bm u_B}\}^2 \max\{r(h_A,h_j),r(h_B,h_j)\}^2.
\end{align}
Inserting \eqref{eq:numerator_max} and \eqref{eq:denominator} into \eqref{eq:frac_max_max} we get
\begin{align*}
&\sigma_2(R^C)^2 \geq \frac12\nrm{\bm u_A}^2\nrm{\bm u_B}^2 \times
\\
&  \min_{j}\min_{k \neq j}\nrm{\bm v_k}^2 (1-r(h_j,h_k)^2)^2 \frac{ \min\{ \max_{k \neq j} r(h_B,h_k),\max_{k \neq j} r(h_A,h_k)\}^2}
{2\max\{\nrm{\bm u_A},\nrm{\bm u_B}\}^2} \notag.
\end{align*}
We conclude the proof by applying the equality $\frac{xy}{\max\{x,y\}}=\min\{x,y\}$,
\begin{align*}
\sigma_2(R^C)^2 &\geq  \frac{1}{4}\min\{\nrm{\bm u_A},\nrm{\bm u_B}\}^2
\\
&\min_j \min_{k \neq j} \nrm{\bm v_k}^2  (1-r(h_j,h_k)^2)^2
\min\{ \max_k r(h_A,h_k),\max_k r(h_B,h_k)\}^2 .
\end{align*}

\end{proof}




\begin{proof}[Proof of Theorem \ref{thm:attenson_equivalent}]
We prove the statement by induction. For simplicity, we assume $\delta^2\geq0.5$. A similar proof holds for $\delta^2>0.5$. Assuming that all pairs of subsets merged in the first $k$ iterations were adjacent clans, we prove that the algorithm will merge another pair of adjacent clans at step $k+1$. In step $1$, this assumption holds trivially, since no merges have taken place yet.
Let $A_i, A_j$ be a pair of adjacent clans and let $A_k,A_l$ be a pair of non-adjacent clans. By our inductive assumption and Theorem \ref{thm:population},
\[
\sigma_2(R^{A_i \cup A_j})=0, \qquad \sigma_2(R^{A_k \cup A_l}) \geq f(m, \delta, \xi).
\]
The Weyl inequality states that for any matrices $A$ and $B$,
$
 |\sigma_i(A+B) - \sigma_i(A)| \leq \nrm{B}.
 $
Recall that $\hat R$ is the estimate of the affinity matrix $R$. Letting $A = R^{A_i \cup A_j}$ and $B = \hat R^{A_i \cup A_j} - R^{A_i \cup A_j}$, Weyl's inequality implies
\[ |\sigma_2( \hat R^{A_i \cup A_j}) - \sigma_2(R^{A_i \cup A_j})| \leq \nrm{ \hat R^{A_i \cup A_j} - R^{A_i \cup A_j}}. \]
The spectral norm of a submatrix is bounded by the spectral norm of the full matrix, thus
\[ |\sigma_2( \hat R^{A_i \cup A_j}) - \sigma_2(R^{A_i \cup A_j})| \leq \nrm{ \hat R - R}.
 \]
For a pair of adjacent clans $(A_i,A_j)$, since $\sigma_2(R^{A_i \cup A_j})=0$, 
\[ \sigma_2(\hat R^{A_i\cup A_j}) \leq \nrm{ \hat R - R}. \]
For non-adjacent clans $(A_k,A_l)$, 
\begin{align*}
  \sigma_2(\hat R^{A_k \cup A_l})
  &\geq \sigma_2(R^{A_k\cup A_l}) - \nrm{ \hat R - R}
  \geq f(m, \delta, \xi) - \nrm{ \hat R - R}.
\end{align*}
If $\nrm{\hat R - R} \leq \frac{f(m, \delta, \xi)}2$,
then for any adjacent  $(A_i,A_j)$ and non adjacent $(A_k,A_l)$
\begin{equation}\label{eq:step_k_1}
\sigma_2(\hat R^{A_i\cup A_j}) \leq \sigma_2(\hat R^{A_k \cup A_l}).
\end{equation}
Combining the merging criterion in  \eqref{eq:sigma_criterion} with
Eq. \eqref{eq:step_k_1} proves that SNJ will merge a pair of adjacent clans in step $k+1$.
\end{proof}

\begin{proof}[Proof of Lemma \ref{lem:tail_bound}]
Consider the estimates $\hat \theta(i,j)$ and $\hat R(i,j)$ in Eq. \eqref{eq:jc_correction}.
Since $\hat \theta(i,j)$ is a sum of $n$ Bernoulli random variables with success probability $\theta(i,j)$, then by Hoeffding's inequality,
\begin{equation}\label{eq:tail_theta} \Pr \Big(|\hat \theta(i,j) - \theta(i,j)| \geq t\Big) \leq 2\exp(-2nt^2).
\end{equation}
Define $g(\theta(i,j))= (1-\frac{d}{d-1}\theta(i,j))^{d-1}$ so that $\hat R(i,j) = g(\hat \theta(i,j))$. For $\theta(i,j) \in [0,1]$ the function $g(\theta(i,j))$ is $d$-Lipschitz. Thus
\begin{equation}\label{eq:lipschitz}
 |\hat R(i,j) - R(i,j)| = |g(\hat \theta(i,j)) - g(\theta(i,j))| \leq d |\hat \theta(i,j) - \theta(i,j)|.
 \end{equation}
Combining Eq. \eqref{eq:lipschitz} with the tail bound in Eq. \eqref{eq:tail_theta}
we get,
\begin{equation*}
\Pr \Big(|\hat R(i,j) - R(i,j)| \geq t\Big)
= \Pr \Big(|\hat \theta(i,j) - \theta(i,j)| \geq \frac td \Big) \leq 2\exp\Big(-\frac{2nt^2}{d^2}\Big).
\end{equation*}
Applying a union bound over all $m^2$ entries of $\hat R - R$ gives
\[ \Pr \Big(|\hat R(i,j) - R(i,j)| \leq t \quad \forall i,j\Big) \geq 1 - 2m^2\exp\Big(-\frac{2nt^2}{d^2}\Big). \]
Finally, since $\|\hat R-R\| \leq m \max_{ij}|\hat R(i,j)-R(i,j)|$ then
\begin{align*}
    \Pr \Big( \nrm{\hat R - R} \leq t \Big) \geq 1 - 2m^2\exp\Big(-\frac{2nt^2}{d^2m^2}\Big),
\end{align*}
which concludes the proof.
\end{proof}

\begin{proof}[Proof of Lemma \ref{lem:tail_bound_general}]
We bound the error in $R$ with three steps: (i) Bound the error in the $d^2$ elements of the transition matrix $P_{x_i|x_j}$ via the Hoeffding inequality (ii) Bound the error in the determinant $\text{det}(P_{x_i|x_j})$. (iii) Take a union bound over all $m^2$ matrices.

Step (i): An element $P_{x_i|x_j}(k,l)$, is estimated by the proportion of times $x_i$ is equal to  state $k$ among the samples for which $x_j$ is equal to state $l$. The number of samples for which $x_j$ is equal to state $l$ is lower bounded by $\gamma n$. 
Thus, for estimating $P_{x_i|x_j}(k,l)$, the Hoeffding inequality with an effective sample size of at least $\gamma n$ yields,
\[
\Pr \big(|P_{x_i|x_j}(k,l)-\hat P_{x_i|x_j}(k,l)|<t   \big) > 1- 2 \exp\big(-2\gamma nt^2\big).
\]
Applying a union bound over the $d^2$ elements of $P_{x_i|x_j}$ gives
\begin{equation}\label{eq:hoeffding_general}
\Pr \big(|P_{x_i|x_j}(k,l)-\hat P_{x_i|x_j}(k,l)|<t \quad \forall (k,l)  \big) > 1- 2d^2 \exp\big(-2\gamma nt^2\big). 
\end{equation}
Step (ii): We bound the error in the estimated determinant $ \text{det}(\hat P_{x_i|x_j})$ by applying a perturbation bound proven in \cite[Theorem 2.12]{ipsen2008perturbation}. Let $P$ be a matrix of size $d$, perturbed by a matrix $E$. Then,
\begin{equation}\label{eq:det_bound}
\big|\text{det}(P)-\text{det}(P+E)\big| \leq d \nrm{E}\max \{\nrm{P},\nrm{P+E}\}^d
\end{equation}
If $\|P\|\leq 1$ then Eq. \eqref{eq:det_bound} implies
\[
\big|\text{det}(P)-\text{det}(P+E)\big| \leq d \nrm{E}(1+\nrm{E})^d. 
\]
In our setting, $E$ is the estimation error of $P_{x_i|x_j}$. 
If $n$ is large enough such that  $|E(k,l)|<1/2d^2$ for all $k,l$, it follows that $\|E \| \leq \frac{1}{2d}$. We apply the inequality $(1+x)^d \leq 1+2dx$  for $x \leq (1/2d)$ to obtain
\begin{align}
\big|\text{det}(P)-\text{det}(P+E)\big| \leq d \|E\|(1+2d\|E\|) \leq 2d \|E\|.
\end{align}
We use the inequlity $\|E\| \leq d \max_{k,l}|E(k,l)|$ to obtain,
\begin{equation}\label{eq:bound_det}
\big|\text{det}(P)-\text{det}(P+E)\big| \leq 2d^2 \max_{k,l} |E(k,l)|.
\end{equation}
Thus for $t \leq 1/2d^2$, combining Eq. \eqref{eq:bound_det} with \eqref{eq:hoeffding_general} yields
\[
\Pr \big(|\text{det}(P_{x_i|x_j})-\text{det}(\hat P_{x_i|x_j})|<2d^2t \big) > 1- 2d^2 \exp\big(2\gamma nt^2\big) 
\]

Step (iii): Taking a union bound over all $m^2$ transition matrices we get that for  $t<1/(2d^2)$
\[
\Pr \big(|R(i,j)-R(i,j))|<2d^2t \quad \forall(i,j)\big) > 1- 2(dm)^2 \exp\big(-2\gamma nt^2\big)
\]
Applying the bound $\|R-\hat R\|\leq m \max_{i,j} R(i,j)$ we get
\[
\Pr \big(\|R-\hat R\|<t \quad \forall(i,j)\big) > 1- 2(dm)^2 \exp\Big(-\frac{\gamma nt^2}{2d^4m^2}\Big).
\]

\end{proof}

\section{Proof of Lemma \ref{lem:quartets}}
\label{sec:lemma_quartet}
\noindent

We use the following two auxiliary lemmas.
The first lemma, proven in \cite{brooks2006coefficients},  gives a general relation between the $k$ size determinants of a square matrix and its singular values.
\begin{lemma}\label{lem:characteristic}
We denote by $\{\sigma_i(S)\}$  the singular values of a square matrix $S \in \R^{m \times m}$. Let $\{A^k_i\}$ be all possible size $k$ subsets of $1,\ldots,m$ and let $S(A^k_i)$ be a submatrix of $S$ that contains all elements $S_{jl}$ where $j,l \in A^k_i$. Then
\[
\sum_{i} |S(A^k_i)| =  \sum_i \prod_{j \in A_i^k} \sigma_j(S).
\]
\end{lemma}
For subsets of size $k=2$, lemma \ref{lem:characteristic} implies
\[
\sum_{i_1, i_2}
    \left|
    \begin{matrix}
    S(i_1,i_1) & S(i_1,i_2) \\
    S(i_2,i_1) & S(i_2,i_2)
    \end{matrix}
    \right|
    = \sum_{i \neq j} \sigma_i(S) \sigma_j(S).
\]
The second lemma addresses the sum of all $2 \times 2$ determinants of an arbitrary matrix.
\begin{restatable}[]{lemma}{lemdeterminants}
\label{lem:determinants}
Let $S = RR^T$ where $R$ is a matrix of arbitrary size $d_1 \times d_2$. Then,
\begin{equation}
    \sum_{i_1,i_2} \sum_{l_1,l_2}
    \left|
    \begin{matrix}
    R(i_1,l_1) & R(i_1,l_2) \\
    R(i_2,l_1) & R(i_2,l_2)
    \end{matrix}
    \right|^2 =
    2 \sum_{i_1,i_2}
    \left|
    \begin{matrix}
    S(i_1,i_1) & S(i_1,i_2) \\
    S(i_2,i_1) & S(i_2,i_2)
    \end{matrix}
    \right|
\end{equation}
\end{restatable}

\begin{proof}[Proof of Lemma \ref{lem:quartets}]
Let $S=RR^T$. Combining Lemmas \ref{lem:characteristic} and \ref{lem:determinants} gives
\begin{equation}\label{eq:combine_lemmas}
  \sum_{i_1,i_2} \sum_{l_1,l_2}
    \left|
    \begin{matrix}
    R(i_1,l_1) & R(i_1,l_2) \\
    R(i_2,l_1) & R(i_2,i_2)
    \end{matrix}
    \right|^2 =  2\sum_{i \neq j} \sigma_i(S) \sigma_j(S) = 2\sum_{i \neq j} (\sigma_i(R) \sigma_j(R))^2.\end{equation}
Let $A$ and $B$ be two clans in $\T$. Then according to
Lemma \ref{lemma:rank2}, $\text{rank}(R^{A \cup B})\leq 2$. Thus, by Eq. \eqref{eq:combine_lemmas}
\begin{equation}\label{eq:det_r_ab}
\sum_{i_1,i_2} \sum_{l_1,l_2}
    \left|
    \begin{matrix}
    R^{A \cup B} (i_1,l_1) & R^{A \cup B}(i_1,l_2) \\
    R^{A \cup B}(i_2,l_1) & R^{A \cup B}(i_2,i_2)
    \end{matrix}
    \right|^2 =   4 \sigma_1(R^{A \cup B})^2 \sigma_2(R^{A \cup B})^2,
\end{equation}
which concludes the proof.
\end{proof}

\begin{proof}[Proof of Lemma \ref{lem:determinants}]
First, we expand the determinant:
\begin{align*}
    &\left|
    \begin{matrix}
    R(i_1,l_1) & R(i_1,l_2) \\
    R(i_2,l_1) & R(i_2,l_2)
    \end{matrix}
    \right|^2 = \Big( R(i_1,l_1)R(i_2,l_2) - R(i_1,l_2)R(i_2,l_1) \Big)^2 \\
    &= R(i_1,l_1)^2 R(i_2,l_2)^2
     - 2R(i_1,l_1)R(i_2,l_2)R(i_1,l_2)R(i_2,l_1) + R(i_1,l_2)^2 R(i_2,l_1)^2. 
\end{align*}
Since $S = RR^T$ then $\sum_l R(i,l)^2 = S_{ii}$. Thus,
\[
\sum_{l_1,l_2} R(i_1,l_1)^2R(i_2,l_2)^2 = S(i_1,i_1)S(i_2,i_2).
\]
Similarly,
\begin{equation*}
    \sum_{l_1,l_2} R(i_1,l_1)R(i_2,l_2)R(i_1,l_2)R(i_2,l_1) 
    = \Big( \sum_{l} R(i_1, l)R(i_2, l)  \Big)^2 = S(i_1,i_2)^2. 
\end{equation*}
Summing up the three terms gives
\begin{equation*}
    \sum_{l_1, l_2} \left|
    \begin{matrix}
    R(i_1,l_1) & R(i_1,l_2) \\
    R(i_2,l_1) & R(i_2,l_2)
    \end{matrix}
    \right|^2
    = 2S(i_1, i_1)S(i_2, i_2) -S(i_1, i_2)^2 
    = 2 \left|
    \begin{matrix}
    S(i_1,i_1) & S(i_1,i_2) \\
    S(i_2,i_1) & S(i_2,i_2)
    \end{matrix}
    \right|.
\end{equation*}
Adding the second double summation  completes the proof.
\end{proof}

\section{Proof of Theorem \ref{thm:max_quartet}}
\label{sec:max_quartet_b}
\noindent

\update{
\begin{proof}
We prove the theorem by the following three steps, that are equivalent to Theorem \ref{thm:population}, Theorem \ref{thm:attenson_equivalent} and Theorem \ref{thm:finite_sample}, respectively. 
\begin{enumerate}
    \item For the case where $x_{A \cup B}$ is not a clan,  we derive a lower bound on the value of the max quartet criterion.
    \item We derive a sufficient gondition on the estimation error of $R$, under which the max quartet based NJ method is guaranteed to recover the accurate tree.
    \item We derive an expression for the number of samples required for (ii) to hold with high probability.
\end{enumerate}
\paragraph{Step 1} 
We define a root of a clan $x_A$ in the following way: A node $h_a$ is the root of $x_A$ if, for some $h_i$, there is an edge $e(h_a,h_i)$ that separates $x_A$ from the remaining terminal nodes, and $h_a$ is closer to $x_A$ than $h_i$.
Let $h_a,h_b$ be the roots of $x_A$ and $x_B$, respectively.  
Assuming that $x_{A \cup B}$ is not a clan, there are at least two other nodes between $h_a$ and $h_b$. For example, in Figure \ref{fig:max_quartet}, the two other nodes are $h_1,h_2$. We will construct a quartet by choosing the node $x_i \in x_A$ and $x_j \in x_B$ closest, in terms of number of edges, to $h_a,h_b$ respectively. 
Let $C,D$ be two subsets of nodes that connect to $h_1$ and $h_2$ respectively (see Figure \ref{fig:max_quartet}). 
To complete the quartet, we choose $x_k \in C ,x_l \in D$ closest to $h_1$ and $h_2$. By definition of depth of a tree, the number of edges from $x_i$ to $x_k$ and from $x_j$ to $x_l$ is smaller or equal to  $2(\text{depth}(\T)+1)$. Thus,
\begin{equation}\label{eq:max_quartet_lower_bound}
R(x_i,x_k) \geq \delta^{2(\text{depth}(\T)+1)},  \qquad R(x_j,x_l) \geq \delta^{2(\text{depth}(\T)+1)}.
\end{equation}
It follows that the max quartet criterion $M(A,B)$  is bounded by
\begin{align*}
M(A,B) &\geq |w(ij;kl)| \geq |R(i,k)R(j,l)-R(i,l)R(j,k)| \\
&=  
R(i,k)R(j,l)(1-R(h_1,h_2)^2) \geq \delta^{4(\text{depth}(\T)+1)}(1-\xi^2).
\end{align*}
\paragraph{Step 2}
Assume that the error $\mathcal E(i,j)$ in the estimate $\hat R(i,j)$ is bounded by
\[
|\mathcal E(i,j)| = |R(i,j) - \hat R(i,j)| < t \qquad \forall i \neq j.
\]
If $A \cup B$ is a clan, then for any quartet $i,k \in A \cup B$ and $j,l \notin A \cup B$ we have $w(ik;jl)=0$.
The value of the estimated quartet $\hat w(ik;jl)$ is bounded by
\begin{align*}
\hat w(ik;jl) &= |\hat R(i,j)\hat R(k,l)-\hat R(i,k)\hat R(k,l)|\\
&=
|(\mathcal E(i,j) + R(i,j))(\mathcal E(k,l) + R(k,l)) -(\mathcal E(i,l) + R(i,l))(\mathcal E(k,j) + R(k,j))| \\
&\leq 
|R(i,j)R(k,l)-R(i,l)R(k,j)| 
\\
&\qquad+ 
|\mathcal E(i,j) R(k,l) + R(i,j) \mathcal E(k,l) - \mathcal E(i,l)R(k,j) - R(i,l) \mathcal E(k,j)|  
\\
&\qquad+
|\mathcal E(i,j) \mathcal E(k,l) - \mathcal E(i,l) \mathcal E(k,j)|
\\
&\leq 0 + 4t + 2t^2
\end{align*}
where we used the fact that $0 < R(i,j)< 1$. 
To simplify the expression, we assume that the number of samples is sufficiently large such that $t<0.5$ and hence
\[
\hat w(ik;jl) \leq 5t.
\]
This bound holds for any quartet, and specifically for the max quartet. It follows that if $A \cup B$ is a clan then
\begin{equation}\label{eq:upper_bound_max_quartet}
M(A,B) \leq 5 t.    
\end{equation}
A similar derivation, together with the bounds in Eq. \eqref{eq:max_quartet_lower_bound} yields that if $A \cup B$ is not a clan then 
\begin{equation}\label{eq:lower_bound_max_quartet}
M(A,B) \geq  \delta^{2(\text{depth}(\T)+1)}-5t.
\end{equation}
By combining \eqref{eq:upper_bound_max_quartet} and \eqref{eq:lower_bound_max_quartet}, we prove that if $t \leq \min\{\delta^{2(\text{depth}(\T)+1)}/10,0.5\} = \delta^{2(\text{depth}(\T)+1)}/10$, minimizing the max quartet criterion always yields a merge between two subsets that form a clan. 
\paragraph{Step 3:} Lastly, by combining Hoeffding's inequality and the union bound (see proof of Lemma \ref{lem:tail_bound})
\[
\Pr(|\hat R(i,j) - R(i,j)|\leq t \quad \forall i,j) \geq 1-2m^2\exp\Big(\frac{2nt^2}{d^2}\Big).
\]
Replacing $t$ with $\delta^{2(\text{depth}(\T)+1)}/10$ we get that the algorithm will recover the correct tree with probability at least $1-\epsilon$ if 
\[
n \geq 100d^2\log \Big ( \frac{2m^2}{\epsilon}\Big) \delta^{-4(\text{depth}(\T)+1)}.
\]
\end{proof}
}

\bibliography{references}

\begin{thebibliography}{10}

\bibitem{allman2017split}
Elizabeth~S Allman, Laura~S Kubatko, and John~A Rhodes.
\newblock Split scores: a tool to quantify phylogenetic signal in genome-scale
  data.
\newblock {\em Systematic Biology}, 66(4):620--636, 2017.

\bibitem{allman2007molecular}
Elizabeth~S Allman and John~A Rhodes.
\newblock Molecular phylogenetics from an algebraic viewpoint.
\newblock {\em Statistica Sinica}, pages 1299--1316, 2007.

\bibitem{anandkumar2011spectral}
Animashree Anandkumar, Kamalika Chaudhuri, Daniel~J Hsu, Sham~M Kakade,
  Le~Song, and Tong Zhang.
\newblock Spectral methods for learning multivariate latent tree structure.
\newblock In {\em Advances in neural information processing systems}, pages
  2025--2033, 2011.

\bibitem{ArisBrosou1996impact}
S~Aris-Brosou and L~Excoffier.
\newblock {The impact of population expansion and mutation rate heterogeneity
  on DNA sequence polymorphism.}
\newblock {\em Molecular Biology and Evolution}, 13(3):494--504, 03 1996.

\bibitem{atteson1999performance}
Kevin Atteson.
\newblock The performance of neighbor-joining methods of phylogenetic
  reconstruction.
\newblock {\em Algorithmica}, 25(2-3):251--278, 1999.

\bibitem{brooks2006coefficients}
Bernard~P Brooks.
\newblock The coefficients of the characteristic polynomial in terms of the
  eigenvalues and the elements of an n $\times$ n matrix.
\newblock {\em Applied mathematics letters}, 19(6):511--515, 2006.

\bibitem{bryant2005uniqueness}
David Bryant.
\newblock On the uniqueness of the selection criterion in neighbor-joining.
\newblock {\em Journal of Classification}, 22(1):3--15, 2005.

\bibitem{camin1965method}
Joseph~H Camin and Robert~R Sokal.
\newblock A method for deducing branching sequences in phylogeny.
\newblock {\em Evolution}, 19(3):311--326, 1965.

\bibitem{cavender1987invariants}
James~A Cavender and Joseph Felsenstein.
\newblock Invariants of phylogenies in a simple case with discrete states.
\newblock {\em Journal of Classification}, 4(1):57--71, 1987.

\bibitem{chang1996full}
Joseph~T Chang.
\newblock Full reconstruction of markov models on evolutionary trees:
  identifiability and consistency.
\newblock {\em Mathematical Biosciences}, 137(1):51--73, 1996.

\bibitem{chang1991reconstruction}
Joseph~T Chang and John~A Hartigan.
\newblock Reconstruction of evolutionary trees from pairwise distributions on
  current species.
\newblock In {\em Computing science and statistics: Proceedings of the 23rd
  symposium on the interface}, pages 254--257. Interface Foundation Fairfax
  Station, VA, 1991.

\bibitem{choi2011learning}
Myung~Jin Choi, Vincent~YF Tan, Animashree Anandkumar, and Alan~S Willsky.
\newblock Learning latent tree graphical models.
\newblock {\em Journal of Machine Learning Research}, 12(May):1771--1812, 2011.

\bibitem{day1986computational}
William~HE Day and David Sankoff.
\newblock Computational complexity of inferring phylogenies by compatibility.
\newblock {\em Systematic Biology}, 35(2):224--229, 1986.

\bibitem{delsuc2005phylogenomics}
Fr{\'e}d{\'e}ric Delsuc, Henner Brinkmann, and Herv{\'e} Philippe.
\newblock Phylogenomics and the reconstruction of the tree of life.
\newblock {\em Nature Reviews Genetics}, 6(5):361, 2005.

\bibitem{durbin1998biological}
Richard Durbin, Sean~R Eddy, Anders Krogh, and Graeme Mitchison.
\newblock {\em Biological sequence analysis: probabilistic models of proteins
  and nucleic acids}.
\newblock Cambridge university press, 1998.

\bibitem{erdHos1999few}
P{\'e}ter~L Erd{\H{o}}s, Michael~A Steel, L{\'a}szl{\'o}~A Sz{\'e}kely, and
  Tandy~J Warnow.
\newblock A few logs suffice to build (almost) all trees (i).
\newblock {\em Random Structures \& Algorithms}, 14(2):153--184, 1999.

\bibitem{eriksson2005tree}
Nicholas Eriksson.
\newblock Tree construction using singular value decomposition.
\newblock In Lior Pachter and Bernd Sturmfels, editors, {\em Algebraic
  statistics for computational biology}, pages 347--358. Cambridge University
  Press, 2005.

\bibitem{estabrook1985comparison}
George~F Estabrook, FR~McMorris, and Christopher~A Meacham.
\newblock Comparison of undirected phylogenetic trees based on subtrees of four
  evolutionary units.
\newblock {\em Systematic Zoology}, 34(2):193--200, 1985.

\bibitem{felsenstein1981evolutionary}
Joseph Felsenstein.
\newblock Evolutionary trees from dna sequences: a maximum likelihood approach.
\newblock {\em Journal of molecular evolution}, 17(6):368--376, 1981.

\bibitem{felsenstein2004inferring}
Joseph Felsenstein.
\newblock {\em Inferring phylogenies}, volume~2.
\newblock Sinauer associates Sunderland, MA, 2004.

\bibitem{fernandez2016invariant}
Jes{\'u}s Fern{\'a}ndez-S{\'a}nchez and Marta Casanellas.
\newblock Invariant versus classical quartet inference when evolution is
  heterogeneous across sites and lineages.
\newblock {\em Systematic Biology}, 65(2):280--291, 2016.

\bibitem{fitch1971toward}
Walter~M Fitch.
\newblock Toward defining the course of evolution: minimum change for a
  specific tree topology.
\newblock {\em Systematic Biology}, 20(4):406--416, 1971.

\bibitem{gascuel2006neighbor}
Olivier Gascuel and Mike Steel.
\newblock Neighbor-joining revealed.
\newblock {\em Molecular Biology and Evolution}, 23(11):1997--2000, 2006.

\bibitem{gascuel2016stochastic}
Olivier Gascuel and Mike Steel.
\newblock A ‘stochastic safety radius’ for distance-based tree
  reconstruction.
\newblock {\em Algorithmica}, 74(4):1386--1403, 2016.

\bibitem{gufuli}
X~Gu, Y~X Fu, and W~H Li.
\newblock {Maximum likelihood estimation of the heterogeneity of substitution
  rate among nucleotide sites.}
\newblock {\em Molecular Biology and Evolution}, 12(4):546--557, 07 1995.

\bibitem{guindon2003simple}
St{\'e}phane Guindon and Olivier Gascuel.
\newblock A simple, fast, and accurate algorithm to estimate large phylogenies
  by maximum likelihood.
\newblock {\em Systematic Biology}, 52(5):696--704, 2003.

\bibitem{hajdinjak2018reconstructing}
Mateja Hajdinjak, Qiaomei Fu, Alexander H{\"u}bner, Martin Petr, Fabrizio
  Mafessoni, Steffi Grote, Pontus Skoglund, Vagheesh Narasimham, H{\'e}l{\`e}ne
  Rougier, Isabelle Crevecoeur, et~al.
\newblock Reconstructing the genetic history of late neanderthals.
\newblock {\em Nature}, 555(7698):652, 2018.

\bibitem{harmeling2010greedy}
Stefan Harmeling and Christopher~KI Williams.
\newblock Greedy learning of binary latent trees.
\newblock {\em IEEE Transactions on Pattern Analysis and Machine Intelligence},
  33(6):1087--1097, 2010.

\bibitem{huang2017scalable}
Furong Huang, Niranjan UN, Joachim Perros, Robert Chen, Jimeng Sun, and Anima
  Anandkumar.
\newblock Scalable latent tree model and its application to health analytics.
\newblock {\em Machine Learning in Healthcare NIPS Workshop 2015,
  arXiv:1406.4566 [cs.LG]}, 2015.

\bibitem{ipsen2008perturbation}
Ilse~CF Ipsen and Rizwana Rehman.
\newblock Perturbation bounds for determinants and characteristic polynomials.
\newblock {\em SIAM Journal on Matrix Analysis and Applications},
  30(2):762--776, 2008.

\bibitem{jaffe2016unsupervised}
Ariel Jaffe, Ethan Fetaya, Boaz Nadler, Tingting Jiang, and Yuval Kluger.
\newblock Unsupervised ensemble learning with dependent classifiers.
\newblock In {\em Artificial Intelligence and Statistics}, pages 351--360,
  2016.

\bibitem{jaffe2015estimating}
Ariel Jaffe, Boaz Nadler, and Yuval Kluger.
\newblock Estimating the accuracies of multiple classifiers without labeled
  data.
\newblock In {\em Artificial Intelligence and Statistics}, pages 407--415,
  2015.

\bibitem{jaffe2018learning}
Ariel Jaffe, Roi Weiss, Shai Carmi, Yuval Kluger, and Boaz Nadler.
\newblock Learning binary latent variable models: A tensor eigenpair approach.
\newblock {\em Proceedings of the 35th International Conference on
  International Conference on Machine Learning}, 2018.

\bibitem{intraspecifichetero}
Fangzhi Jia, Nathan Lo, and Simon Y.~W. Ho.
\newblock The impact of modelling rate heterogeneity among sites on
  phylogenetic estimates of intraspecific evolutionary rates and timescales.
\newblock {\em PLOS ONE}, 9(5):1--8, 05 2014.

\bibitem{jiang2001polynomial}
Tao Jiang, Paul Kearney, and Ming Li.
\newblock A polynomial time approximation scheme for inferring evolutionary
  trees from quartet topologies and its application.
\newblock {\em SIAM Journal on Computing}, 30(6):1942--1961, 2001.

\bibitem{john2003performance}
Katherine~St John, Tandy Warnow, Bernard~ME Moret, and Lisa Vawter.
\newblock Performance study of phylogenetic methods:(unweighted) quartet
  methods and neighbor-joining.
\newblock {\em Journal of Algorithms}, 48(1):173--193, 2003.

\bibitem{jukes1969evolution}
Thomas~H Jukes and Charles~R Cantor.
\newblock Evolution of protein molecules.
\newblock In H.N. Munro, editor, {\em Mammalian Protein Metabolism}, pages 21
  -- 132. Academic Press, 1969.

\bibitem{lacey2006signal}
Michelle~R Lacey and Joseph~T Chang.
\newblock A signal-to-noise analysis of phylogeny estimation by
  neighbor-joining: insufficiency of polynomial length sequences.
\newblock {\em Mathematical Biosciences}, 199(2):188--215, 2006.

\bibitem{lake1994reconstructing}
James~A Lake.
\newblock Reconstructing evolutionary trees from dna and protein sequences:
  paralinear distances.
\newblock {\em Proceedings of the National Academy of Sciences},
  91(4):1455--1459, 1994.

\bibitem{lanciotti2016phylogeny}
Robert~S Lanciotti, Amy~J Lambert, Mark Holodniy, Sonia Saavedra, and Leticia
  del Carmen~Castillo Signor.
\newblock Phylogeny of zika virus in western hemisphere, 2015.
\newblock {\em Emerging infectious diseases}, 22(5):933, 2016.

\bibitem{mihaescu2009neighbor}
Radu Mihaescu, Dan Levy, and Lior Pachter.
\newblock Why neighbor-joining works.
\newblock {\em Algorithmica}, 54(1):1--24, 2009.

\bibitem{mossel2005learning}
Elchanan Mossel and S{\'e}bastien Roch.
\newblock Learning nonsingular phylogenies and hidden {M}arkov models.
\newblock In {\em Proceedings of the thirty-seventh annual ACM symposium on
  Theory of computing}, pages 366--375, 2005.

\bibitem{mourad2013survey}
Rapha{\"e}l Mourad, Christine Sinoquet, Nevin~Lianwen Zhang, Tengfei Liu, and
  Philippe Leray.
\newblock A survey on latent tree models and applications.
\newblock {\em Journal of Artificial Intelligence Research}, 47:157--203, 2013.

\bibitem{nei2000molecular}
Masatoshi Nei and Sudhir Kumar.
\newblock {\em Molecular evolution and phylogenetics}.
\newblock Oxford university press, 2000.

\bibitem{parisi2014ranking}
Fabio Parisi, Francesco Strino, Boaz Nadler, and Yuval Kluger.
\newblock Ranking and combining multiple predictors without labeled data.
\newblock {\em Proceedings of the National Academy of Sciences},
  111(4):1253--1258, 2014.

\bibitem{pauplin2000direct}
Yves Pauplin.
\newblock Direct calculation of a tree length using a distance matrix.
\newblock {\em Journal of Molecular Evolution}, 51(1):41--47, 2000.

\bibitem{pearl1986structuring}
Judea Pearl and Michael Tarsi.
\newblock Structuring causal trees.
\newblock {\em Journal of Complexity}, 2(1):60--77, 1986.

\bibitem{rannala1996probability}
Bruce Rannala and Ziheng Yang.
\newblock Probability distribution of molecular evolutionary trees: a new
  method of phylogenetic inference.
\newblock {\em Journal of molecular evolution}, 43(3):304--311, 1996.

\bibitem{ranwez2001quartet}
Vincent Ranwez and Olivier Gascuel.
\newblock Quartet-based phylogenetic inference: improvements and limits.
\newblock {\em Molecular Biology and Evolution}, 18(6):1103--1116, 2001.

\bibitem{rhodes2019topological}
John~A Rhodes.
\newblock Topological metrizations of trees, and new quartet methods of tree
  inference.
\newblock {\em IEEE/ACM Transactions on Computational Biology and
  Bioinformatics}, 2019.

\bibitem{roch2006short}
Sebastien Roch.
\newblock A short proof that phylogenetic tree reconstruction by maximum
  likelihood is hard.
\newblock {\em IEEE/ACM Transactions on Computational Biology and
  Bioinformatics}, 3(1):92--94, 2006.

\bibitem{rusinko2012invariant}
Joseph~P Rusinko and Brian Hipp.
\newblock Invariant based quartet puzzling.
\newblock {\em Algorithms for Molecular Biology}, 7(1):35, 2012.

\bibitem{saitou1987neighbor}
Naruya Saitou and Masatoshi Nei.
\newblock The neighbor-joining method: a new method for reconstructing
  phylogenetic trees.
\newblock {\em Molecular Biology and Evolution}, 4(4):406--425, 1987.

\bibitem{semple2003phylogenetics}
Charles Semple, Mike Steel, et~al.
\newblock {\em Phylogenetics}, volume~24.
\newblock Oxford University Press on Demand, 2003.

\bibitem{smith1994rooting}
Andrew~B Smith.
\newblock Rooting molecular trees: problems and strategies.
\newblock {\em Biological Journal of the Linnean Society}, 51(3):279--292,
  1994.

\bibitem{snir2008short}
Sagi Snir, Tandy Warnow, and Satish Rao.
\newblock Short quartet puzzling: A new quartet-based phylogeny reconstruction
  algorithm.
\newblock {\em Journal of Computational Biology}, 15(1):91--103, 2008.

\bibitem{sokal1958statistical}
Robert~R Sokal.
\newblock A statistical method for evaluating systematic relationships.
\newblock {\em Univ. Kansas, Sci. Bull.}, 38:1409--1438, 1958.

\bibitem{stamatakis2006raxml}
Alexandros Stamatakis.
\newblock Raxml-vi-hpc: maximum likelihood-based phylogenetic analyses with
  thousands of taxa and mixed models.
\newblock {\em Bioinformatics}, 22(21):2688--2690, 2006.

\bibitem{steel2016phylogeny}
Mike Steel.
\newblock {\em Phylogeny: discrete and random processes in evolution}.
\newblock SIAM, 2016.

\bibitem{strimmer1996accuracy}
Korbinian Strimmer and Arndt von Haeseler.
\newblock Accuracy of neighbor joining for n-taxon trees.
\newblock {\em Systematic Biology}, 45(4):516--523, 1996.

\bibitem{strimmer1996quartet}
Korbinian Strimmer and Arndt Von~Haeseler.
\newblock Quartet puzzling: a quartet maximum-likelihood method for
  reconstructing tree topologies.
\newblock {\em Molecular Biology and Evolution}, 13(7):964--969, 1996.

\bibitem{sukumaran2010dendropy}
Jeet Sukumaran and Mark~T Holder.
\newblock Dendropy: a python library for phylogenetic computing.
\newblock {\em Bioinformatics}, 26(12):1569--1571, 2010.

\bibitem{susko2004inconsistency}
Edward Susko, Yuji Inagaki, and Andrew~J Roger.
\newblock On inconsistency of the neighbor-joining, least squares, and minimum
  evolution estimation when substitution processes are incorrectly modeled.
\newblock {\em Molecular Biology and Evolution}, 21(9):1629--1642, 2004.

\bibitem{tamura2004prospects}
Koichiro Tamura, Masatoshi Nei, and Sudhir Kumar.
\newblock Prospects for inferring very large phylogenies by using the
  neighbor-joining method.
\newblock {\em Proceedings of the National Academy of Sciences},
  101(30):11030--11035, 2004.

\bibitem{waddell1997general}
Peter~J Waddell and MA~Steel.
\newblock General time-reversible distances with unequal rates across sites:
  mixing $\gamma$ and inverse gaussian distributions with invariant sites.
\newblock {\em Molecular phylogenetics and evolution}, 8(3):398--414, 1997.

\bibitem{wakeley2009coalescent}
John Wakeley.
\newblock {\em Coalescent theory: an introduction}.
\newblock Number 575: 519.2 WAK. 2009.

\bibitem{wilkinson2007clades}
Mark Wilkinson, James~O McInerney, Robert~P Hirt, Peter~G Foster, and T~Martin
  Embley.
\newblock Of clades and clans: terms for phylogenetic relationships in unrooted
  trees.
\newblock {\em Trends in ecology \& evolution}, 22(3):114--115, 2007.

\bibitem{yang2012molecular}
Ziheng Yang and Bruce Rannala.
\newblock Molecular phylogenetics: principles and practice.
\newblock {\em Nature reviews genetics}, 13(5):303, 2012.

\end{thebibliography}
\bibliographystyle{plain}
\end{document}